\documentclass{article} 
\usepackage{newtxtext}
\usepackage[left=1.25in, top=1in, bottom=1in, right=1.25in]{geometry}
\usepackage[utf8]{inputenc}
\usepackage[T1]{fontenc}
\usepackage[authoryear]{natbib}
\usepackage{microtype}
\usepackage{float}

\usepackage{amsmath,amssymb,amsthm}
\usepackage{bm}
\usepackage{mathtools}
\usepackage{booktabs}
\usepackage{array}
\usepackage{enumitem}
\usepackage{graphicx}
\usepackage{xcolor}
\usepackage{ifthen}
\usepackage{url}

\newcommand{\ShowColors}{yes}
\ifthenelse{\equal{\ShowColors}{yes}}{
\colorlet{linkcol}{blue}\colorlet{urlcol}{magenta}
\colorlet{cmtorange}{orange}\colorlet{cmtred}{red}\colorlet{cmtblue}{blue}\colorlet{cmtcyan}{cyan}
\definecolor{cmtgreen}{RGB}{0,128,80}
\definecolor{addcolor}{RGB}{204,153,0}
}{
\colorlet{linkcol}{black}\colorlet{urlcol}{black}
\colorlet{cmtorange}{black}\colorlet{cmtred}{black}\colorlet{cmtblue}{black}\colorlet{cmtcyan}{black}
\colorlet{cmtgreen}{black}
\colorlet{addcolor}{black}
}

\usepackage{hyperref}
\hypersetup{colorlinks=true,linkcolor=linkcol,citecolor=linkcol,urlcolor=urlcol}
\usepackage[capitalize]{cleveref}
\newcommand{\tianyu}[1]{}
\newcommand{\sy}[1]{}
\newcommand{\garvit}[1]{}
\newcommand{\addressedyaoqing}[1]{}
\newcommand{\addressedtianyu}[1]{}
\newcommand{\addressedsy}[1]{}
\newcommand{\LEI}[1]{}
\newcommand{\LEIB}[1]{}
\newcommand{\PY}[1]{}
\newcommand{\PYB}[1]{}
\newtheorem{theorem}{Theorem}
\newtheorem{lemma}{Lemma}

\newtheorem{corollary}{Corollary}
\newtheorem{assumption}{Assumption}
\newtheorem{remark}{Remark}

\makeatletter
\renewenvironment{proof}[1][Proof]{%
  \par\addvspace{\topsep}\noindent{\bfseries #1.\ }\ignorespaces%
}{%
  \hfill\rule{1.1ex}{1.1ex}\par\addvspace{\topsep}%
}
\newenvironment{theoremrst}[1]{%
  \par\addvspace{\topsep}\noindent{\bfseries Theorem (#1).}\space\normalfont\itshape\ignorespaces%
}{%
  \par\upshape\addvspace{\topsep}%
}
\makeatother

\crefname{assumption}{Assumption}{Assumptions}
\Crefname{assumption}{Assumption}{Assumptions}
\crefname{theorem}{Theorem}{Theorems}\Crefname{theorem}{Theorem}{Theorems}
\crefname{lemma}{Lemma}{Lemmas}\Crefname{lemma}{Lemma}{Lemmas}
\crefname{corollary}{Corollary}{Corollaries}\Crefname{corollary}{Corollary}{Corollaries}
\crefname{proposition}{Proposition}{Propositions}\Crefname{proposition}{Proposition}{Propositions}
\crefname{remark}{Remark}{Remarks}\Crefname{remark}{Remark}{Remarks}
\crefname{definition}{Definition}{Definitions}\Crefname{definition}{Definition}{Definitions}
\crefname{equation}{Eq.}{Eqs.}
\crefname{section}{Section}{Sections}
\crefname{appendix}{Appendix}{Appendices}
\crefname{figure}{Figure}{Figures}
\crefname{table}{Table}{Tables}

\newcommand{\R}{\mathbb{R}}
\newcommand{\E}{\mathbb{E}}
\newcommand{\bN}{\mathcal{N}}
\newcommand{\bL}{\mathcal{L}}
\newcommand{\norm}[1]{\left\|#1\right\|}
\newcommand{\Frob}{\mathrm{F}}
\DeclareMathOperator{\tr}{tr}
\DeclareMathOperator{\diagop}{diag}
\newcommand{\diag}{\diagop}
\DeclareMathOperator{\Var}{Var}
\newcommand{\1}{\mathbf{1}}
\newcommand{\Mstar}{\bm{M}^{\star}}
\newcommand{\eqdef}{:=}
\newcommand{\kron}{\otimes}
\newcommand{\WQ}{\bm{W}_{Q}}
\newcommand{\WK}{\bm{W}_{K}}
\newcommand{\WV}{\bm{W}_{V}}
\newcommand{\WO}{\bm{W}_{O}}
\newcommand{\GQ}{\bm{G}_{Q}}
\newcommand{\GK}{\bm{G}_{K}}
\newcommand{\GV}{\bm{G}_{V}}
\newcommand{\GO}{\bm{G}_{O}}
\newcommand{\mA}{\mathsf{A}}

\newcommand{\ShowAdditions}{no}
\ifthenelse{\equal{\ShowAdditions}{yes}}{
\newcommand{\addtxt}[1]{{\color{addcolor}#1}}
\newenvironment{addblock}{\color{addcolor}}{\par}
}{
\newcommand{\addtxt}[1]{#1}
\newenvironment{addblock}{}{\par}
}
\definecolor{revcolor}{RGB}{0,90,200}
\newcommand{\ShowRevisions}{no}
\ifthenelse{\equal{\ShowRevisions}{yes}}{
\newcommand{\revtxt}[1]{{\color{revcolor}#1}}
\newenvironment{revblock}{\color{revcolor}}{\par}
}{
\newcommand{\revtxt}[1]{#1}
\newenvironment{revblock}{}{\par}
}

\title{How Does Orthogonalization Adapt to the Neural Network Hessian Structure? \\ A Gradient Self Outer-Product Analysis at Initialization}

\author{Shenyang Deng \\
Dartmouth College \\
\texttt{shenyang.deng.gr@dartmouth.edu}
\and
Garvit Goyal \\
Dartmouth College \\
\texttt{garvit.goyal.gr@dartmouth.edu}
\and
Tianyu Pang \\
Dartmouth College \\
\texttt{tianyu.pang.gr@dartmouth.edu}
\and
Shuhua Yu \\
Meta \\
\texttt{yu.shuhuaxh@gmail.com}
\and
Yaoqing Yang \\
Dartmouth College \\
\texttt{yaoqing.yang@dartmouth.edu}
}

\begin{document}

\maketitle

\begin{abstract}%
Muon orthogonalizes a weight matrix's momentum before each step,
and on neural networks this simple preconditioner beats AdamW in a
wide range of tasks. Existing analyses explain Muon within general
optimization frameworks or on a few specific models, and many of
them view orthogonalization as a preconditioner. Since the Hessian
is the optimal preconditioner under a quadratic approximation of the
loss, these works leave a basic question open: what makes
orthogonalization suited to the Hessian of a neural network? Prior
work shows that the
layer-wise Hessian of a neural network is row-block diagonally
dominant, and Muon's implicit preconditioner, the Kronecker product
$(\bm{V}\bm{V}^\top)^{1/2}\otimes\bm{I}$ with $\bm{V}$ the momentum,
has the same row-block structure exactly when $\bm{V}\bm{V}^\top$ is
diagonally dominant. Diagonal dominance of $\bm{V}\bm{V}^\top$ has
been observed in language models, but why a neural network produces
it has not been explained. This work provides a quantitative
asymptotic theory for neural networks at initialization, and points
out the intuition that the overparameterization of the weight
matrices drives the diagonal dominance. We analyze
$\mathbb{E}[\bm{G}\bm{G}^\top]$, which equals
$\mathbb{E}[\bm{V}\bm{V}^\top]$ at initialization, in four canonical
models: symmetric matrix factorization, deep linear networks,
two-layer ReLU networks, and single-layer linear attention. In each
case the diagonal entries outgrow the off-diagonal ones at explicit
rates as the hidden dimensions grow; for the query, key and value
weights of attention the off-diagonal vanishes altogether.
Simulations of the four models and experiments on language models
at initialization confirm the predicted scaling. The code for all
our simulations and language-model experiments is available at this
\href{https://anonymous.4open.science/r/orthogonalization-hessian-CB2C/}{link}.
\end{abstract}

\section{Introduction}\label{sec:intro}

Muon~\citep{jordan2024muon} optimizes a matrix-shaped weight $\bm{W}_t\in\R^{m\times n}$ by orthogonalizing the momentum buffer instead of using it directly. From a stochastic gradient $\bm{G}_t$,
\begin{equation}\label{eq:muon}
\bm{V}_t=\beta\bm{V}_{t-1}+(1-\beta)\bm{G}_t,\qquad
\bm{D}_t=\mathrm{NS}_5(\bm{V}_t)\approx(\bm{V}_t\bm{V}_t^\top)^{-1/2}\bm{V}_t,\qquad
\bm{W}_{t+1}=\bm{W}_t-\eta_t\bm{D}_t,
\end{equation}
with $\bm{V}_0=\bm{G}_0$ and $\mathrm{NS}_5$ a fixed-point iteration for the orthogonal polar factor: if $\bm{V}_t=\bm{P}\bm{\Sigma}\bm{Q}^\top$ is a thin SVD then $\bm{D}_t=\bm{P}\bm{Q}^\top=(\bm{V}_t\bm{V}_t^\top)^{-1/2}\bm{V}_t$. A growing body of work studies why Muon works within general optimization frameworks: as momentum steepest descent under a chosen operator norm~\citep{bernstein2024old,bernstein2024modular,pethick2025training,lau2025polargrad}, as a non-Euclidean trust-region step~\citep{kovalev2025understanding}, under an isotropic curvature model~\citep{su2025isotropic}, with an implicit spectral-norm constraint on the iterates~\citep{chen2025muon}, and with convergence guarantees under smoothness or H\"older smoothness, heavy-tailed noise, inexact orthogonalization, variance reduction, and nonsmooth or convex Lipschitz objectives~\citep{li2025note,shen2025convergence,an2025asgo,kovalev2026non,nagashima2026improved,sato2026criticalbatch,gratton2026unified,hubler2026free,iiduka2026muon,choudhury2026muon,shulgin2025inexact,kim2026newtonschulz,qian2026convergence,do2026muon,chang2026beyond,qian2025mvr,li2026muonb,yang2026nonsmooth,lai2026nonsmooth,parshakova2026muon,liu2026continuous,long2026dimension}.

However, those works derive their analyses from abstract assumptions on the loss, such as smoothness or convexity, imposed uniformly over a problem class rather than from any specific model.  The limitation of those works is that they cannot explain why Muon, a simple matrix preconditioner, works so well \emph{for neural networks specifically}. The few theoretical results for specific models, such as matrix factorization and in-context learning with linear transformers~\citep{ma2026preconditioning}, linear or bilinear classifiers on separable or imbalanced data~\citep{fan2025implicit,vasudeva2025muon}, and random-feature and transformer blocks at initialization~\citep{davis2025spectral}, sharpen the picture in particular regimes. \revtxt{Nevertheless, since the Hessian is the optimal preconditioner under a quadratic approximation of the loss, those previous works leave a basic structural question for preconditioning methods open:} \textbf{What makes the orthogonalization preconditioner suited to the Hessian of a neural network specifically?} \revtxt{\textbf{This paper gives a theoretical answer that Muon's implicit preconditioner shares a similar row-block diagonally dominant heterogeneous structure to the neural-network Hessian described by \citet{zhang2024transformersadam,zhang2024adammini}.}}  

\paragraph{Preconditioner structure versus Hessian structure.}
Let us first take a view of the equivalent preconditioner form of orthogonalization. The orthogonalized update is $\bm{D}_t=(\bm{V}_t\bm{V}_t^\top)^{-1/2}\bm{V}_t$, a left multiplication that mixes the rows of $\bm{V}_t$. Stacking the rows of $\bm{D}_t$ and of $\bm{V}_t$ into vectors turns it into $\mathrm{vec}(\bm{D}_t)=\bigl((\bm{V}_t\bm{V}_t^\top)^{-1/2}\kron\bm{I}_n\bigr)\mathrm{vec}(\bm{V}_t)$, so Muon descends in the metric
\begin{equation}\label{eq:Hmuon}
H_{\mathrm{Muon}}=\bm{P}_t^{1/2}\kron\bm{I}_n,\qquad \bm{P}_t\eqdef\bm{V}_t\bm{V}_t^\top\in\R^{m\times m},
\end{equation}
a Kronecker structure determined entirely by $\bm{P}_t$. At initialization, $\bm{V}_0=\bm{G}_0$ and hence $\bm{P}_0=\bm{G}_0\bm{G}_0^\top$. Previous work has shown that the neural-network Hessian has a special heterogeneous structure: for a weight $\bm{W}\in\R^{m\times n}$, the $m\times m$ blocks indexed by its rows (output neurons) are dominated by the diagonal ones, that is, the layer-wise Hessian is \emph{row-block diagonally dominant}~\citep{zhang2024transformersadam,zhang2024adammini,dong2025towards}. \revtxt{Meanwhile, \citet{deng2026rmnp} point out empirically that $\bm{G}_t\bm{G}_t^\top$ is diagonally dominant throughout training, and the larger the model, the stronger the dominance. When $\bm{P}_t$ is close to diagonal, $H_{\mathrm{Muon}}$ in \eqref{eq:Hmuon} shares the row-block structure of the Hessian (\cref{fig:hessian}), so the question reduces to a clean property of the gradient.}

\begin{figure}[t]
\centering
\includegraphics[width=0.7\textwidth]{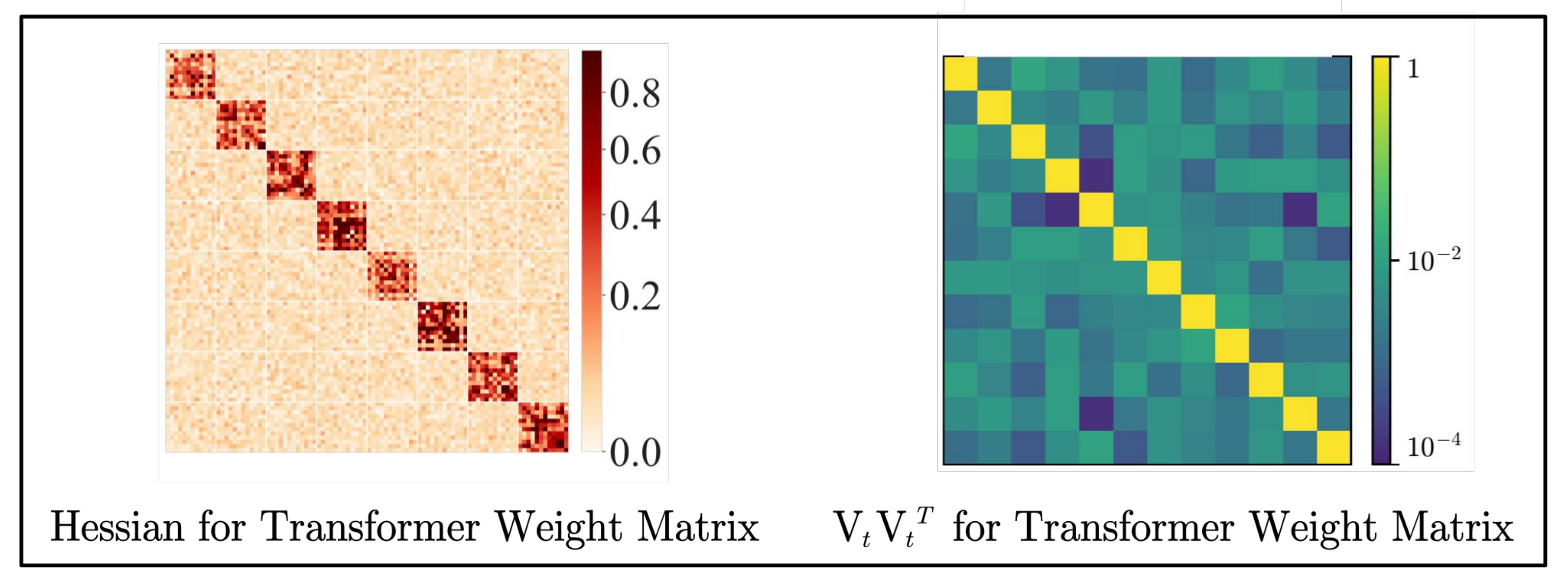}
\vspace{-6pt}
\caption{\revtxt{Left: Hessian of a Transformer weight matrix at initialization, whose mass concentrates on the diagonal blocks. Right: the momentum self outer-product $\bm{V}_t\bm{V}_t^\top$ of a Transformer weight matrix at initialization (log scale), whose mass concentrates on the diagonal.}}
\label{fig:hessian}
\vspace{-12pt}
\end{figure}

\vspace{-4pt}
\begin{center}
\fbox{\begin{minipage}{0.93\linewidth}\centering
\textbf{Question.}\quad \revtxt{How does Neural Networks produce such a diagonally dominant $\bm{V}_t\bm{V}_t^\top$ (equivalently, $\bm{G}_t\bm{G}_t^\top$ at initialization)?}
\end{minipage}}
\end{center}
\vspace{-4pt}

\revtxt{This paper answers this question for neural networks in an asymptotic regime, going a step further than the empirical observation of \citet{deng2026rmnp} by turning it into a quantitative asymptotic theoretical result at initialization.}  We prove that in four canonical neural-network models the diagonal of $\E[\bm{G}\bm{G}^\top]$ outgrows the off-diagonal as the hidden dimensions of the weight matrices grow, at rates summarized in \cref{tab:orders}.
The overparameterization of the weight matrices is thus a key indicator of this diagonalization. Our contributions can be summarized as follows.

\begin{enumerate}[leftmargin=1.4em,itemsep=2pt,topsep=2pt]
\item \textbf{Hidden-dimension asymptotics theory for Muon's preconditioner.} For symmetric matrix factorization, deep linear networks, two-layer ReLU networks and single-layer linear attention, we prove that Muon's implicit preconditioner becomes diagonally dominant as the hidden dimensions grow, with explicit rates (\cref{thm:mf,thm:dl,thm:relu,thm:att} and \cref{tab:orders}).
\item \textbf{Empirical evaluation of the predicted scaling, from simulations to language models.} Simulations of the four models reproduce the predicted rates (\cref{sec:empirical-verification}), and Qwen3.5, Llama-3.2 and Granite-4.0 at initialization show the dominance strengthening with model width (\cref{sec:llm}). 
\end{enumerate}

\begin{figure}[t]
\centering
\includegraphics[width=0.85\textwidth]{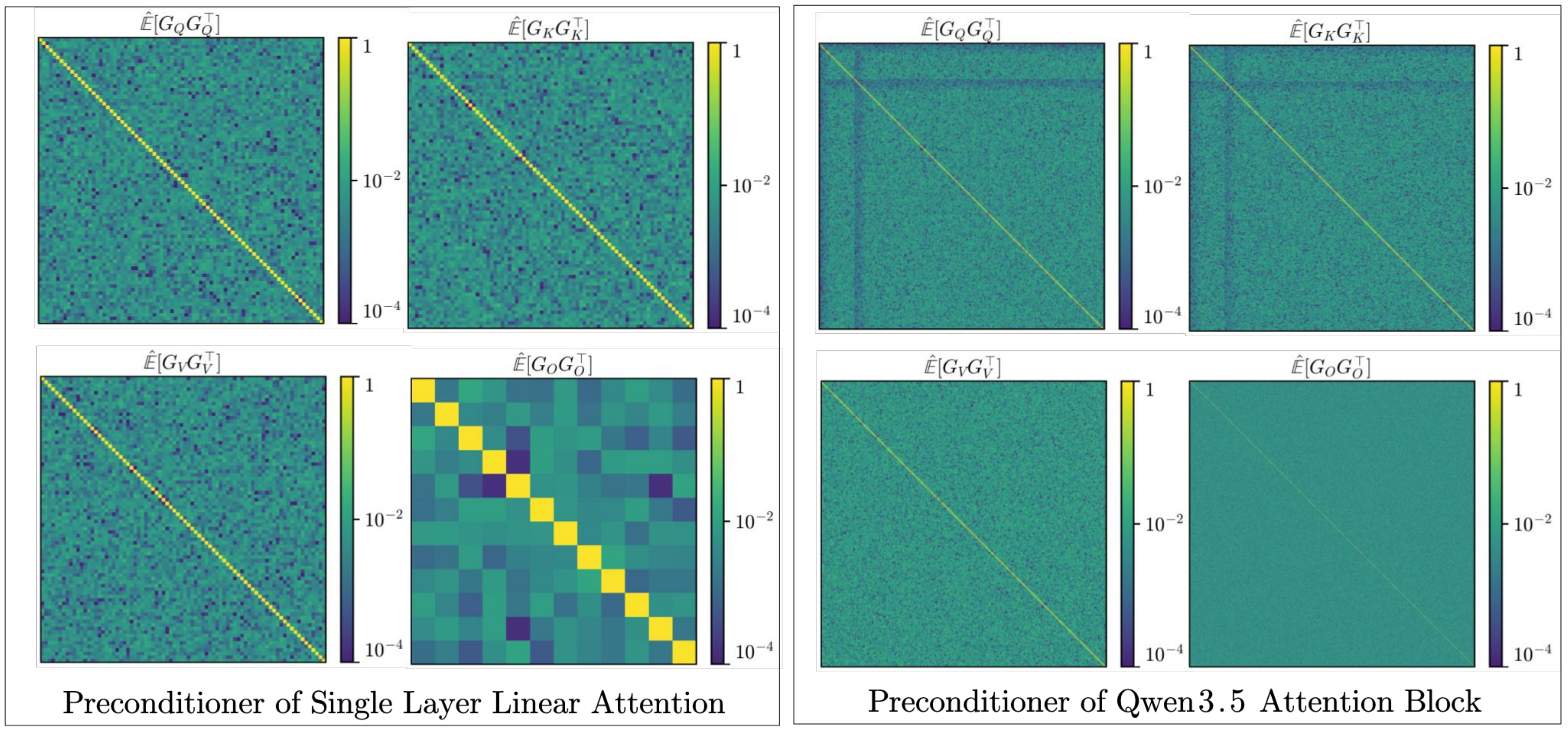}
\caption{Empirical $\widehat{\E}[\bm{G}\bm{G}^\top]$ (log scale) of the query, key, value and output projections at initialization, for single-layer linear attention (left, details in \cref{app:sim-attention}) and the first attention block of Qwen3.5 (right, details in \cref{app:llm-gram-structure}). In both cases the mass concentrates on the diagonal. Each heatmap is normalized by its empirical mean diagonal magnitude.}
\label{fig:main}
\end{figure}

\section{Related Work}\label{sec:related}

\paragraph{Mechanistic analyses of Muon.}
Muon~\citep{jordan2024muon} replaces the momentum of each weight matrix by its polar factor. This is the momentum form of steepest descent in the spectral norm, a view developed by \citet{bernstein2024old,bernstein2024modular}. Work on why this update suits neural networks follows two routes. The first route models the curvature of the loss. \citet{su2025isotropic} introduces an isotropic curvature model and shows that, under a growth condition on the curvature, the optimal update flattens the spectrum of the gradient, as orthogonalization does. \citet{davis2025spectral} give a layerwise condition, comparing the nuclear-to-Frobenius ratio of the gradient with the stable rank of the incoming activations, under which one spectral step decreases the loss more than one Euclidean step. \citet{ma2026preconditioning} prove that a simplified Muon converges linearly on matrix factorization and on in-context learning with linear transformers, at a rate independent of the condition number and ahead of gradient descent and Adam. \citet{nguyen2026flattening} recast Muon as a preconditioned gradient method and show, under a Kronecker-factored curvature model, that spectral flattening improves the convergence factor. \citet{wang2026muon} compare the one-step loss decrease of Muon and Adam under a second-order Taylor approximation and attribute Muon's advantage to a smaller curvature penalty, and \citet{wu2026spectral} probe the loss landscape along real training runs, estimating the loss-optimal step size along each singular direction of the momentum. Since the isotropic geometry behind Muon only approximates the true curvature, several curvature-aware variants have followed~\citep{zhang2026mousse,du2026newtonmuon,che2026second,zhang2026matrix}. The second route studies the update itself. \citet{beneventano2026spectral} characterize the polar factor as the entropy-maximizing bounded update that keeps the singular directions of the gradient and derive its spectral dynamics. \citet{shumaylov2026notspecial} find that random or inverted spectra train equally well, and \citet{wang2025muon} attribute Muon's advantage on heavy-tailed data to its more isotropic update spectrum. \citet{zhang2026polaradamw} separate spectral control from gauge equivariance, \citet{eschenhagen2026clarifying} relate Muon to Shampoo, and \citet{zhang2026grouping} ask at which granularity attention weights should be orthogonalized. Both routes take a curvature model or a landscape as given. None of them compares the structure of Muon's implicit preconditioner with that of the network's Hessian, which is the question we study.

\paragraph{Block structure of the neural-network Hessian.}
The Hessian of a neural network is near block diagonal, with most of its mass in dense principal blocks. This observation goes back to \citet{collobert2004large}, and the same structure is assumed by the block-diagonal curvature approximations of \citet{leroux2008topmoumoute} and \citet{martens2015kfac}. \citet{dong2025towards} quantify the effect at random initialization: for linear and one-hidden-layer classifiers, the off-diagonal blocks vanish relative to the diagonal ones as the number of classes grows. The dense blocks within a layer are indexed by output neurons. This is the row-block structure we build on, and it is the partition by which Adam-mini assigns one learning rate per block to the value, projection and MLP weights~\citep{zhang2024adammini}. Beyond being block diagonal, the Hessian is heterogeneous across parameter blocks. \citet{zhang2024transformersadam} show that in Transformers, unlike CNNs, the Hessian spectrum varies strongly from one weight matrix to another, so that no single learning rate serves all of them. \citet{ormaniec2024transformer} derive the Hessian of a self-attention layer and trace this heterogeneity to the softmax and the query-key parameterization, \citet{tomihari2025gradient} show that gradient heterogeneity across blocks complements it, and \citet{wang2025sharpness} exploit the sharpness disparity between blocks to speed up pre-training. Within a block, \citet{zhang2025concurrence} prove that under anisotropic inputs a layer-wise Kronecker-factored preconditioner recovers the feature-learning rate that SGD and Adam miss. Our work builds on the block-diagonal part of this picture: a preconditioner should be block diagonal in the row-block basis, and Muon's $(\bm{V}\bm{V}^\top)^{1/2}\otimes\bm{I}$ has this form precisely when $\bm{V}\bm{V}^\top$ is diagonal. Further preconditioning methods are discussed in \cref{app:related}.

\section{Preliminaries}\label{sec:prelim}

Throughout, $\bm{W}\in\R^{m\times n}$ denotes a weight matrix, $\bL$ the population loss and $\bm{G}=\partial\bL/\partial\bm{W}$ the gradient of the loss with respect to that matrix. Our object of study is the $m\times m$ matrix $\E[\bm{G}\bm{G}^\top]$ at a Gaussian initialization. This matrix is the expectation of $\bm{P}_0=\bm{V}_0\bm{V}_0^\top$ in \eqref{eq:Hmuon}. We analyze it in four settings. \cref{sec:setup} introduces the four models, and \cref{sec:assump} states the initialization assumed for each. We write $\bm{I}_d$ for the $d\times d$ identity, $\tr$ for the trace and $\norm{\cdot}_\Frob$ for the Frobenius norm. The ReLU is $\sigma(z)=\max(z,0)$, with derivative $\sigma'(z)=\1[z>0]$, and $\kron$ is the Kronecker product. Finally, $a=\Theta(b)$ means that $c\,b\le a\le C\,b$ for some constants $0<c\le C$.

\subsection{Setup}\label{sec:setup}

\noindent\fbox{\begin{minipage}{\dimexpr\linewidth-2\fboxsep-2\fboxrule\relax}
\paragraph{Symmetric matrix factorization.}
For positive integers $d,k,r$ with $k\ge r$, a weight $\bm{W}\in\R^{d\times k}$ fits a fixed target $\Mstar\in\R^{d\times d}$, symmetric positive semidefinite of rank $r$, under the over-parameterized loss
\begin{equation}\label{eq:mf-setup}
\bL(\bm{W})\eqdef\tfrac14\norm{\bm{W}\bm{W}^\top-\Mstar}_\Frob^2,
\qquad
\bm{G}\eqdef\frac{\partial\bL}{\partial\bm{W}}=(\bm{W}\bm{W}^\top-\Mstar)\,\bm{W}\in\R^{d\times k}.
\end{equation}
The hidden dimension is $k$.
\end{minipage}}
\medskip

\noindent\fbox{\begin{minipage}{\dimexpr\linewidth-2\fboxsep-2\fboxrule\relax}
\paragraph{Deep linear network.}
For $L\ge1$ and layer dimensions $d_0,\dots,d_L$, weights $\bm{W}_i\in\R^{d_i\times d_{i-1}}$ fit a fixed linear target $\bm{\Phi}\in\R^{d_L\times d_0}$ on inputs $\bm{x}\sim\bN(\bm{0},\bm{I}_{d_0})$,
\begin{equation}\label{eq:dl-setup}
f(\bm{x})=\bm{W}_L\cdots\bm{W}_1\bm{x},
\qquad
\bL(\bm{W}_1,\dots,\bm{W}_L)\eqdef\tfrac12\,\E_{\bm{x}}\norm{f(\bm{x})-\bm{\Phi}\bm{x}}^2 .
\end{equation}
With $\bm{B}_i\eqdef\bm{W}_{i-1}\cdots\bm{W}_1$ ($\bm{B}_1\eqdef\bm{I}$), $\bm{A}_i\eqdef\bm{W}_L\cdots\bm{W}_{i+1}$ ($\bm{A}_L\eqdef\bm{I}$) and $\bm{R}\eqdef\bm{W}_L\cdots\bm{W}_1-\bm{\Phi}$, the layer-$i$ gradient is
\begin{equation}\label{eq:dl-grad}
\bm{G}_i\eqdef\frac{\partial\bL}{\partial\bm{W}_i}=\bm{A}_i^\top\,\bm{R}\,\bm{B}_i^\top\in\R^{d_i\times d_{i-1}}.
\end{equation}
The hidden dimensions are $d_1,\dots,d_{L-1}$.
\end{minipage}}
\medskip

\noindent\fbox{\begin{minipage}{\dimexpr\linewidth-2\fboxsep-2\fboxrule\relax}
\paragraph{Two-layer ReLU network.}
For positive $d_0,d_1,d_2$, weights $\bm{W}_1\in\R^{d_1\times d_0}$ and $\bm{W}_2\in\R^{d_2\times d_1}$ fit a fixed linear target $\bm{\Phi}\in\R^{d_2\times d_0}$ on inputs $\bm{x}\sim\bN(\bm{0},\bm{I}_{d_0})$,
\begin{equation}\label{eq:relu-setup}
f(\bm{x})=\bm{W}_2\,\sigma(\bm{W}_1\bm{x}),
\qquad
\bL(\bm{W}_1,\bm{W}_2)\eqdef\tfrac12\,\E_{\bm{x}}\norm{f(\bm{x})-\bm{\Phi} \bm{x}}^2 .
\end{equation}
With $\bm{h}(\bm{x})\eqdef\sigma(\bm{W}_1\bm{x})$ the hidden features, $\bm{r}(\bm{x})\eqdef f(\bm{x})-\bm{\Phi}\bm{x}$ the residual and $\bm{D}(\bm{x})\eqdef\diag(\sigma'(\bm{W}_1\bm{x}))$, the gradients $\bm{G}_i\eqdef\partial\bL/\partial\bm{W}_i$ are
\begin{equation}\label{eq:relu-grad}
\bm{G}_1=\E_{\bm{x}}\bigl[\bm{D}(\bm{x})\bm{W}_2^\top\bm{r}(\bm{x})\bm{x}^\top\bigr]\in\R^{d_1\times d_0},
\qquad
\bm{G}_2=\E_{\bm{x}}\bigl[\bm{r}(\bm{x})\bm{h}(\bm{x})^\top\bigr]\in\R^{d_2\times d_1}.
\end{equation}
Let $\bm{w}_a$ denote the $a$-th row of $\bm{W}_1$. The hidden dimension is $d_1$.
\end{minipage}}
\medskip

\begin{addblock}
\noindent\fbox{\begin{minipage}{\dimexpr\linewidth-2\fboxsep-2\fboxrule\relax}
\paragraph{Single-layer linear attention.}
For positive $T,d_0,d_k,d_v,d_2$, tokens $\bm{X}\in\R^{T\times d_0}$ (one per row) pass through weights $\WQ,\WK\in\R^{d_k\times d_0}$, $\WV\in\R^{d_v\times d_0}$, $\WO\in\R^{d_2\times d_v}$, which fit a fixed linear target $\bm{\Phi}\in\R^{d_2\times d_0}$,
\begin{equation}\label{eq:att-layer}
\begin{gathered}
\bm{A}\eqdef\tfrac{1}{\sqrt{d_k}}\,\bm{X}\WQ^\top\WK\bm{X}^\top,
\qquad
f(\bm{X})=\bm{A}\,\bm{X}\WV^\top\WO^\top,\\
\bL(\WQ,\WK,\WV,\WO)\eqdef\tfrac12\norm{f(\bm{X})-\bm{X}\bm{\Phi}^\top}_\Frob^2 .
\end{gathered}
\end{equation}
Let $\bm{C}\eqdef\bm{X}^\top\bm{X}$, $c_p\eqdef\tr(\bm{C}^p)$, $\bm{Z}\eqdef\bm{X}\WV^\top$ and $\bm{R}\eqdef f(\bm{X})-\bm{X}\bm{\Phi}^\top$. The gradients $\bm{G}_\bullet\eqdef\partial\bL/\partial\bm{W}_\bullet$, $\bullet\in\{Q,K,V,O\}$, are
\begin{equation}\label{eq:att-grads}
\begin{aligned}
\GQ&=\tfrac{1}{\sqrt{d_k}}\,\WK\bm{X}^\top\bm{Z}\WO^\top\bm{R}^\top\bm{X},
&\qquad
\GK&=\tfrac{1}{\sqrt{d_k}}\,\WQ\bm{X}^\top\bm{R}\WO\bm{Z}^\top\bm{X},\\
\GV&=\WO^\top\bm{R}^\top\bm{A}\bm{X},
&\qquad
\GO&=\bm{R}^\top\bm{A}\bm{Z} .
\end{aligned}
\end{equation}
The hidden dimensions are $d_k$ and $d_v$. In most practical language models the input and output dimensions of an attention block both equal the model width, $d_0=d_2$~\citep{vaswani2017attention,radford2019gpt2,brown2020gpt3,touvron2023llama,grattafiori2024llama3,llama32,qwen35,granite4}, and scaling a model up increases this width while $d_k,d_v$ stay fixed, so for attention we also track $d_0$ and $d_2$.
\end{minipage}}
\end{addblock}

\subsection{Assumptions}\label{sec:assump}

Each setting is analyzed at a standard Gaussian initialization, stated separately below.

\begin{assumption}[Matrix factorization]\label{ass:mf}
The entries of $\bm{W}$ are i.i.d.\ $\bN(0,1)$, and $\Mstar$ is deterministic and independent of $\bm{W}$. Hence $\bm{S}\eqdef\bm{W}\bm{W}^\top\sim\mathcal{W}_d(k,\bm{I}_d)$ is Wishart with $k$ degrees of freedom.
\end{assumption}

\begin{assumption}[Deep linear network]\label{ass:dl}
The entries of $\bm{W}_1,\dots,\bm{W}_L$ are i.i.d.\ $\bN(0,1)$; the input $\bm{x}\sim\bN(\bm{0},\bm{I}_{d_0})$ is integrated out in $\bL$; and $\bm{\Phi}$ is deterministic and independent of the weights. Expectations are over $\bm{W}_1,\dots,\bm{W}_L$.
\end{assumption}

\begin{assumption}[Two-layer ReLU network]\label{ass:relu}
The entries of $\bm{W}_1$ and $\bm{W}_2$ are i.i.d.\ $\bN(0,1)$; the input $\bm{x}\sim\bN(0,\bm{I}_{d_0})$ is integrated out in $\bL$; and $\bm{\Phi}$ is deterministic and independent of $\bm{W}_1,\bm{W}_2$. Expectations are over $\bm{W}_1,\bm{W}_2$.
\end{assumption}

\begin{addblock}
\begin{assumption}[Linear attention]\label{ass:att}
$\WQ,\WK,\WV,\WO$ are mutually independent with i.i.d.\ $\bN(0,1)$ entries; the rows of $\bm{X}$ are i.i.d.\ $\bN(\bm{0},\bm{I}_{d_0})$ and independent of the weights, so that $\bm{C}\sim\mathcal{W}_{d_0}(T,\bm{I}_{d_0})$ is Wishart; and $\bm{\Phi}$ is deterministic. Expectations are over the weights and the tokens jointly.
\end{assumption}
\end{addblock}

\section{Main Results}\label{sec:main}

We state the closed form of $\E[\bm{G}\bm{G}^\top]$ in each setup, discuss its diagonal dominance, and defer the proofs to \addtxt{\cref{app:mf,app:dl,app:relu,app:att}}.

\subsection{Symmetric matrix factorization}
\begin{theorem}\label{thm:mf}
Under \cref{ass:mf}, with $\bm{S}\eqdef\bm{W}\bm{W}^\top\sim\mathcal{W}_d(k,\bm{I}_d)$ Wishart,
\begin{equation}\label{eq:mf-main}
\E[\bm{G}\bm{G}^\top]=\alpha_3\,\bm{I}_d-2\alpha_2\,\Mstar+k\,(\Mstar)^2,
\end{equation}
where $\alpha_2=k(k+d+1)$ and $\alpha_3=k^3+3(d+1)k^2+(d^2+3d+4)k$ are the (isotropic) diagonal entries of $\E[\bm{S}^2]$ and $\E[\bm{S}^3]$. \revtxt{Consequently, every diagonal entry $\alpha_3-2\alpha_2\Mstar_{ii}+k(\Mstar)_{ii}^2$ is $\Theta(k^3)$, and every off-diagonal entry $-2\alpha_2\Mstar_{ij}+k(\Mstar)_{ij}^2$ with $\Mstar_{ij}\neq0$ is $\Theta(k^2)$.}
\end{theorem}
\revtxt{As the hidden dimension $k$ grows, the ratio of diagonal to off-diagonal entries is therefore $\Theta(k)$, and $\E[\bm{G}\bm{G}^\top]$ approaches the isotropic matrix $\alpha_3\bm{I}_d$. At initialization the preconditioner thus reduces, to leading order, to a uniform rescaling of the momentum, and the target enters only through the lower-order terms $-2\alpha_2\Mstar+k(\Mstar)^2$.}

\subsection{Deep linear network}
\begin{theorem}\label{thm:dl}
Under \cref{ass:dl}, for every layer $i\in\{1,\dots,L\}$,
\begin{equation}\label{eq:dl-main}
\E[\bm{G}_i\bm{G}_i^\top]=V_i\,s_i\,\bm{I}_{d_i}+\bm{T}_i,\qquad
\bm{T}_i=\begin{cases}\bigl(\textstyle\prod_{j=1}^{L-1}d_j\bigr)\bm{\Phi}\bm{\Phi}^\top,&i=L,\\[3pt]
\bigl(\textstyle\prod_{j\neq i,\,1\le j\le L-1}d_j\bigr)\norm{\bm{\Phi}}_\Frob^2\,\bm{I}_{d_i},&1\le i<L,\end{cases}
\end{equation}
where $V_i,s_i$ are positive dimension-only scalars given by two-step recursions in $(d_0,\dots,d_L)$ (\cref{app:dl}), with
\begin{equation}\label{eq:dl-order}
V_is_i=\begin{cases}\Theta\bigl(d_i\prod_{j=1,\,j\neq i}^{L-1}d_j^2\bigr),&1\le i<L,\\[3pt]
\Theta\bigl(\prod_{j=1}^{L-1}d_j^2\bigr),&i=L.\end{cases}
\end{equation}
\end{theorem}
For a hidden layer ($i<L$) the off-diagonal block $\bm{T}_i$ is a multiple of $\bm{I}_{d_i}$, so $\E[\bm{G}_i\bm{G}_i^\top]$ is exactly diagonal. At the output layer the diagonal noise $V_Ls_L=\Theta\bigl(\prod_{j=1}^{L-1}d_j^2\bigr)$ grows with the hidden dimensions while the off-diagonal stays proportional to $\bigl(\prod_{j=1}^{L-1}d_j\bigr)(\bm{\Phi}\bm{\Phi}^\top)_{ab}$, so the ratio, of order $\prod_{j=1}^{L-1}d_j$, still diverges. \revtxt{For equal hidden dimensions $d_1=\dots=d_{L-1}=d$ the output-layer diagonal is $\Theta(d^{2(L-1)})$ against an off-diagonal of $\Theta(d^{L-1})$.}

\subsection{Two-layer ReLU network}
Let $\theta\eqdef\angle(\bm{u},\bm{v})$ for independent $\bm{u},\bm{v}\sim\bN(0,\bm{I}_{d_0})$, and set the constants $\kappa_1\eqdef\E\bigl[\tfrac{\sin^2\theta+(\pi-\theta)^2}{4\pi^2}+\tfrac{\sin\theta\cos\theta(\pi-\theta)}{2\pi^2}\bigr]$ and $\kappa_2\eqdef\tfrac{\E[\sin\theta(\pi-\theta)]}{2\pi^2}$.
\begin{theorem}\label{thm:relu}
Under \cref{ass:relu}, with $\bm{H}\eqdef\E_{\bm{x}}[\bm{h}(\bm{x})\bm{h}(\bm{x})^\top]$ and $\alpha\eqdef\E_{\bm{W}_1}[\tr(\bm{H}^2)]$,
\[
\E[\bm{G}_2\bm{G}_2^\top]=\alpha\,\bm{I}_{d_2}+\tfrac{d_1}{4}\bm{\Phi}\bm{\Phi}^\top,\qquad\alpha=\Theta(d_1^2),
\]
and $\E[\bm{G}_1\bm{G}_1^\top]$ is permutation-invariant in the hidden index with diagonal entries
\[
\E[(\bm{G}_1\bm{G}_1^\top)_{aa}]=\tfrac{d_0d_2^2}{4}+d_0d_2\bigl[(d_1-1)\kappa_1+\tfrac12\bigr]+\tfrac{\norm{\bm{\Phi}}_\Frob^2}{4}
\]
and off-diagonal entries $\E[(\bm{G}_1\bm{G}_1^\top)_{ac}]=d_0d_2\,\kappa_2+O(d_2\sqrt{d_0})$ for $a\neq c$. \revtxt{Consequently, for $\bm{G}_1$ the diagonal entries are $\Theta(d_1)$ and the off-diagonal entries $\Theta(1)$, and for $\bm{G}_2$ they are $\Theta(d_1^2)$ and $\Theta(d_1)$.}
\end{theorem}
\revtxt{As the hidden dimension $d_1$ grows, the ratio of diagonal to off-diagonal entries is therefore $\Theta(d_1)$ for both layers. For $\bm{G}_2$ the off-diagonal $\tfrac{d_1}{4}(\bm{\Phi}\bm{\Phi}^\top)_{ac}$ comes only from the target, as at the output layer of \cref{thm:dl}. For $\bm{G}_1$ the off-diagonal $d_0d_2\kappa_2$ is target-free:} it comes from a single term in the proof and traces back to ReLU's lack of odd symmetry, which is why the dominance here is asymptotic rather than exact. Simulations match every formula (\cref{app:sim-relu}).

\begin{addblock}
\subsection{Single-layer linear attention}
\begin{theorem}\label{thm:att}
Under \cref{ass:att},
\begin{equation}\label{eq:att-main}
\begin{gathered}
\E[\GQ\GQ^\top]=c_Q\,\bm{I}_{d_k},\qquad
\E[\GK\GK^\top]=c_K\,\bm{I}_{d_k},\qquad
\E[\GV\GV^\top]=c_V\,\bm{I}_{d_v},\\
\E[\GO\GO^\top]=\bar\alpha_O\,\bm{I}_{d_2}+d_v\,\gamma\;\bm{\Phi}\bm{\Phi}^\top,\\
\gamma=\tfrac1{d_0}\E[c_2^2]=T\bigl[(Td_0{+}8)(T{+}d_0)^2+2(Td_0{+}10)(T{+}d_0)+5(Td_0{+}4)\bigr],
\end{gathered}
\end{equation}
where $c_Q,c_K,c_V,\bar\alpha_O$ are explicit positive polynomials in the dimensions and $\norm{\bm{\Phi}}_\Frob^2$ (\cref{app:att}); for fixed $\bm{\Phi}$,
\[
c_Q\asymp c_K=\Theta\bigl(d_2^2d_v^2d_0^6/d_k\bigr),\qquad c_V=\Theta\bigl(d_2(d_2{+}d_v)\,d_0^6\bigr),\qquad \bar\alpha_O=\Theta\bigl(d_v^2d_0^6\bigr).
\]
\end{theorem}
For $\WQ,\WK,\WV$ the off-diagonal is exactly zero (\cref{lem:atta-iso}). For $\WO$, as at the output layer of \cref{thm:dl}, the off-diagonal $d_v\gamma(\bm{\Phi}\bm{\Phi}^\top)_{ab}=\Theta(d_vd_0^3)$ comes only from the target while the diagonal $\bar\alpha_O=\Theta(d_v^2d_0^6)$ is target-free noise, so the ratio $\Theta(d_vd_0^3)$ diverges with $d_v$ and, faster, with $d_0$; $d_k$ does not affect it. With $d_0=d_2$ equal to the model width (\cref{sec:setup}), the ratio grows like the cube of the model width (\cref{sec:llm}).
\end{addblock}

\subsection{Summary and proof sketch}
In all \addtxt{four} settings the diagonal of $\E[\bm{G}\bm{G}^\top]$ outgrows the off-diagonal as the hidden dimension grows (\cref{tab:orders}), so $\bm{P}_t=\bm{V}_t\bm{V}_t^\top$ is asymptotically diagonal, $H_{\mathrm{Muon}}=\bm{P}_t^{1/2}\kron\bm{I}_n$ is asymptotically block-diagonal in the Hessian's row-block structure, and (\cref{sec:equiv}) orthogonalization is asymptotically a row-wise normalization.

\begin{table}[H]
\centering
\fbox{\begin{minipage}{\dimexpr\linewidth-2\fboxsep-2\fboxrule\relax}\smallskip\centering
\small
\renewcommand{\arraystretch}{1.25}\setlength{\tabcolsep}{6pt}
\begin{tabular}{@{}llll@{}}
\toprule
Model & Diagonal & Off-diagonal & Asymptotic dims\\
\midrule
Matrix factorization & $\Theta(k^3)$ & $\Theta(k^2)$ & $k$\\
Deep linear, hidden layer $i$ & $\Theta\bigl(d_i\prod_{j=1,\,j\neq i}^{L-1}d_j^2\bigr)$ & $0$ & $d_1,\dots,d_{L-1}$\\
Deep linear, output layer & $\Theta\bigl(\prod_{j=1}^{L-1}d_j^2\bigr)$ & $\Theta\bigl(\prod_{j=1}^{L-1}d_j\bigr)$ & $d_1,\dots,d_{L-1}$\\
Two-layer ReLU, layer $\bm{W}_1$ & $\Theta(d_1)$ & $\Theta(1)$ & $d_1$\\
Two-layer ReLU, layer $\bm{W}_2$ & $\Theta(d_1^2)$ & $\Theta(d_1)$ & $d_1$\\
\addtxt{Linear attention, $\WQ,\WK$} & \addtxt{$\Theta(d_2^2d_v^2d_0^6/d_k)$} & \addtxt{$0$} & \addtxt{$d_v,d_0,d_2$}\\
\addtxt{Linear attention, $\WV$} & \addtxt{$\Theta(d_2(d_2{+}d_v)d_0^6)$} & \addtxt{$0$} & \addtxt{$d_v,d_0,d_2$}\\
\addtxt{Linear attention, $\WO$} & \addtxt{$\Theta(d_v^2d_0^6)$} & \addtxt{$\Theta(d_vd_0^3)$} & \addtxt{$d_v,d_0$}\\
\bottomrule
\end{tabular}
\smallskip\end{minipage}}
\caption{Summary of the orders of $\E[\bm{G}\bm{G}^\top]$ at a Gaussian initialization in the hidden dimensions of each model and, for attention, the token dimensions $d_0,d_2$.}
\label{tab:orders}
\end{table}

\begin{revblock}
The four proofs share a similar strategy. Each row of $\bm{G}$ is a random vector, and the $(a,b)$ entry of $\bm{G}\bm{G}^\top$ is the inner product of rows $a$ and $b$, a measure of their correlation. A diagonal entry is the correlation of a row with itself, its squared norm, which adds up over all hidden coordinates. An off-diagonal entry is the correlation between two different rows, which are built from independent Gaussian vectors at initialization, so most of its terms cancel in expectation and only the lower-order terms carried by the fixed target remain. The gap between the orders of the diagonal and off-diagonal entries in \cref{tab:orders} comes from this difference. In the proofs we compute these statistics of $\bm{G}\bm{G}^\top$ exactly, using moments of the Wishart distribution for matrix factorization, deep linear networks and linear attention, and Stein's lemma together with the arc-cosine kernel of \citet{cho2009kernel} for the ReLU network, whose nonlinearity is not a polynomial. Each closed form is reproduced by the simulations in \cref{sec:empirical-verification}.
\end{revblock}

\section{Simulation of Theoretical Results}
\label{sec:empirical-verification}

\paragraph{Simulation setup.}
We verify the dimension scalings of \cref{thm:mf,thm:dl,thm:relu,thm:att} under \cref{ass:mf,ass:dl,ass:relu,ass:att} by varying the model dimensions. For each model, we vary the
dimensions in the corresponding theoretical results and examine the
resulting structure of the expected gradient self outer-product $\bm{P}\eqdef\E[\bm{G}\bm{G}^\top]\in\R^{m\times m}$, that is, the expectation of $\bm{P}_0$ in \eqref{eq:Hmuon}. For this matrix, we summarize the mean diagonal magnitude, the mean off-diagonal magnitude, and their ratio,
\begin{equation}
D(\bm{P})=\frac{1}{m}\sum_{i=1}^{m}|P_{ii}|,\qquad
\mathrm{OD}(\bm{P})=\frac{1}{m(m-1)}\sum_{i\neq j}|P_{ij}|,\qquad
R(\bm{P})=\frac{D(\bm{P})}{\mathrm{OD}(\bm{P})},
\label{eq:main-DOR}
\end{equation}
We estimate $\bm{P}$ by $\widehat{\E}[\bm{G}\bm{G}^\top]=\frac1S\sum_{s=1}^{S}\bm{G}^{(s)}\bm{G}^{(s)\top}$ over $S$ independent initializations; $\widehat D,\widehat{\mathrm{OD}},\widehat R$ denote these estimates. Where the expected off-diagonal is exactly zero, the nonzero $\widehat{\mathrm{OD}}$ is Monte Carlo variance (details in \cref{app:empirical-verification}). We next apply this procedure to matrix factorization, deep linear networks, two-layer ReLU networks, and single-layer linear attention, varying in each
case the dimensions governing the corresponding theoretical scaling.

\paragraph{Symmetric matrix factorization.}
\label{sec:sim-mf}
We begin with symmetric matrix factorization, varying $k\in\{5,\allowbreak 10,\allowbreak 20,\allowbreak 40,\allowbreak 80,\allowbreak 160,\allowbreak 320,\allowbreak 640\}$ with $d=25$ and a fixed rank-5 target, using $S=10{,}000$ initializations for each $k$. In \cref{fig:main-mf-width-scaling} (left), $\widehat R$ increases with $k$, in line with $R=\Theta(k)$ from \cref{thm:mf}; $\widehat D$ and $\widehat{\mathrm{OD}}$ follow $\Theta(k^3)$ and $\Theta(k^2)$ (\cref{app:sim-mf}).  

\begin{figure}[tb]
    \centering
    \includegraphics[width=\linewidth]{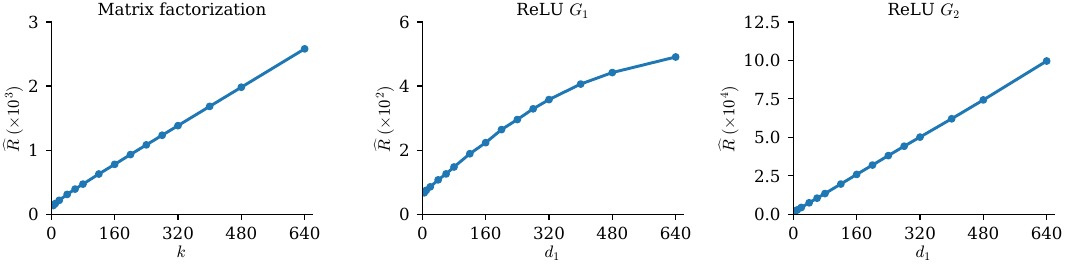}
    \caption{Diagonal-dominance ratio $\widehat R$ for symmetric matrix factorization against $k$ (left; $d=25$, rank-5 target) and for the two-layer ReLU network against $d_1$ (right; $d_0=512$, $d_2=12$).}
    \label{fig:main-mf-width-scaling}
\end{figure}

\paragraph{Deep linear network.}
\label{sec:empirical-dl}
We next consider a depth-$4$ linear network with baseline dimensions $(d_0,\dots,d_4)=(512,80,80,80,512)$. We vary each hidden
dimension $d_1,d_2,d_3$ in turn over the same grid while holding the remaining dimensions fixed, using $S=10{,}000$ initializations per
configuration. In \cref{fig:main-dl-dimension-scaling}, $\widehat R$ increases
with every hidden dimension for all layers. For the output layer, \cref{thm:dl} gives
$D=\Theta((d_1d_2d_3)^2)$ and
$\mathrm{OD}=\Theta(d_1d_2d_3)$, and hence
$R=\Theta(d_1d_2d_3)$; accordingly, $\widehat R$ grows approximately linearly when any one hidden dimension is varied. For the hidden layers, the expected off-diagonal is zero (\cref{app:sim-dl}).

\begin{figure}[tb]
    \centering
    \includegraphics[width=0.8\linewidth]{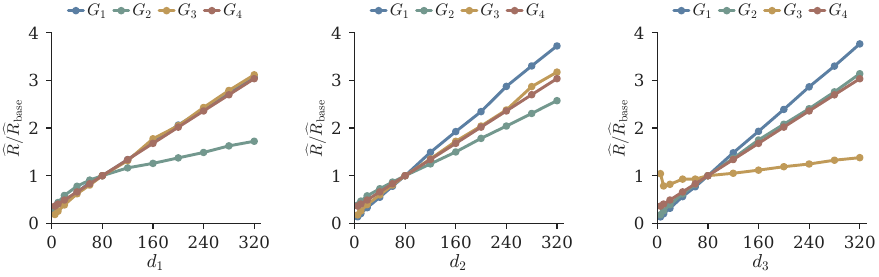}
    \caption{Diagonal-dominance ratio $\widehat R$ of $\bm{G}_1,\dots,\bm{G}_4$ for the depth-$4$ linear network when $d_1$, $d_2$ or $d_3$ is varied (left to right). Each curve is normalized by its value at the baseline
    ($d_1=d_2=d_3=80$) to facilitate comparison across layers with different absolute scales.}
    \label{fig:main-dl-dimension-scaling}
\end{figure}

\paragraph{Two-layer ReLU network.}
\label{sec:empirical-relu}
We then study the two-layer ReLU network, fixing $d_0=512$ and $d_2=12$ and varying the hidden width $d_1$ over the same grid. For each
configuration, we use $S=10{,}000$ weight initializations and estimate the population gradient using $128{,}000$ Gaussian inputs. In \cref{fig:main-mf-width-scaling} (right), $\widehat R$ increases with $d_1$ for both layers, matching $R=\Theta(d_1)$ in \cref{thm:relu} (\cref{app:sim-relu}).

\paragraph{Single-layer linear attention.}
\label{sec:empirical-attention}
Finally, we consider single-layer linear attention with $T=32$, $d_0=512$, and $d_2=12$. We examine four dimension sweeps: $d_k=d_v$ jointly, $d_k$ with $d_v=80$, $d_v$ with $d_k=80$, and $d_0=d_2$ jointly. In \cref{fig:main-attention-dimension-scaling}, $\widehat R$ of $\GO$ grows approximately linearly with $d_v$, is nearly unchanged with $d_k$, and grows superlinearly with $d_0$, consistent with $R=\Theta(d_vd_0^3)$ from \cref{thm:att}. For $\GQ,\GK,\GV$, the expected off-diagonal entries are zero. Their finite $\widehat R$ values nevertheless follow the corresponding diagonal trends: they increase with $d_k=d_v$, increase with $d_v$,
and grow strongly with $d_0=d_2$, while $\GQ$ and $\GK$ decrease with $d_k$, consistent with
$D=\Theta(d_2^2d_v^2d_0^6/d_k)$ (\cref{app:sim-attention}).

\begin{figure}[tb]
    \centering
    \includegraphics[width=\linewidth]{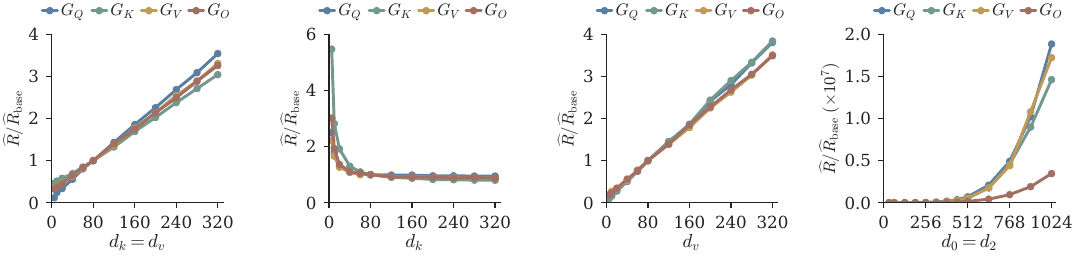}
    \caption{Diagonal-dominance ratio $\widehat R$ of the four attention projections when $d_k=d_v$, $d_k$, $d_v$, and $d_0=d_2$ are varied (left to right). Each curve is normalized by its value at the baseline (first three: $T=32$, $d_0=512$, $d_k=d_v=80$, $d_2=12$; fourth: $T=16, d_0=d_2=32, d_k=d_v=20$) to
    facilitate comparison across projections with different absolute scales.}
    \label{fig:main-attention-dimension-scaling}
\end{figure}

\section{Empirical Evaluation on Large Language Models}
\label{sec:llm}

We next test whether the diagonal dominance of \cref{thm:att} persists in Transformer-based language models at random initialization: Qwen3.5-0.8B-Base~\citep{qwen35}, Granite-4.0-1B-Base~\citep{granite4} and Llama-3.2-1B~\citep{llama32}. \Cref{fig:main} (right) shows the empirical $\widehat{\E}[\bm{G}\bm{G}^\top]$ of the four projections in Qwen3.5, whose mass concentrates on the diagonal. For the first head of the first eligible attention block of each model we vary the model width $d_0=d_2$ at fixed head dimension $d_k=d_v$, using $25$ random initializations and $64$ fixed FineWeb-Edu~\citep{penedo2024the} sequences per configuration (details in \cref{app:llm-width-scaling}). \Cref{fig:llm-d0d2-scaling} shows $\widehat R$ increasing with model width for all four projections in all three models, as the $\Theta(d_vd_0^3)$ ratio of \cref{thm:att} predicts.

\begin{figure}[tb]
    \centering
    \includegraphics[width=0.95\linewidth]
    {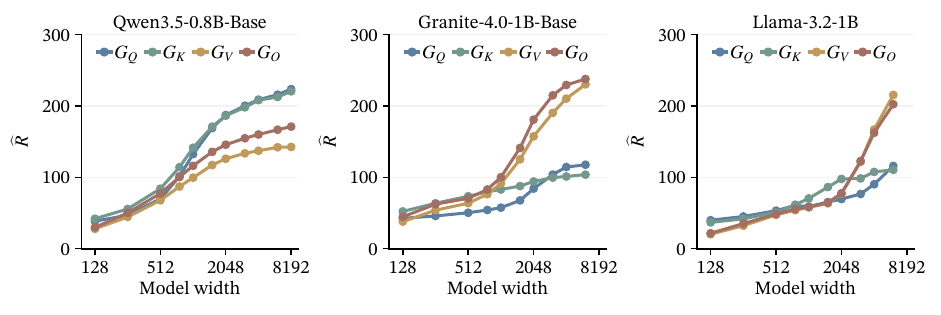}
    \caption{
    Model width scaling of \(\widehat R\) for the query, key, value, and output projections in the first head of the first eligible
    attention block of the three language models at random initialization.
    }
    \label{fig:llm-d0d2-scaling}
\end{figure}

\section{Conclusion and Limitations}\label{sec:conclusion}
\paragraph{Conclusion.}
This work provides an asymptotic theory that shows Muon's implicit preconditioner matches the row-block diagonally dominant structure of the neural-network Hessian, because $\E[\bm{G}\bm{G}^\top]$ itself becomes diagonally dominant as the weight matrices are overparameterized, at the explicit rates of \cref{tab:orders}. Specifically, in matrix factorization the isotropic $\Theta(k^3)$ diagonal outgrows the target-driven $\Theta(k^2)$ off-diagonal. In deep linear networks the hidden layers are exactly diagonal, and the ratio at the output layer grows with the product of all hidden dimensions. In the two-layer ReLU network the nonlinearity leaves an off-diagonal of constant order, yet the ratio still grows linearly with the hidden dimension. In linear attention the query, key and value weights are exactly diagonal. For the output weight the ratio $\Theta(d_vd_0^3)$ grows with the head dimension and grows as the cube of the model width, so the dominance strengthens as models get wider. Our experiments on language models of increasing width match this intuition.

\paragraph{Limitations.}
Our theory is stated at initialization and in expectation: it does not characterize the training dynamics of $\bm{V}_t\bm{V}_t^\top$, although the empirical evidence of \citet{deng2026rmnp} indicates that the dominance persists during training. It is also worth studying whether this dominance reflects a degenerate behavior of the model that emerges at large scale. We leave these questions to future work.

\subsection*{Ethics statement}
This work is primarily theoretical: it characterizes the expected gradient self outer-product of standard neural-network models at initialization. The experiments consist of synthetic simulations and training runs of publicly available language models on standard public corpora; they involve no human subjects, no personal or sensitive data, and no newly collected datasets. We do not foresee ethical concerns or direct negative societal impacts arising from this work.

\subsection*{Reproducibility statement}
All results are stated under explicit assumptions (\cref{sec:assump}) and proved in full in the appendix: Appendix~\ref{app:mf} (symmetric matrix factorization), Appendix~\ref{app:dl} (deep linear networks), Appendix~\ref{app:relu} (two-layer ReLU networks) and Appendix~\ref{app:att} (single-layer linear attention). Each appendix restates the corresponding theorem together with its standing assumption and all supporting lemmas, and the closed-form Wishart trace moments used for the attention layer are listed explicitly in Lemma~\ref{lem:atta-wm}. The numerical checks (Appendix~\ref{app:empirical-verification}) specify every dimension, sample size and initialization used, and the language-model experiments report the model sizes, datasets, optimizer settings and number of training steps; the code for both the simulations and the language-model runs is provided at the anonymous repository \url{https://anonymous.4open.science/r/orthogonalization-hessian-CB2C/}.

\bibliography{sample}
\bibliographystyle{iclr2027_conference}

\appendix
\crefalias{section}{appendix}
\crefalias{subsection}{appendix}

\setcounter{tocdepth}{2}
\tableofcontents
\bigskip

\section{Extended Related Work}\label{app:related}

Mechanistic analyses of Muon and the heterogeneity of the neural-network Hessian are discussed in \cref{sec:related}; this appendix covers the remaining line of work.

\paragraph{Preconditioning methods for neural-network optimization.}
Muon belongs to the family of optimizers defined by a linear minimization oracle (LMO) over a norm ball~\citep{pethick2025training,chen2025muon,kovalev2025understanding}, which also contains Scion and Gluon, unifies sign, spectral and Schatten-$p$ updates~\citep{lau2025polargrad,an2025asgo, pang2026htmuonimprovingmuonheavytailed, pethick2026use}, and comes with convergence theory for the idealized and the practical algorithm~\citep{li2025note,shen2025convergence,kovalev2026non}. Two-sided Kronecker preconditioners such as Shampoo and SOAP whiten both the row and the column side of a matrix gradient~\citep{gupta2018shampoo,morwani2024shampoo,vyas2024soap}; \citet{anil2020scalable} make this second-order preconditioning practical at the scale of modern deep learning, and \citet{anil2019memory} go in the opposite direction, compressing the state of adaptive optimizers to a memory-efficient factored form. A parallel line keeps only the row side and drops optimizer state by preconditioning the gradient row by row: SWAN, gradient multi-normalization, a minimalist optimizer and SRON replace Adam's state with a row-wise $\ell_2$ normalization~\citep{ma2025swan,scetbon2025gradient,glentis2025minimalist,wen2025sron}; \citet{gu2026mano,xu2026width} study its manifold-optimization and width-scaling sides; \citet{pethick2025training} reach it from norm-constrained steepest descent; \citet{lion2026muown} make the row magnitudes an explicit optimizer variable, which \citet{hubler2026muown} show implicitly decays the angular step size; and \citet{deng2026rmnp} show that row-momentum normalization matches Muon at scale. Our analysis explains the common thread: once $\bm{V}\bm{V}^\top$ is diagonally dominant, row normalization and orthogonalization are the same operation (\cref{sec:equiv}).

\section{Symmetric Matrix Factorization}\label{app:mf}

\subsection{Setup}\label{app:mf:setup}

Let $d,k,r$ be positive integers with $k\ge r$. The over-parameterized symmetric matrix-factorization loss is
\begin{equation}\label{eq:mfa-loss}
\bL(\bm{W})\;\eqdef\;\tfrac14\,\norm{\bm{W}\bm{W}^\top-\Mstar}_\Frob^2,\qquad\bm{W}\in\R^{d\times k},
\end{equation}
with target $\Mstar\in\R^{d\times d}$ symmetric positive semidefinite of rank $r$. A direct computation gives the gradient
\begin{equation}\label{eq:mfa-grad}
\bm{G}\;\eqdef\;\nabla\bL(\bm{W})\;=\;(\bm{W}\bm{W}^\top-\Mstar)\,\bm{W}\;\in\;\R^{d\times k},
\end{equation}
and the object of interest is the $d\times d$ matrix $\E[\bm{G}\bm{G}^\top]$.

\subsection{Standing assumption}\label{app:mf:assump}

\begin{assumption}\label{ass:mfa}
The entries $W_{ij}$ are i.i.d.\ $\bN(0,1)$, and $\Mstar$ is deterministic and independent of $\bm{W}$. Hence $\bm{S}\eqdef\bm{W}\bm{W}^\top$ is Wishart with $k$ degrees of freedom and identity scale, $\bm{S}\sim\mathcal{W}_d(k,\bm{I}_d)$.
\end{assumption}

\subsection{Main results restated}\label{app:mf:results}

\begin{theoremrst}{\cref{thm:mf}, restated}
Under \cref{ass:mfa},
\begin{equation}\label{eq:mfa-main}
\E[\bm{G}\bm{G}^\top]\;=\;\alpha_3(d,k)\,\bm{I}_d\;-\;2\,\alpha_2(d,k)\,\Mstar\;+\;k\,(\Mstar)^2,
\end{equation}
where the scalars are the diagonal entries of the first three Wishart moments,
\begin{align}
\alpha_1(d,k) &= k, \label{eq:mfa-a1}\\
\alpha_2(d,k) &= k(k+d+1), \label{eq:mfa-a2}\\
\alpha_3(d,k) &= k^3+3(d+1)k^2+(d^2+3d+4)k. \label{eq:mfa-a3}
\end{align}
\end{theoremrst}

\begin{corollary}\label{cor:mfa-scaling}
Under \cref{ass:mfa}, $\E[\bm{G}\bm{G}^\top]$ is invariant under any orthogonal change of basis fixing $\Mstar$, with
\begin{align}
\E\bigl[(\bm{G}\bm{G}^\top)_{ii}\bigr] &= \alpha_3-2\alpha_2\Mstar_{ii}+k(\Mstar)^2_{ii}, \label{eq:mfa-diag}\\
\E\bigl[(\bm{G}\bm{G}^\top)_{ij}\bigr] &= -2\alpha_2\Mstar_{ij}+k(\Mstar)^2_{ij}\qquad(i\neq j). \label{eq:mfa-off}
\end{align}
As $k\to\infty$, $\E[(\bm{G}\bm{G}^\top)_{ii}]=\Theta(k^3)$ and $\E[(\bm{G}\bm{G}^\top)_{ij}]=\Theta(k^2)$ whenever $\Mstar_{ij}\neq0$, so the dominance ratio is $\Theta(k/\Mstar_{ij})$.
\end{corollary}

\subsection{Preliminary lemmas}\label{app:mf:lemmas}

\begin{lemma}\label{lem:mfa-quad}
Let $\bm{X}\in\R^{m\times n}$ have i.i.d.\ $\bN(0,1)$ entries. Then for deterministic $\bm{M}\in\R^{n\times n}$ and $\bm{N}\in\R^{m\times m}$,
\begin{equation}\label{eq:mfa-quad}
\E[\bm{X}\bm{M}\bm{X}^\top]=\tr(\bm{M})\,\bm{I}_m,\qquad
\E[\bm{X}^\top\bm{N}\bm{X}]=\tr(\bm{N})\,\bm{I}_n.
\end{equation}
\end{lemma}

\begin{proof}
Using $\E[\bm{X}_{ap}\bm{X}_{bq}]=\delta_{ab}\delta_{pq}$,
\[
\E[(\bm{X}\bm{M}\bm{X}^\top)_{ab}]=\sum_{p,q}\bm{M}_{pq}\,\delta_{ab}\delta_{pq}=\delta_{ab}\tr(\bm{M}),
\]
which is the first identity; the second follows by applying it to $\bm{X}^\top$ (entries i.i.d.\ $\bN(0,1)$).
\end{proof}

\begin{lemma}\label{lem:mfa-isotropy}
Let $\bm{S}=\bm{W}\bm{W}^\top$ with $\bm{W}$ as in \cref{ass:mfa}. For every $m\in\mathbb{N}$ there is a scalar $\alpha_m(d,k)$ with
\begin{equation}\label{eq:mfa-iso}
\E[\bm{S}^m]=\alpha_m(d,k)\,\bm{I}_d,\qquad \alpha_m(d,k)=\tfrac1d\,\E[\tr(\bm{S}^m)].
\end{equation}
\end{lemma}

\begin{proof}
For $\bm{Q}\in O(d)$, $\bm{Q}\bm{W}\stackrel{d}{=}\bm{W}$, hence $\bm{Q}\bm{S}^m\bm{Q}^\top\stackrel{d}{=}\bm{S}^m$ and
\[
\E[\bm{S}^m]=\bm{Q}\,\E[\bm{S}^m]\,\bm{Q}^\top\qquad\text{for all }\bm{Q}\in O(d).
\]
By Schur's lemma $\E[\bm{S}^m]=\alpha_m\bm{I}_d$; taking the trace gives $d\alpha_m=\E[\tr(\bm{S}^m)]$.
\end{proof}

\begin{lemma}[Wick/Isserlis]\label{lem:mfa-wick}
Let $X_1,\dots,X_{2n}$ be jointly centered Gaussian. Then
\begin{equation}\label{eq:mfa-wick}
\E[X_1\cdots X_{2n}]=\sum_{P}\prod_{\{i,j\}\in P}\E[X_iX_j],
\end{equation}
the sum over all pairings $P$ of $\{1,\dots,2n\}$; odd moments vanish~\citep{isserlis1918formula,wick1950evaluation}.
\end{lemma}

\begin{proof}
Differentiate $\E[\exp(\bm{t}^\top\bm{X})]=\exp(\tfrac12\bm{t}^\top\bm{\Sigma}\bm{t})$ in $t_1,\dots,t_{2n}$ at $\bm{t}=\bm{0}$; cf.\ \citet[\S1.4]{janson1997gaussian}.
\end{proof}

\subsection{Proofs}\label{app:mf:proofs}

\begin{lemma}\label{lem:mfa-expand}
With $\bm{S}\eqdef\bm{W}\bm{W}^\top$ and $\bm{B}\eqdef\bm{S}-\Mstar$,
\begin{equation}\label{eq:mfa-expand}
\bm{G}\bm{G}^\top=\bm{B}\bm{S}\bm{B}=\bm{S}^3-\bm{S}^2\Mstar-\Mstar\bm{S}^2+\Mstar\bm{S}\Mstar.
\end{equation}
\end{lemma}

\begin{proof}
Since $\bm{B}=\bm{B}^\top$,
\[
\bm{G}\bm{G}^\top=\bm{B}\bm{W}(\bm{B}\bm{W})^\top=\bm{B}\bm{W}\bm{W}^\top\bm{B}^\top=\bm{B}\bm{S}\bm{B},
\]
and substituting $\bm{B}=\bm{S}-\Mstar$ and expanding gives \eqref{eq:mfa-expand}.
\end{proof}

\begin{lemma}\label{lem:mfa-a1}
Under \cref{ass:mfa}, $\E[\bm{S}]=k\,\bm{I}_d$, hence $\alpha_1(d,k)=k$.
\end{lemma}

\begin{proof}
Let $\bm{w}_a$ denote the $a$-th column of $\bm{W}$. Then
\[
\E[\bm{S}]=\E\!\Bigl[\sum_{a=1}^k\bm{w}_a\bm{w}_a^\top\Bigr]=\sum_{a=1}^k\E[\bm{w}_a\bm{w}_a^\top]=k\,\bm{I}_d.
\]
\end{proof}

\begin{lemma}\label{lem:mfa-a2}
Under \cref{ass:mfa}, $\E[\bm{S}^2]=k(k+d+1)\,\bm{I}_d$, hence $\alpha_2(d,k)=k(k+d+1)$.
\end{lemma}

\begin{proof}
By \cref{lem:mfa-isotropy}, $\E[\bm{S}^2]=\alpha_2\bm{I}_d$. With columns $\bm{w}_a\stackrel{\text{iid}}{\sim}\bN(\bm{0},\bm{I}_d)$,
\[
\bm{S}^2=\sum_{a,b=1}^k\langle\bm{w}_a,\bm{w}_b\rangle\,\bm{w}_a\bm{w}_b^\top.
\]
Split into the $k$ terms with $a=b$ and the $k(k-1)$ with $a\neq b$. By \cref{lem:mfa-wick} and independence,
\begin{align*}
\E[\norm{\bm{w}_a}^2\bm{w}_a\bm{w}_a^\top] &= (d+2)\,\bm{I}_d,\\
\E[\langle\bm{w}_a,\bm{w}_b\rangle\,\bm{w}_a\bm{w}_b^\top]
&= \sum_p\E[w_{a,p}\bm{w}_a]\,\E[w_{b,p}\bm{w}_b^\top]
= \sum_p\bm{e}_p\bm{e}_p^\top
= \bm{I}_d.
\end{align*}
Summing, $\E[\bm{S}^2]=k(d+2)\bm{I}_d+k(k-1)\bm{I}_d=k(k+d+1)\bm{I}_d$.
\end{proof}

\begin{lemma}\label{lem:mfa-a3}
Under \cref{ass:mfa}, $\E[\bm{S}^3]=\bigl[k^3+3(d+1)k^2+(d^2+3d+4)k\bigr]\bm{I}_d$, hence $\alpha_3(d,k)$ as in \eqref{eq:mfa-a3}.
\end{lemma}

\begin{proof}
By \cref{lem:mfa-isotropy}, $\alpha_3=\tfrac1d\E[\tr(\bm{S}^3)]$, and with $\bm{S}=\sum_{a}\bm{w}_a\bm{w}_a^\top$ and cyclicity of the trace,
\[
\tr(\bm{S}^3)=\sum_{a,b,c}\langle\bm{w}_a,\bm{w}_b\rangle\,\langle\bm{w}_b,\bm{w}_c\rangle\,\langle\bm{w}_c,\bm{w}_a\rangle.
\]
Partition $(a,b,c)\in\{1,\dots,k\}^3$ by multiplicity pattern:
\begin{align*}
\#\{a=b=c\}=k:\qquad & \E[\norm{\bm{w}_a}^6]=d(d+2)(d+4),\\
\#\{|\{a,b,c\}|=2\}=3k(k-1):\qquad & \E[\norm{\bm{w}_a}^2\langle\bm{w}_a,\bm{w}_c\rangle^2]=\E[\norm{\bm{w}_a}^4]=d(d+2),\\
\#\{a,b,c\text{ distinct}\}=k(k-1)(k-2):\qquad & \textstyle\sum_{p,q,r}\delta_{pr}\delta_{pq}\delta_{qr}=d,
\end{align*}
where for the middle pattern, e.g.\ $a=b\neq c$ gives $\norm{\bm{w}_a}^2\langle\bm{w}_a,\bm{w}_c\rangle^2$ and, conditioning on $\bm{w}_a$, $\langle\bm{w}_a,\bm{w}_c\rangle\mid\bm{w}_a\sim\bN(0,\norm{\bm{w}_a}^2)$; the other two cases are symmetric. Hence
\begin{align*}
d\,\alpha_3 &= k\,d(d+2)(d+4)+3k(k-1)\,d(d+2)+k(k-1)(k-2)\,d,\\
\alpha_3 &= k(d+2)(d+4)+3k(k-1)(d+2)+k(k-1)(k-2)
        = k^3+3(d+1)k^2+(d^2+3d+4)k.
\end{align*}
\end{proof}

\begin{proof}[Proof of \cref{thm:mf}]
Take expectations in \eqref{eq:mfa-expand}. Since $\Mstar$ is deterministic, \cref{lem:mfa-isotropy} gives
\begin{align*}
\E[\bm{S}^3] &= \alpha_3\bm{I}_d,\\
\E[\bm{S}^2\Mstar]=\E[\Mstar\bm{S}^2] &= \alpha_2\Mstar,\\
\E[\Mstar\bm{S}\Mstar]=\Mstar\,\E[\bm{S}]\,\Mstar &= k\,(\Mstar)^2,
\end{align*}
so summing gives \eqref{eq:mfa-main}; the values of $\alpha_1,\alpha_2,\alpha_3$ are \cref{lem:mfa-a1,lem:mfa-a2,lem:mfa-a3}. \Cref{cor:mfa-scaling} follows from the entries of \eqref{eq:mfa-main}: with $d,\Mstar$ fixed, $\alpha_3=\Theta(k^3)$, $\alpha_2=\Theta(k^2)$, $k(\Mstar)^2_\bullet=\Theta(k)$.
\end{proof}

\begin{remark}
This computation answers the question raised by \citet{ma2026preconditioning} of how gradient orthogonalization meets the curvature of training: the $\Theta(k)$ gap between the diagonal and off-diagonal scales is the regime in which orthogonalization, which suppresses the isotropic $\Theta(k^3)\bm{I}_d$ part, acts as a real preconditioner relative to the target-aligned signal $\Mstar$.
\end{remark}

\section{Deep Linear Networks}\label{app:dl}

\subsection{Setup}\label{app:dl:setup}

Let $L\ge1$ and $d_0,\dots,d_L$ be positive integers. The depth-$L$ linear network
\begin{equation}\label{eq:dla-net}
f(\bm{x})=\bm{W}_L\bm{W}_{L-1}\cdots\bm{W}_1\bm{x},\qquad\bm{W}_i\in\R^{d_i\times d_{i-1}},
\end{equation}
learns a deterministic linear target $\bm{\Phi}\in\R^{d_L\times d_0}$ under the population squared loss
\begin{equation}\label{eq:dla-loss}
\bL(\bm{W}_1,\dots,\bm{W}_L)\;\eqdef\;\tfrac12\,\E_{\bm{x}}\norm{\bm{W}_L\cdots\bm{W}_1\bm{x}-\bm{\Phi}\bm{x}}^2.
\end{equation}
Define the layer-$i$ forward, backward and residual factors
\begin{equation}\label{eq:dla-ABR}
\bm{B}_i\eqdef\bm{W}_{i-1}\cdots\bm{W}_1\in\R^{d_{i-1}\times d_0},\qquad
\bm{A}_i\eqdef\bm{W}_L\cdots\bm{W}_{i+1}\in\R^{d_L\times d_i},\qquad
\bm{R}\eqdef\bm{W}_L\cdots\bm{W}_1-\bm{\Phi},
\end{equation}
with $\bm{B}_1\eqdef\bm{I}_{d_0}$ and $\bm{A}_L\eqdef\bm{I}_{d_L}$. The layer-$i$ gradient is
\begin{equation}\label{eq:dla-grad}
\bm{G}_i\;\eqdef\;\frac{\partial\bL}{\partial\bm{W}_i}\;=\;\bm{A}_i^\top\bm{R}\,\bm{B}_i^\top\;\in\;\R^{d_i\times d_{i-1}}.
\end{equation}
The closed form below uses two pairs of dimension-only scalar sequences. For the backward factors $\{\bm{B}_i\}$,
\begin{align}
r_1 &= 1, & v_1 &= d_0^2, \label{eq:dla-rv-init}\\
r_i &= v_{i-1}+(d_{i-1}+1)d_{i-2}r_{i-1}, & v_i &= d_{i-1}^2v_{i-1}+2d_{i-1}d_{i-2}r_{i-1}\quad(i\ge2); \label{eq:dla-rv-rec}
\end{align}
for the forward factors $\{\bm{A}_i\}$,
\begin{align}
s_L &= 1, & u_L &= d_L^2, \label{eq:dla-su-init}\\
s_{i-1} &= u_i+(d_{i-1}+1)d_is_i, & u_{i-1} &= d_{i-1}^2u_i+2d_{i-1}d_is_i\quad(i\le L); \label{eq:dla-su-rec}
\end{align}
and the composite Frobenius coefficient is
\begin{equation}\label{eq:dla-Vi}
V_i\;\eqdef\;d_{i-1}r_i\;=\;\E\bigl[\norm{\bm{B}_i\bm{B}_i^\top}_\Frob^2\bigr].
\end{equation}

\subsection{Standing assumption}\label{app:dl:assump}

\begin{assumption}\label{ass:dla}
The entries $(\bm{W}_i)_{ab}$ are i.i.d.\ $\bN(0,1)$ across all $i,a,b$; $\bm{x}\sim\bN(\bm{0},\bm{I}_{d_0})$; and $\bm{\Phi}$ is deterministic and independent of $\bm{W}_1,\dots,\bm{W}_L$. Expectations $\E[\cdot]$ are over $\bm{W}_1,\dots,\bm{W}_L$ ($\bm{x}$ has been integrated out in $\bL$).
\end{assumption}

\subsection{Main results restated}\label{app:dl:results}

\begin{theoremrst}{\cref{thm:dl}, restated}
Under \cref{ass:dla}, for every layer $i\in\{1,\dots,L\}$,
\begin{equation}\label{eq:dla-main}
\E[\bm{G}_i\bm{G}_i^\top]=V_i\,s_i\,\bm{I}_{d_i}+\bm{T}_i,\qquad
\bm{T}_i=\begin{cases}\Bigl(\prod_{j=1}^{L-1}d_j\Bigr)\bm{\Phi}\bm{\Phi}^\top,&i=L,\\[6pt]
\Bigl(\prod_{j=1,\,j\neq i}^{L-1}d_j\Bigr)\norm{\bm{\Phi}}_\Frob^2\,\bm{I}_{d_i},&1\le i<L,\end{cases}
\end{equation}
with $V_i,s_i$ from \eqref{eq:dla-rv-init}--\eqref{eq:dla-Vi}.
\end{theoremrst}

\begin{corollary}\label{cor:dla-scaling}
Under \cref{ass:dla}, the diagonal entries of $\E[\bm{G}_i\bm{G}_i^\top]$ are
\begin{equation}\label{eq:dla-diag}
\E\bigl[(\bm{G}_i\bm{G}_i^\top)_{aa}\bigr]=\begin{cases}V_is_i+\bigl(\prod_{j\neq i,\,1\le j\le L-1}d_j\bigr)\norm{\bm{\Phi}}_\Frob^2,&1\le i<L,\\[4pt]
V_Ls_L+\bigl(\prod_{j=1}^{L-1}d_j\bigr)(\bm{\Phi}\bm{\Phi}^\top)_{aa},&i=L,\end{cases}
\end{equation}
and, for $a\neq b$, the off-diagonal entries are $\E[(\bm{G}_i\bm{G}_i^\top)_{ab}]=0$ if $1\le i<L$ and $\bigl(\prod_{j=1}^{L-1}d_j\bigr)(\bm{\Phi}\bm{\Phi}^\top)_{ab}$ if $i=L$. In particular $\E[\bm{G}_i\bm{G}_i^\top]$ is a scalar multiple of $\bm{I}_{d_i}$ for every $i<L$. Moreover,
\begin{equation}\label{eq:dla-order}
V_is_i=\Theta\Bigl(d_i\prod_{j=1,\,j\neq i}^{L-1}d_j^2\Bigr)\quad(1\le i<L),\qquad
V_Ls_L=\Theta\Bigl(\prod_{j=1}^{L-1}d_j^2\Bigr).
\end{equation}
\end{corollary}

\subsection{Preliminary lemmas}\label{app:dl:lemmas}

\begin{lemma}\label{lem:dla-quad}
Same as \cref{lem:mfa-quad}.
\end{lemma}

\begin{lemma}[Wick/Isserlis]\label{lem:dla-wick}
Same as \cref{lem:mfa-wick}.
\end{lemma}

\begin{lemma}\label{lem:dla-fourth}
Let $\bm{X}\in\R^{m\times n}$ have i.i.d.\ $\bN(0,1)$ entries and let $\bm{M}\in\R^{n\times n}$, $\bm{N}\in\R^{m\times m}$ be deterministic symmetric. Then
\begin{align}
\E\bigl[\bm{X}\bm{M}\bm{X}^\top\bm{X}\bm{M}\bm{X}^\top\bigr] &= \bigl(\tr(\bm{M})^2+(m+1)\norm{\bm{M}}_\Frob^2\bigr)\bm{I}_m, \label{eq:dla-fourth-left}\\
\E\bigl[\bm{X}^\top\bm{N}\bm{X}\bm{X}^\top\bm{N}\bm{X}\bigr] &= \bigl(\tr(\bm{N})^2+(n+1)\norm{\bm{N}}_\Frob^2\bigr)\bm{I}_n. \label{eq:dla-fourth-right}
\end{align}
\end{lemma}

\begin{proof}
We prove \eqref{eq:dla-fourth-left}; \eqref{eq:dla-fourth-right} follows by applying it to $\bm{X}^\top$. For $a,b,c,d\in\{1,\dots,m\}$, \cref{lem:dla-wick} gives
\[
\E[\bm{X}_{ap}\bm{X}_{bq}\bm{X}_{cr}\bm{X}_{ds}]=\delta_{ab}\delta_{pq}\delta_{cd}\delta_{rs}+\delta_{ac}\delta_{pr}\delta_{bd}\delta_{qs}+\delta_{ad}\delta_{ps}\delta_{bc}\delta_{qr},
\]
and contracting against $\bm{M}_{pq}\bm{M}_{rs}$ with $\bm{M}=\bm{M}^\top$,
\[
\E[(\bm{X}\bm{M}\bm{X}^\top)_{ab}(\bm{X}\bm{M}\bm{X}^\top)_{cd}]=\delta_{ab}\delta_{cd}\tr(\bm{M})^2+(\delta_{ac}\delta_{bd}+\delta_{ad}\delta_{bc})\norm{\bm{M}}_\Frob^2.
\]
Summing over $b=c$,
\[
\E[(\bm{X}\bm{M}\bm{X}^\top\bm{X}\bm{M}\bm{X}^\top)_{ac}]=\sum_b\E[(\bm{X}\bm{M}\bm{X}^\top)_{ab}(\bm{X}\bm{M}\bm{X}^\top)_{bc}]=\delta_{ac}\bigl(\tr(\bm{M})^2+(m+1)\norm{\bm{M}}_\Frob^2\bigr).
\]
\end{proof}

\begin{lemma}\label{lem:dla-trsq}
Let $\bm{X}\in\R^{m\times n}$ have i.i.d.\ $\bN(0,1)$ entries and $\bm{M}\in\R^{n\times n}$ be deterministic symmetric. Then
\begin{equation}\label{eq:dla-trsq}
\E\bigl[\tr(\bm{X}\bm{M}\bm{X}^\top)^2\bigr]=m^2\tr(\bm{M})^2+2m\norm{\bm{M}}_\Frob^2.
\end{equation}
\end{lemma}

\begin{proof}
Expanding,
\[
\tr(\bm{X}\bm{M}\bm{X}^\top)^2=\sum_{a,a',p,q,p',q'}\bm{X}_{ap}\bm{M}_{pq}\bm{X}_{aq}\,\bm{X}_{a'p'}\bm{M}_{p'q'}\bm{X}_{a'q'};
\]
by \cref{lem:dla-wick} the three pair contractions of $\E[\bm{X}_{ap}\bm{X}_{aq}\bm{X}_{a'p'}\bm{X}_{a'q'}]$ contribute (after summing over $\bm{M}\bm{M}$) $m^2\tr(\bm{M})^2$, $m\norm{\bm{M}}_\Frob^2$ and $m\tr(\bm{M}^2)=m\norm{\bm{M}}_\Frob^2$, which sum to \eqref{eq:dla-trsq}.
\end{proof}

\subsection{Proofs}\label{app:dl:proofs}

The proof of \cref{thm:dl} has three steps: a $\bm{W}_i$-expansion; recursive moments of $\bm{B}_i\bm{B}_i^\top$ and $\bm{A}_i^\top\bm{A}_i$; and the target-dependent term.

\begin{lemma}\label{lem:dla-Wexp}
Under \cref{ass:dla}, for every $i$,
\begin{equation}\label{eq:dla-Wexp}
\E\bigl[\bm{G}_i\bm{G}_i^\top\bigm|\bm{A}_i,\bm{B}_i\bigr]=\tr\bigl((\bm{B}_i\bm{B}_i^\top)^2\bigr)\,(\bm{A}_i^\top\bm{A}_i)^2+\bm{A}_i^\top\bm{\Phi}\bm{B}_i^\top\bm{B}_i\bm{\Phi}^\top\bm{A}_i.
\end{equation}
\end{lemma}

\begin{proof}
Substituting $\bm{R}=\bm{A}_i\bm{W}_i\bm{B}_i-\bm{\Phi}$ into \eqref{eq:dla-grad},
\[
\bm{G}_i\bm{G}_i^\top=(\bm{A}_i^\top\bm{A}_i)\bm{W}_i(\bm{B}_i\bm{B}_i^\top\bm{B}_i\bm{B}_i^\top)\bm{W}_i^\top(\bm{A}_i^\top\bm{A}_i)+(\text{cross})+\bm{A}_i^\top\bm{\Phi}\bm{B}_i^\top\bm{B}_i\bm{\Phi}^\top\bm{A}_i,
\]
where $(\text{cross})$ collects the two terms linear in $\bm{W}_i$; since $\bm{W}_i\stackrel{d}{=}-\bm{W}_i$ is independent of $(\bm{A}_i,\bm{B}_i)$, $\E[(\text{cross})\mid\bm{A}_i,\bm{B}_i]=\bm{0}$. By \cref{lem:dla-quad} with $\bm{K}=\bm{B}_i\bm{B}_i^\top\bm{B}_i\bm{B}_i^\top$,
\[
\E_{\bm{W}_i}\bigl[\bm{W}_i\bm{K}\bm{W}_i^\top\bigm|\bm{A}_i,\bm{B}_i\bigr]=\tr(\bm{K})\,\bm{I}_{d_i}=\tr\bigl((\bm{B}_i\bm{B}_i^\top)^2\bigr)\bm{I}_{d_i},
\]
which yields \eqref{eq:dla-Wexp}.
\end{proof}

\begin{lemma}\label{lem:dla-Bmom}
Under \cref{ass:dla}, for every $i\in\{1,\dots,L\}$,
\begin{equation}\label{eq:dla-Bmom}
\E\bigl[(\bm{B}_i\bm{B}_i^\top)^2\bigr]=r_i\,\bm{I}_{d_{i-1}},\qquad
\E\bigl[\tr(\bm{B}_i\bm{B}_i^\top)^2\bigr]=v_i,
\end{equation}
with $(r_i,v_i)$ from \eqref{eq:dla-rv-init}--\eqref{eq:dla-rv-rec}.
\end{lemma}

\begin{proof}
At $i=1$, $\bm{B}_1=\bm{I}_{d_0}$, so $(\bm{B}_1\bm{B}_1^\top)^2=\bm{I}_{d_0}$ and $\tr(\bm{B}_1\bm{B}_1^\top)^2=d_0^2$, matching $r_1=1$, $v_1=d_0^2$. For $i\ge2$, $\bm{B}_i=\bm{W}_{i-1}\bm{B}_{i-1}$ with $\bm{W}_{i-1}\in\R^{d_{i-1}\times d_{i-2}}$ independent of $\bm{B}_{i-1}$, so $\bm{B}_i\bm{B}_i^\top=\bm{W}_{i-1}\bm{M}\bm{W}_{i-1}^\top$ with $\bm{M}\eqdef\bm{B}_{i-1}\bm{B}_{i-1}^\top$. By \cref{lem:dla-fourth},
\[
\E\bigl[(\bm{B}_i\bm{B}_i^\top)^2\bigm|\bm{B}_{i-1}\bigr]=\bigl(\tr(\bm{M})^2+(d_{i-1}+1)\norm{\bm{M}}_\Frob^2\bigr)\bm{I}_{d_{i-1}},
\]
and taking $\E$ over $\bm{B}_{i-1}$ with $\E[\tr(\bm{M}^2)]=d_{i-2}r_{i-1}$ gives $r_i=v_{i-1}+(d_{i-1}+1)d_{i-2}r_{i-1}$. Likewise, by \cref{lem:dla-trsq},
\[
\E\bigl[\tr(\bm{B}_i\bm{B}_i^\top)^2\bigm|\bm{B}_{i-1}\bigr]=d_{i-1}^2\tr(\bm{M})^2+2d_{i-1}\norm{\bm{M}}_\Frob^2,
\]
so $v_i=d_{i-1}^2v_{i-1}+2d_{i-1}d_{i-2}r_{i-1}$.
\end{proof}

\begin{lemma}\label{lem:dla-Amom}
Under \cref{ass:dla}, for every $i\in\{0,\dots,L\}$,
\begin{equation}\label{eq:dla-Amom}
\E\bigl[(\bm{A}_i^\top\bm{A}_i)^2\bigr]=s_i\,\bm{I}_{d_i},\qquad
\E\bigl[\tr(\bm{A}_i^\top\bm{A}_i)^2\bigr]=u_i,
\end{equation}
with $(s_i,u_i)$ from \eqref{eq:dla-su-init}--\eqref{eq:dla-su-rec}.
\end{lemma}

\begin{proof}
At $i=L$, $\bm{A}_L=\bm{I}_{d_L}$, matching $s_L=1$, $u_L=d_L^2$. For $i<L$, $\bm{A}_i=\bm{A}_{i+1}\bm{W}_{i+1}$ with $\bm{W}_{i+1}\in\R^{d_{i+1}\times d_i}$ independent of $\bm{A}_{i+1}$, so $\bm{A}_i^\top\bm{A}_i=\bm{W}_{i+1}^\top\bm{N}\bm{W}_{i+1}$ with $\bm{N}\eqdef\bm{A}_{i+1}^\top\bm{A}_{i+1}$. By \cref{lem:dla-fourth} (right identity),
\[
\E\bigl[(\bm{A}_i^\top\bm{A}_i)^2\bigm|\bm{A}_{i+1}\bigr]=\bigl(\tr(\bm{N})^2+(d_i+1)\norm{\bm{N}}_\Frob^2\bigr)\bm{I}_{d_i},
\]
so $s_i=u_{i+1}+(d_i+1)d_{i+1}s_{i+1}$; the recursion for $u_i$ is that of $v_i$ with $d_{i-1}\leftrightarrow d_i$ and $d_{i-2}\leftrightarrow d_{i+1}$, giving $u_{i-1}=d_{i-1}^2u_i+2d_{i-1}d_is_i$.
\end{proof}

\begin{lemma}\label{lem:dla-Phi}
Under \cref{ass:dla}, $\E[\bm{A}_i^\top\bm{\Phi}\bm{B}_i^\top\bm{B}_i\bm{\Phi}^\top\bm{A}_i]=\bm{T}_i$, with $\bm{T}_i$ as in \eqref{eq:dla-main}.
\end{lemma}

\begin{proof}
By independence of $\bm{A}_i,\bm{B}_i$ and \cref{lem:dla-quad} iterated through $\bm{B}_i=\bm{W}_{i-1}\cdots\bm{W}_1$,
\[
\E[\bm{B}_i^\top\bm{B}_i]=\Bigl(\prod_{j=1}^{i-1}d_j\Bigr)\bm{I}_{d_0}
\qquad\Longrightarrow\qquad
\E[\bm{A}_i^\top\bm{\Phi}\bm{B}_i^\top\bm{B}_i\bm{\Phi}^\top\bm{A}_i]=\Bigl(\prod_{j=1}^{i-1}d_j\Bigr)\,\E[\bm{A}_i^\top\bm{\Phi}\bm{\Phi}^\top\bm{A}_i]
\]
(the empty product is $1$ when $i=1$). Let $h_i(\bm{M})\eqdef\E[\bm{A}_i^\top\bm{M}\bm{A}_i]$ for symmetric $\bm{M}\in\R^{d_L\times d_L}$. Then $h_L(\bm{M})=\bm{M}$, and for $i<L$, by \cref{lem:dla-quad},
\[
h_i(\bm{M})=\E\bigl[\tr(\bm{A}_{i+1}^\top\bm{M}\bm{A}_{i+1})\bigr]\bm{I}_{d_i}=\tr\bigl(h_{i+1}(\bm{M})\bigr)\bm{I}_{d_i}.
\]
Iterating,
\[
h_i(\bm{M})=\begin{cases}\bm{M},&i=L,\\[2pt]\bigl(\prod_{j=i+1}^{L-1}d_j\bigr)\tr(\bm{M})\,\bm{I}_{d_i},&i<L,\end{cases}
\]
with $\prod_{j=L}^{L-1}d_j=1$. Taking $\bm{M}=\bm{\Phi}\bm{\Phi}^\top$ (so $\tr(\bm{M})=\norm{\bm{\Phi}}_\Frob^2$) and multiplying by $\prod_{j=1}^{i-1}d_j$ gives $\bm{T}_i$.
\end{proof}

\begin{proof}[Proof of \cref{thm:dl}]
By \cref{lem:dla-Wexp},
\[
\E[\bm{G}_i\bm{G}_i^\top]=\E\bigl[\tr((\bm{B}_i\bm{B}_i^\top)^2)\,(\bm{A}_i^\top\bm{A}_i)^2\bigr]+\E\bigl[\bm{A}_i^\top\bm{\Phi}\bm{B}_i^\top\bm{B}_i\bm{\Phi}^\top\bm{A}_i\bigr].
\]
Since $\bm{A}_i,\bm{B}_i$ depend on disjoint sets of $\bm{W}_j$, they are independent, so by \cref{lem:dla-Bmom,lem:dla-Amom},
\[
\E\bigl[\tr((\bm{B}_i\bm{B}_i^\top)^2)\,(\bm{A}_i^\top\bm{A}_i)^2\bigr]=\E\bigl[\tr((\bm{B}_i\bm{B}_i^\top)^2)\bigr]\,\E\bigl[(\bm{A}_i^\top\bm{A}_i)^2\bigr]=d_{i-1}r_i\cdot s_i\,\bm{I}_{d_i}=V_is_i\,\bm{I}_{d_i},
\]
and the second term equals $\bm{T}_i$ by \cref{lem:dla-Phi}. \Cref{cor:dla-scaling} follows from the entries of \eqref{eq:dla-main}: for $i<L$, $\bm{T}_i$ is a multiple of $\bm{I}_{d_i}$ so the off-diagonal vanishes; for $i=L$, $\bm{T}_L$ contributes $(\bm{\Phi}\bm{\Phi}^\top)_{ab}$ entry-wise.

It remains to prove \eqref{eq:dla-order}. Let $\pi_i\eqdef\prod_{j=0}^{i-1}d_j$ and $\sigma_i\eqdef\sum_{m=0}^{i-1}d_m^{-1}$, so that $\pi_i=d_{i-1}\pi_{i-1}$ and $\sigma_i=\sigma_{i-1}+d_{i-1}^{-1}$. Since $V_i=d_{i-1}r_i$, the recursion \eqref{eq:dla-rv-init}--\eqref{eq:dla-rv-rec} reads
\begin{gather*}
V_1=d_0,\qquad v_1=d_0^2,\\
V_i=d_{i-1}v_{i-1}+d_{i-1}(d_{i-1}+1)V_{i-1},\qquad
v_i=d_{i-1}^2v_{i-1}+2d_{i-1}V_{i-1}\qquad(i\ge2).
\end{gather*}
Let $a_i\eqdef v_i/\pi_i^2$ and $b_i\eqdef V_i/(\pi_i^2\sigma_i)$, so $a_1=b_1=1$. Dividing the two recursions by $\pi_i^2$ and by $\pi_i^2\sigma_i$ gives
\[
a_i=a_{i-1}+\frac{2\sigma_{i-1}}{d_{i-1}}\,b_{i-1},
\qquad
b_i=\frac{d_{i-1}^{-1}}{\sigma_i}\,a_{i-1}+\frac{\sigma_{i-1}}{\sigma_i}\,\bigl(1+d_{i-1}^{-1}\bigr)b_{i-1}.
\]
The second identity expresses $b_i$ as a convex combination of $a_{i-1}$ and $(1+d_{i-1}^{-1})b_{i-1}$, so $\min(a_{i-1},b_{i-1})\le b_i\le2\max(a_{i-1},b_{i-1})$; the first gives $a_{i-1}\le a_i\le a_{i-1}+2(i-1)b_{i-1}$, using $\sigma_{i-1}\le i-1$ and $d_{i-1}\ge1$. By induction, $1\le a_i,b_i\le(2L)^{L}$ for every $i\le L$, that is,
\[
v_i=\Theta(\pi_i^2),\qquad V_i=\Theta(\pi_i^2\sigma_i),
\]
The recursion \eqref{eq:dla-su-init}--\eqref{eq:dla-su-rec} has the same form read from the output side: with $S_i\eqdef d_is_i$, $\rho_i\eqdef\prod_{j=i}^{L}d_j$ and $\tau_i\eqdef\sum_{m=i}^{L}d_m^{-1}$,
\begin{gather*}
S_L=d_L,\qquad u_L=d_L^2,\\
S_{i-1}=d_{i-1}u_i+d_{i-1}(d_{i-1}+1)S_i,\qquad
u_{i-1}=d_{i-1}^2u_i+2d_{i-1}S_i\qquad(i\le L),
\end{gather*}
and the same argument gives $u_i=\Theta(\rho_i^2)$ and $S_i=\Theta(\rho_i^2\tau_i)$. Therefore
\[
V_is_i=\frac{V_iS_i}{d_i}=\Theta\Bigl(\frac{\pi_i^2\rho_i^2}{d_i}\,\sigma_i\tau_i\Bigr),
\qquad
\frac{\pi_i^2\rho_i^2}{d_i}=\frac{1}{d_i}\Bigl(\prod_{j=0}^{L}d_j\Bigr)^2 .
\]
Since $d_0^{-1}\le\sigma_i\le L$ and $d_L^{-1}\le\tau_i\le L$, the factor $\sigma_i\tau_i$ is bounded above and below, and $\bigl(\prod_{j=0}^{L}d_j\bigr)^2/d_i$ equals $d_0^2d_L^2\,d_i\prod_{j=1,\,j\neq i}^{L-1}d_j^2$ for $i<L$ and $d_0^2d_L\prod_{j=1}^{L-1}d_j^2$ for $i=L$, which gives \eqref{eq:dla-order}.
\end{proof}

\section{Two-Layer ReLU Networks}\label{app:relu}

\subsection{Setup}\label{app:relu:setup}

Let $d_0,d_1,d_2$ be positive integers. The two-layer ReLU network
\begin{equation}\label{eq:rela-net}
f(\bm{x})=\bm{W}_2\,\sigma(\bm{W}_1\bm{x}),\qquad \bm{W}_1\in\R^{d_1\times d_0},\quad \bm{W}_2\in\R^{d_2\times d_1},
\end{equation}
acts on $\bm{x}\in\R^{d_0}$ via $\sigma(z)=\max(z,0)$, $\sigma'(z)=\1[z>0]$. Let $\bm{h}(\bm{x})\eqdef\sigma(\bm{W}_1\bm{x})\in\R^{d_1}$, and let $\bm{w}_a\in\R^{d_0}$ denote the $a$-th row of $\bm{W}_1$. The population loss against a deterministic linear target $\bm{\Phi}\in\R^{d_2\times d_0}$ is
\begin{equation}\label{eq:rela-loss}
\bL(\bm{W}_1,\bm{W}_2)\;\eqdef\;\tfrac12\,\E_{\bm{x}}\norm{\bm{W}_2\sigma(\bm{W}_1\bm{x})-\bm{\Phi} \bm{x}}^2,
\end{equation}
and $\bm{G}_i\eqdef\partial\bL/\partial \bm{W}_i$, $i\in\{1,2\}$. For independent $\bm{u},\bm{v}\sim\bN(0,\bm{I}_{d_0})$ let $\theta\eqdef\angle(\bm{u},\bm{v})\in[0,\pi]$ be their angle, so that $\cos\theta=\langle \bm{u},\bm{v}\rangle/(\norm{\bm{u}}\norm{\bm{v}})$, and define
\begin{align}
\kappa_1 &\eqdef \E\!\left[\frac{\sin^2\theta+(\pi-\theta)^2}{4\pi^2}+\frac{\sin\theta\cos\theta(\pi-\theta)}{2\pi^2}\right], \label{eq:rela-kappa1}\\
\kappa_2 &\eqdef \frac{\E[\sin\theta(\pi-\theta)]}{2\pi^2}. \label{eq:rela-kappa2}
\end{align}

\subsection{Standing assumption}\label{app:relu:assump}

\begin{assumption}\label{ass:rela}
$(\bm{W}_1)_{ij}$ and $(\bm{W}_2)_{ij}$ are i.i.d.\ $\bN(0,1)$; $\bm{x}\sim\bN(0,\bm{I}_{d_0})$; $\bm{\Phi}$ is deterministic and independent of $\bm{W}_1,\bm{W}_2$. Expectations $\E[\cdot]$ are over $\bm{W}_1,\bm{W}_2$ ($\bm{x}$ has been integrated out in $\bL$).
\end{assumption}

\subsection{Main results restated}\label{app:relu:results}

\begin{theoremrst}{\cref{thm:relu}, restated}
Under \cref{ass:rela}, with $\bm{H}\eqdef\E_{\bm{x}}[\bm{h}(\bm{x})\bm{h}(\bm{x})^\top]$ and $\alpha\eqdef\E_{\bm{W}_1}[\tr(\bm{H}^2)]$,
\begin{equation}\label{eq:rela-G2}
\E[\bm{G}_2\bm{G}_2^\top]=\alpha\,\bm{I}_{d_2}+\tfrac{d_1}{4}\,\bm{\Phi}\bm{\Phi}^\top,\qquad\alpha=\Theta(d_1^2),
\end{equation}
and $\E[\bm{G}_1\bm{G}_1^\top]$ is permutation-invariant in the hidden index with
\begin{align}
\E\bigl[(\bm{G}_1\bm{G}_1^\top)_{aa}\bigr] &= \tfrac{d_0d_2^2}{4}+d_0d_2\bigl[(d_1-1)\kappa_1+\tfrac12\bigr]+\tfrac{\norm{\bm{\Phi}}_\Frob^2}{4}, \label{eq:rela-G1diag}\\
\E\bigl[(\bm{G}_1\bm{G}_1^\top)_{ac}\bigr] &= d_0d_2\,\kappa_2+O(d_2\sqrt{d_0})\qquad(a\neq c). \label{eq:rela-G1off}
\end{align}
As $d_0\to\infty$, $\kappa_1\to\tfrac{1}{4\pi^2}+\tfrac{1}{16}$ and $\kappa_2\to\tfrac{1}{4\pi}$.
\end{theoremrst}

\subsection{Preliminary lemmas}\label{app:relu:lemmas}

\begin{lemma}\label{lem:rela-quad}
Same as \cref{lem:mfa-quad}.
\end{lemma}

\begin{lemma}[Gaussian integration by parts]\label{lem:rela-stein}
Let $\bm{x}\sim\bN(0,\bm{I}_n)$ and $f:\R^n\to\R$ be locally absolutely continuous along almost every line parallel to the $j$-th axis, with weak derivative $\partial_jf$, and $\E|x_jf(\bm{x})|<\infty$, $\E|\partial_jf(\bm{x})|<\infty$. Then
\begin{equation}\label{eq:rela-stein}
\E[x_jf(\bm{x})]=\E[\partial_jf(\bm{x})].
\end{equation}
In particular this holds whenever $f$ is Lipschitz.
\end{lemma}

\begin{proof}
A standard identity~\citep{stein1981estimation,nourdin2012normal}. The standard Gaussian density $\varphi$ satisfies $\partial_j\varphi(\bm{x})=-x_j\varphi(\bm{x})$, so by Fubini and one-dimensional integration by parts in $x_j$ (the boundary term vanishes by integrability),
\[
\E[x_jf(\bm{x})]=-\int f\,\partial_j\varphi=\int\partial_jf\,\varphi=\E[\partial_jf(\bm{x})].
\]
The Lipschitz case follows from Rademacher's theorem.
\end{proof}

\begin{remark}\label{rem:rela-stein}
$z\mapsto\sigma(z)$ is $1$-Lipschitz, so $f(\bm{x})=\sigma(\bm{w}^\top \bm{x})$ falls under \cref{lem:rela-stein} with $\partial_jf(\bm{x})=\1[\bm{w}^\top \bm{x}>0]\,w_j$ a.e. By contrast $g(\bm{x})=\1[\bm{w}^\top \bm{x}>0]q(\bm{x})$ with $q$ not vanishing on $\{\bm{w}^\top \bm{x}=0\}$ is discontinuous across that hyperplane and is \emph{not} covered by \cref{lem:rela-stein}; for such integrands we use a half-space argument (\cref{lem:rela-halfspace}) or differentiate a separately established Lipschitz identity (\cref{lem:rela-cho}).
\end{remark}

\begin{lemma}\label{lem:rela-relu}
For $Z\sim\bN(0,s^2)$ with $s>0$,
\begin{equation}\label{eq:rela-relu}
\E[\sigma(Z)]=\frac{s}{\sqrt{2\pi}},\qquad \E[\sigma(Z)^2]=\frac{s^2}{2},\qquad \E[\sigma'(Z)]=\frac12.
\end{equation}
\end{lemma}

\begin{proof}
Direct computation: $\E[\sigma(Z)]=\int_0^\infty z\varphi_s(z)\,dz=s/\sqrt{2\pi}$; $\E[\sigma(Z)^2]=\int_0^\infty z^2\varphi_s(z)\,dz=s^2/2$; $\E[\sigma'(Z)]=\Pr(Z>0)=1/2$.
\end{proof}

\begin{lemma}[Cho--Saul kernel]\label{lem:rela-cho}
Let $\bm{u},\bm{v}\in\R^n\setminus\{0\}$, $\bm{x}\sim\bN(0,\bm{I}_n)$, $\theta\eqdef\angle(\bm{u},\bm{v})$. Then~\citep{cho2009kernel}
\begin{align}
\E[\sigma(\bm{u}^\top \bm{x})\sigma(\bm{v}^\top \bm{x})] &= \frac{\norm{\bm{u}}\norm{\bm{v}}}{2\pi}\bigl(\sin\theta+(\pi-\theta)\cos\theta\bigr), \label{eq:rela-K1}\\
\E[\sigma'(\bm{u}^\top \bm{x})\sigma'(\bm{v}^\top \bm{x})] &= \frac{\pi-\theta}{2\pi}. \label{eq:rela-K2}
\end{align}
\end{lemma}

\begin{proof}
Both quantities depend on $(\bm{u}^\top \bm{x},\bm{v}^\top \bm{x})$, bivariate Gaussian with variances $\norm{\bm{u}}^2,\norm{\bm{v}}^2$ and correlation $\cos\theta$. Equation \eqref{eq:rela-K2} is $\Pr(\bm{u}^\top \bm{x}>0,\,\bm{v}^\top \bm{x}>0)$, the angular measure of the cone $\{Z_1>0,Z_2>0\}$ for a standard bivariate Gaussian with correlation $\cos\theta$, namely $(\pi-\theta)/(2\pi)$. For \eqref{eq:rela-K1}, by Price's theorem~\citep{price1958useful},
\[
\frac{\partial}{\partial\rho}\E[\sigma(\bm{u}^\top \bm{x})\sigma(\bm{v}^\top \bm{x})]=\norm{\bm{u}}\norm{\bm{v}}\,\E[\sigma'(\bm{u}^\top \bm{x})\sigma'(\bm{v}^\top \bm{x})]=\frac{\norm{\bm{u}}\norm{\bm{v}}(\pi-\theta)}{2\pi};
\]
at $\rho=-1$ (i.e.\ $\theta=\pi$) the product vanishes a.s., and integrating $\partial_\rho$ from $-1$ to $\cos\theta$ gives \eqref{eq:rela-K1}.
\end{proof}

\begin{lemma}[Wick/Isserlis]\label{lem:rela-wick}
Same as \cref{lem:mfa-wick}.
\end{lemma}

\begin{lemma}\label{lem:rela-halfspace}
Let $\bm{x}\sim\bN(0,\bm{I}_n)$ and $\bm{u}\in\R^n\setminus\{0\}$. Then
\begin{equation}\label{eq:rela-halfspace}
\E\bigl[\bm{x}\bm{x}^\top\,\1[\bm{u}^\top \bm{x}>0]\bigr]=\tfrac12\,\bm{I}_n.
\end{equation}
\end{lemma}

\begin{proof}
Let $\bm{A}\eqdef\E[\bm{x}\bm{x}^\top\1[\bm{u}^\top \bm{x}>0]]$, and decompose $\bm{x}=(\bm{u}^\top \bm{x}/\norm{\bm{u}}^2)\bm{u}+\bm{x}_\perp$ with $\bm{x}_\perp$ the projection onto $\bm{u}^\perp$. In an orthonormal basis containing $\bm{u}/\norm{\bm{u}}$, the off-diagonal entries of $\bm{A}$ vanish by the reflection symmetry of $\bm{x}_\perp$; the entry along $\bm{u}/\norm{\bm{u}}$ is $\E[(\bm{u}^\top \bm{x}/\norm{\bm{u}})^2\1[\bm{u}^\top \bm{x}>0]]=\tfrac12$ ($\bm{u}^\top \bm{x}$ symmetric), and each remaining diagonal entry is $\tfrac12\E[x_j^2]=\tfrac12$ by independence of $\bm{x}_\perp$ from $\bm{u}^\top \bm{x}$ and $\Pr(\bm{u}^\top \bm{x}>0)=\tfrac12$. Hence $\bm{A}=\tfrac12\bm{I}_n$.
\end{proof}

\subsection{Proof of the gradient self outer-product of \texorpdfstring{$\bm{W}_2$}{W2}}\label{app:relu:proofG2}

\begin{lemma}\label{lem:rela-G2form}
$\bm{G}_2=\bm{W}_2\bm{H}-\bm{\Phi} \bm{C}^\top$, where $\bm{H}\eqdef\E_{\bm{x}}[\bm{h}(\bm{x})\bm{h}(\bm{x})^\top]\in\R^{d_1\times d_1}$ and $\bm{C}\eqdef\E_{\bm{x}}[\bm{h}(\bm{x})\bm{x}^\top]\in\R^{d_1\times d_0}$.
\end{lemma}

\begin{proof}
Differentiating \eqref{eq:rela-loss} in $\bm{W}_2$,
\[
\bm{G}_2=\E_{\bm{x}}[(\bm{W}_2\bm{h}(\bm{x})-\bm{\Phi} \bm{x})\bm{h}(\bm{x})^\top]=\bm{W}_2\,\E_{\bm{x}}[\bm{h}(\bm{x})\bm{h}(\bm{x})^\top]-\bm{\Phi}\,\E_{\bm{x}}[\bm{x}\bm{h}(\bm{x})^\top]=\bm{W}_2\bm{H}-\bm{\Phi} \bm{C}^\top.
\]
\end{proof}

\begin{lemma}\label{lem:rela-Cform}
$\bm{C}=\tfrac12\bm{W}_1$.
\end{lemma}

\begin{proof}
$\bm{C}_{aj}=\E_{\bm{x}}[\sigma(\bm{w}_a^\top \bm{x})x_j]$, and $f(\bm{x})\eqdef\sigma(\bm{w}_a^\top \bm{x})$ is Lipschitz with $\partial_jf(\bm{x})=\1[\bm{w}_a^\top \bm{x}>0](\bm{w}_a)_j$ a.e. By \cref{lem:rela-stein} and \cref{lem:rela-relu},
\[
\bm{C}_{aj}=\E_{\bm{x}}[\partial_jf(\bm{x})]=(\bm{w}_a)_j\,\E_{\bm{x}}[\sigma'(\bm{w}_a^\top \bm{x})]=\tfrac12(\bm{w}_a)_j.
\]
\end{proof}

\begin{proof}[{Proof of \eqref{eq:rela-G2}}]
By \cref{lem:rela-G2form},
\[
\bm{G}_2\bm{G}_2^\top=\bm{W}_2\bm{H}^2\bm{W}_2^\top-\bm{W}_2\bm{H}\bm{C}\bm{\Phi}^\top-\bm{\Phi} \bm{C}^\top \bm{H}\bm{W}_2^\top+\bm{\Phi} \bm{C}^\top \bm{C}\bm{\Phi}^\top.
\]
Since $\E[\bm{W}_2]=0$, the cross terms vanish under $\E_{\bm{W}_2}[\cdot\mid \bm{W}_1]$, and by \cref{lem:rela-quad} (left identity, $\bm{M}=\bm{H}^2$),
\begin{equation}\label{eq:rela-G2cond}
\E[\bm{G}_2\bm{G}_2^\top\mid \bm{W}_1]=\tr(\bm{H}^2)\,\bm{I}_{d_2}+\bm{\Phi} \bm{C}^\top \bm{C}\bm{\Phi}^\top.
\end{equation}
By \cref{lem:rela-Cform}, $\bm{C}^\top \bm{C}=\tfrac14\bm{W}_1^\top \bm{W}_1$, and \cref{lem:rela-quad} (right identity, $\bm{N}=\bm{I}_{d_1}$) gives $\E_{\bm{W}_1}[\bm{W}_1^\top \bm{W}_1]=d_1\bm{I}_{d_0}$, so $\E_{\bm{W}_1}[\bm{C}^\top \bm{C}]=\tfrac{d_1}{4}\bm{I}_{d_0}$. Taking $\E_{\bm{W}_1}$ of \eqref{eq:rela-G2cond},
\[
\E[\bm{G}_2\bm{G}_2^\top]=\E_{\bm{W}_1}[\tr(\bm{H}^2)]\,\bm{I}_{d_2}+\tfrac{d_1}{4}\bm{\Phi}\bm{\Phi}^\top=\alpha\,\bm{I}_{d_2}+\tfrac{d_1}{4}\bm{\Phi}\bm{\Phi}^\top.
\]
For the order of $\alpha$, split $\tr(\bm{H}^2)=\sum_aH_{aa}^2+\sum_{a\neq b}\bm{H}_{ab}^2$. By \cref{lem:rela-relu}, $\bm{H}_{aa}=\tfrac12\norm{\bm{w}_a}^2$ with $\norm{\bm{w}_a}^2\sim\chi^2_{d_0}$, so
\[
\sum_a\E[\bm{H}_{aa}^2]=\frac{d_1d_0(d_0+2)}{4}=\Theta(d_1).
\]
By \eqref{eq:rela-K1}, $\bm{H}_{ab}=\norm{\bm{w}_a}\norm{\bm{w}_b}\,g(\theta_{ab})$ with $g(\theta)=(\sin\theta+(\pi-\theta)\cos\theta)/(2\pi)$ and $\theta_{ab}=\angle(\bm{w}_a,\bm{w}_b)$ independent of $\norm{\bm{w}_a},\norm{\bm{w}_b}$, so $\E[\bm{H}_{ab}^2]=d_0^2\,\E[g(\theta)^2]$, where $\E[g(\theta)^2]>0$, and
\[
\sum_{a\neq b}\E[\bm{H}_{ab}^2]=d_1(d_1-1)\,d_0^2\,\E[g(\theta)^2]=\Theta(d_1^2).
\]
The off-diagonal term dominates, so $\alpha=\Theta(d_1^2)$.
\end{proof}

\subsection{Proof of the gradient self outer-product of \texorpdfstring{$\bm{W}_1$}{W1}}\label{app:relu:proofG1}

By the chain rule, with residual $\bm{r}(\bm{x})\eqdef \bm{W}_2\sigma(\bm{W}_1\bm{x})-\bm{\Phi} \bm{x}$ and $\bm{D}(\bm{x})\eqdef\diag(\sigma'(\bm{W}_1\bm{x}))$,
\begin{equation}\label{eq:rela-G1decomp}
\begin{gathered}
\bm{G}_1=\bm{G}_1^{(1)}-\bm{G}_1^{(2)},\\
\bm{G}_1^{(1)}\eqdef\E_{\bm{x}}[\bm{D}(\bm{x})\bm{W}_2^\top \bm{W}_2\sigma(\bm{W}_1\bm{x})\bm{x}^\top],\qquad
\bm{G}_1^{(2)}\eqdef\E_{\bm{x}}[\bm{D}(\bm{x})\bm{W}_2^\top\bm{\Phi} \bm{x}\bm{x}^\top].
\end{gathered}
\end{equation}
$\bm{G}_1^{(1)}$ is quadratic in $\bm{W}_2$ and $\bm{G}_1^{(2)}$ is linear, so the cross terms in $\bm{G}_1\bm{G}_1^\top$ are odd in $\bm{W}_2$ and vanish under $\E_{\bm{W}_2}[\cdot\mid \bm{W}_1]$:
\begin{equation}\label{eq:rela-parity}
\E_{\bm{W}_2}[\bm{G}_1\bm{G}_1^\top\mid \bm{W}_1]=\E_{\bm{W}_2}[\bm{G}_1^{(1)}(\bm{G}_1^{(1)})^\top\mid \bm{W}_1]+\E_{\bm{W}_2}[\bm{G}_1^{(2)}(\bm{G}_1^{(2)})^\top\mid \bm{W}_1].
\end{equation}

\begin{lemma}\label{lem:rela-G12}
$\bm{G}_1^{(2)}=\tfrac12\bm{W}_2^\top\bm{\Phi}$, and $\E_{\bm{W}_2}[\bm{G}_1^{(2)}(\bm{G}_1^{(2)})^\top\mid \bm{W}_1]=\tfrac{\norm{\bm{\Phi}}_\Frob^2}{4}\,\bm{I}_{d_1}$.
\end{lemma}

\begin{proof}
Let $\bm{v}_a\eqdef\bm{\Phi}^\top(\bm{W}_2)_{:,a}\in\R^{d_0}$, so $(\bm{W}_2^\top\bm{\Phi})_{ab}=(\bm{v}_a)_b$. Then, by \cref{lem:rela-halfspace},
\[
(\bm{G}_1^{(2)})_{ab}=\E_{\bm{x}}[\1[\bm{w}_a^\top \bm{x}>0]\,x_b\,(\bm{x}^\top \bm{v}_a)]=\bigl(\E_{\bm{x}}[\1[\bm{w}_a^\top \bm{x}>0]\,\bm{x}\bm{x}^\top]\,\bm{v}_a\bigr)_b=\tfrac12(\bm{v}_a)_b,
\]
so $\bm{G}_1^{(2)}=\tfrac12\bm{W}_2^\top\bm{\Phi}$ and $\bm{G}_1^{(2)}(\bm{G}_1^{(2)})^\top=\tfrac14\bm{W}_2^\top\bm{\Phi}\bm{\Phi}^\top \bm{W}_2$; \cref{lem:rela-quad} (right identity, $\bm{N}=\bm{\Phi}\bm{\Phi}^\top$) gives $\E_{\bm{W}_2}[\bm{W}_2^\top\bm{\Phi}\bm{\Phi}^\top \bm{W}_2]=\norm{\bm{\Phi}}_\Frob^2\bm{I}_{d_1}$.
\end{proof}

\begin{lemma}\label{lem:rela-Fb}
Define $F_b(\bm{u},\bm{v})\eqdef\E_{\bm{x}}[\sigma'(\bm{u}^\top \bm{x})\sigma(\bm{v}^\top \bm{x})x_b]$. Then $F_b(\bm{u},\bm{u})=\tfrac12u_b$, and for $\bm{u}\notin\mathrm{span}(\bm{v})$,
\begin{equation}\label{eq:rela-Fclosed}
F_b(\bm{u},\bm{v})=u_b\cdot\frac{\norm{\bm{v}}\sin\theta}{2\pi\norm{\bm{u}}}+v_b\cdot\frac{\pi-\theta}{2\pi},\qquad\theta\eqdef\angle(\bm{u},\bm{v}).
\end{equation}
\end{lemma}

\begin{proof}
For $\bm{u}=\bm{v}$, $\sigma'(\bm{u}^\top \bm{x})\sigma(\bm{u}^\top \bm{x})=\sigma(\bm{u}^\top \bm{x})$ a.e., so by \cref{lem:rela-stein} (applied to the Lipschitz $\bm{x}\mapsto\sigma(\bm{u}^\top \bm{x})$) and \cref{lem:rela-relu}, $F_b(\bm{u},\bm{u})=\E_{\bm{x}}[\sigma(\bm{u}^\top \bm{x})x_b]=\E_{\bm{x}}[\sigma'(\bm{u}^\top \bm{x})]u_b=\tfrac12u_b$. For $\bm{u}\neq \bm{v}$, $f(\bm{x})=\sigma'(\bm{u}^\top \bm{x})\sigma(\bm{v}^\top \bm{x})$ is discontinuous across $\{\bm{u}^\top \bm{x}=0\}$, so we differentiate the Cho--Saul kernel instead. With $K(\bm{u},\bm{v})\eqdef\E_{\bm{x}}[\sigma(\bm{u}^\top \bm{x})\sigma(\bm{v}^\top \bm{x})]$, dominated convergence (since $\sigma$ is Lipschitz with linear growth) gives
\[
\frac{\partial K}{\partial u_b}(\bm{u},\bm{v})=\E_{\bm{x}}[\sigma'(\bm{u}^\top \bm{x})x_b\sigma(\bm{v}^\top \bm{x})]=F_b(\bm{u},\bm{v}).
\]
By \cref{lem:rela-cho}, $K(\bm{u},\bm{v})=\tfrac{\norm{\bm{u}}\norm{\bm{v}}}{2\pi}(\sin\theta+(\pi-\theta)\cos\theta)$, smooth for $\bm{u}\notin\mathrm{span}(\bm{v})$; differentiating with $\partial\norm{\bm{u}}/\partial u_b=u_b/\norm{\bm{u}}$, $\partial\cos\theta/\partial u_b=\tfrac1{\norm{\bm{u}}}(v_b/\norm{\bm{v}}-\cos\theta\,u_b/\norm{\bm{u}})$, $\partial\theta/\partial u_b=-\tfrac1{\sin\theta}\partial\cos\theta/\partial u_b$, and simplifying yields \eqref{eq:rela-Fclosed}.
\end{proof}

\begin{lemma}\label{lem:rela-WtW}
$\E_{\bm{W}_2}[(\bm{W}_2^\top \bm{W}_2)_{ak}(\bm{W}_2^\top \bm{W}_2)_{cl}]=d_2^2\,\delta_{ak}\delta_{cl}+d_2\,\delta_{ac}\delta_{kl}+d_2\,\delta_{al}\delta_{ck}$.
\end{lemma}

\begin{proof}
$(\bm{W}_2^\top \bm{W}_2)_{ak}=\sum_{i=1}^{d_2}(\bm{W}_2)_{ia}(\bm{W}_2)_{ik}$, and applying \cref{lem:rela-wick} to the four Gaussian factors $(\bm{W}_2)_{ia}$, $(\bm{W}_2)_{ik}$, $(\bm{W}_2)_{jc}$, $(\bm{W}_2)_{jl}$ (summed over $i,j$) the three pairings give
\[
\delta_{ak}\delta_{cl}d_2^2+\delta_{ac}\delta_{kl}d_2+\delta_{al}\delta_{ck}d_2 .
\]
\end{proof}

\begin{proof}[{Proof of \eqref{eq:rela-G1diag}--\eqref{eq:rela-G1off}}]
Permutation invariance follows from the exchangeability of the rows of $\bm{W}_1$ and columns of $\bm{W}_2$. The $\bm{G}_1^{(2)}$ contribution to \eqref{eq:rela-parity} is $\tfrac{\norm{\bm{\Phi}}_\Frob^2}{4}\bm{I}_{d_1}$ by \cref{lem:rela-G12}. For $\bm{G}_1^{(1)}$, we have $(\bm{G}_1^{(1)})_{ab}=\sum_{k=1}^{d_1}(\bm{W}_2^\top \bm{W}_2)_{ak}F_b(\bm{w}_a,\bm{w}_k)$, so
\[
\E_{\bm{W}_2}\bigl[(\bm{G}_1^{(1)}(\bm{G}_1^{(1)})^\top)_{ac}\bigm|\bm{W}_1\bigr]=\sum_{b,k,l}\E_{\bm{W}_2}[(\bm{W}_2^\top \bm{W}_2)_{ak}(\bm{W}_2^\top \bm{W}_2)_{cl}]\,F_b(\bm{w}_a,\bm{w}_k)F_b(\bm{w}_c,\bm{w}_l),
\]
and substituting \cref{lem:rela-WtW} gives $T_1^{(ac)}+T_2^{(ac)}+T_3^{(ac)}$ with
\begin{align}
T_1^{(ac)} &= d_2^2\textstyle\sum_bF_b(\bm{w}_a,\bm{w}_a)F_b(\bm{w}_c,\bm{w}_c), \label{eq:rela-T1}\\
T_2^{(ac)} &= d_2\,\delta_{ac}\textstyle\sum_{k=1}^{d_1}\sum_bF_b(\bm{w}_a,\bm{w}_k)^2, \label{eq:rela-T2}\\
T_3^{(ac)} &= d_2\textstyle\sum_bF_b(\bm{w}_a,\bm{w}_c)F_b(\bm{w}_c,\bm{w}_a). \label{eq:rela-T3}
\end{align}

\emph{Term $T_1$.} By \cref{lem:rela-Fb}, $F_b(\bm{w}_a,\bm{w}_a)=\tfrac12(\bm{w}_a)_b$, so $T_1^{(ac)}=\tfrac{d_2^2}{4}\bm{w}_a^\top \bm{w}_c$, giving $\E[T_1^{(aa)}]=\tfrac{d_0d_2^2}{4}$ and $\E[T_1^{(ac)}]=0$ for $a\neq c$.

\emph{Term $T_2$.} For $\bm{u}\neq \bm{v}$, let $A\eqdef\tfrac{\norm{\bm{v}}\sin\theta}{2\pi\norm{\bm{u}}}$ and $B\eqdef\tfrac{\pi-\theta}{2\pi}$ in \eqref{eq:rela-Fclosed}; using $\bm{u}^\top \bm{v}=\norm{\bm{u}}\norm{\bm{v}}\cos\theta$,
\begin{gather*}
\sum_bF_b(\bm{u},\bm{v})^2=A^2\norm{\bm{u}}^2+2AB(\bm{u}^\top \bm{v})+B^2\norm{\bm{v}}^2=\norm{\bm{v}}^2\,\psi(\theta),\\
\psi(\theta)\eqdef\frac{\sin^2\theta+(\pi-\theta)^2}{4\pi^2}+\frac{\sin\theta\cos\theta(\pi-\theta)}{2\pi^2}.
\end{gather*}
In $T_2^{(aa)}$ the term $k=a$ contributes $\sum_bF_b(\bm{w}_a,\bm{w}_a)^2=\tfrac14\norm{\bm{w}_a}^2$ (expectation $\tfrac{d_0}{4}$), and each $k\neq a$ contributes $\E[\norm{\bm{w}_k}^2]\E[\psi(\theta_{ak})]=d_0\kappa_1$, so
\[
\E[T_2^{(aa)}]=d_2\Bigl[\tfrac{d_0}{4}+(d_1-1)d_0\kappa_1\Bigr]=d_0d_2\bigl[(d_1-1)\kappa_1+\tfrac14\bigr],
\]
and $T_2^{(ac)}=0$ for $a\neq c$.

\emph{Term $T_3$.} For $a=c$, the sum $\sum_bF_b(\bm{w}_a,\bm{w}_a)^2$ equals $\tfrac14\norm{\bm{w}_a}^2$, so $\E[T_3^{(aa)}]=\tfrac{d_0d_2}{4}$. For $a\neq c$, expanding $F_b(\bm{w}_a,\bm{w}_c)F_b(\bm{w}_c,\bm{w}_a)$ via \eqref{eq:rela-Fclosed} (with $\theta\eqdef\theta_{ac}$) and summing over $b$,
\[
\sum_bF_b(\bm{w}_a,\bm{w}_c)F_b(\bm{w}_c,\bm{w}_a)=\frac{\norm{\bm{w}_a}\norm{\bm{w}_c}}{4\pi^2}\Bigl[\cos\theta\bigl(\sin^2\theta+(\pi-\theta)^2\bigr)+2\sin\theta(\pi-\theta)\Bigr].
\]
Here $\norm{\bm{w}_a}\norm{\bm{w}_c}$ is independent of $\theta_{ac}$ with $\E[\norm{\bm{w}_a}\norm{\bm{w}_c}]=d_0+O(1)$; as $d_0\to\infty$, $\theta_{ac}$ concentrates at $\pi/2$ with $\E[\cos\theta_{ac}]=0$ and $\Var(\cos\theta_{ac})=1/d_0$, so $|\E[\cos\theta_{ac}\,q(\theta_{ac})]|=O(1/\sqrt{d_0})$ for bounded $q$, while $\E[\sin\theta_{ac}(\pi-\theta_{ac})]=\Theta(1)$. The $\cos\theta(\cdot)$ part contributes $O(d_2\sqrt{d_0})$ and the $2\sin\theta(\pi-\theta)$ part gives $d_0d_2\kappa_2+O(d_2)$, so $\E[T_3^{(ac)}]=d_0d_2\kappa_2+O(d_2\sqrt{d_0})$.

\emph{Combining.} For $a=c$, adding the three terms and the $\bm{G}_1^{(2)}$ contribution and merging the two $\tfrac{d_0d_2}{4}$ pieces,
\[
\E\bigl[(\bm{G}_1\bm{G}_1^\top)_{aa}\bigr]=\tfrac{d_0d_2^2}{4}+d_0d_2\bigl[(d_1-1)\kappa_1+\tfrac12\bigr]+\tfrac{\norm{\bm{\Phi}}_\Frob^2}{4};
\]
for $a\neq c$ only $T_3^{(ac)}$ survives, $\E[(\bm{G}_1\bm{G}_1^\top)_{ac}]=d_0d_2\kappa_2+O(d_2\sqrt{d_0})$. Finally $\psi(\pi/2)=\tfrac{1+(\pi/2)^2}{4\pi^2}=\tfrac1{4\pi^2}+\tfrac1{16}$ and $\sin(\pi/2)(\pi-\pi/2)/(2\pi^2)=\tfrac1{4\pi}$ give the high-dimensional limits of $\kappa_1,\kappa_2$.
\end{proof}

\subsection{Discussion}\label{app:relu:discussion}

At initialization the gradient of $\bm{W}_2$ is mostly target-independent: the $\Theta(d_1^2)$ leading order of $\E[(\bm{G}_2\bm{G}_2^\top)_{aa}]$ comes entirely from $\tr(\bm{H}^2)$, which depends only on $\bm{W}_1$, while the target term $\tfrac{d_1}{4}(\bm{\Phi}\bm{\Phi}^\top)_{aa}$ is dominated unless $\norm{\bm{\Phi}}_\Frob^2$ is exceptionally large. The nonzero $\Theta(1)$ off-diagonal of $\E[\bm{G}_1\bm{G}_1^\top]$ comes from $T_3$ and reflects ReLU's lack of odd symmetry: in a deep linear network the symmetry $\bm{W}_1\mapsto \bm{D}\bm{W}_1$, $\bm{W}_2\mapsto \bm{W}_2\bm{D}$ for $\bm{D}=\diag(\pm1)$ preserves the loss and forces off-diagonal expectations to zero (cf.\ \cref{app:dl}), whereas $\sigma(-z)\neq-\sigma(z)$ breaks it. The dominance is therefore asymptotic in $d_1$ rather than exact.
\begin{addblock}
\section{Single-Layer Linear Attention}\label{app:att}

\subsection{Setup}\label{app:att:setup}

Let $T,d_0,d_k,d_v,d_2$ be positive integers. Single-layer linear attention with tokens $\bm{X}\in\R^{T\times d_0}$ (one per row) and weights $\WQ,\WK\in\R^{d_k\times d_0}$, $\WV\in\R^{d_v\times d_0}$, $\WO\in\R^{d_2\times d_v}$, and its loss against a deterministic linear target $\bm{\Phi}\in\R^{d_2\times d_0}$, are
\begin{equation}\label{eq:atta-layer}
\bm{A}\eqdef\frac{1}{\sqrt{d_k}}\,\bm{X}\WQ^\top\WK\bm{X}^\top,
\qquad
f(\bm{X})=\bm{A}\,\bm{X}\WV^\top\WO^\top,
\qquad
\bL\eqdef\tfrac12\,\norm{f(\bm{X})-\bm{X}\bm{\Phi}^\top}_\Frob^2 .
\end{equation}
Let
\begin{equation}\label{eq:atta-notation}
\bm{Z}\eqdef\bm{X}\WV^\top,\qquad
\bm{R}\eqdef f(\bm{X})-\bm{X}\bm{\Phi}^\top,\qquad
\bm{Y}\eqdef\bm{X}\bm{\Phi}^\top,\qquad
\bm{C}\eqdef\bm{X}^\top\bm{X},\qquad
c_p\eqdef\tr(\bm{C}^p),
\end{equation}
with $c_0=\tr(\bm{C}^0)=d_0$, and define the three data moments
\begin{equation}\label{eq:atta-Adef}
\begin{gathered}
\mA_1\eqdef\E[c_1^2c_4{+}c_2c_4{+}c_2^3],\qquad
\mA_2\eqdef\E[c_1^2c_2^2{+}2c_2c_4],\\
\mA_3\eqdef\E[c_1^2c_4{+}c_1^2c_2^2{+}3c_2c_4{+}c_2^3]=\mA_1+\mA_2 ,
\end{gathered}
\end{equation}
which are explicit polynomials in $(T,d_0)$ (\cref{lem:atta-wm}). The gradients $\bm{G}_\bullet\eqdef\partial\bL/\partial\bm{W}_\bullet$, $\bullet\in\{Q,K,V,O\}$, are
\begin{equation}\label{eq:atta-grads}
\begin{aligned}
\GQ&=\tfrac{1}{\sqrt{d_k}}\,\WK\bm{X}^\top\bm{Z}\WO^\top\bm{R}^\top\bm{X},
&\qquad
\GK&=\tfrac{1}{\sqrt{d_k}}\,\WQ\bm{X}^\top\bm{R}\WO\bm{Z}^\top\bm{X},\\
\GV&=\WO^\top\bm{R}^\top\bm{A}\bm{X},
&\qquad
\GO&=\bm{R}^\top\bm{A}\bm{Z} .
\end{aligned}
\end{equation}
Muon acts per weight matrix, so the objects of study are the four gradient self outer-products $\GQ\GQ^\top$ and $\GK\GK^\top$ (both $d_k\times d_k$), $\GV\GV^\top$ ($d_v\times d_v$) and $\GO\GO^\top$ ($d_2\times d_2$), as in \cref{app:mf,app:dl,app:relu}.

\subsection{Standing assumption}\label{app:att:assump}

\begin{assumption}\label{ass:atta}
\sloppy
$\WQ,\WK,\WV,\WO$ are mutually independent with i.i.d.\ $\bN(0,1)$ entries; the rows of $\bm{X}$ are i.i.d.\ $\bN(\bm{0},\bm{I}_{d_0})$, independent of the weights, so $\bm{C}\sim\mathcal{W}_{d_0}(T,\bm{I}_{d_0})$; $\bm{\Phi}$ is deterministic. Expectations $\E[\cdot]$ are over data and weights jointly, and $\E[\cdot\,|\,\bm{X}]$ over the weights only.
\end{assumption}

This is \cref{ass:att} of the main text; fan-in scaling is treated in \cref{rem:atta-scale}.

\subsection{Main results restated}\label{app:att:results}

Every $d_k$ dependence below enters through the single ratio $\mA_2/\mA_1$ of the moments \eqref{eq:atta-Adef}.

\begin{theoremrst}{\cref{thm:att}, restated}
Under \cref{ass:atta}, $\E[\GQ\GQ^\top]=c_Q\bm{I}_{d_k}$, $\E[\GK\GK^\top]=c_K\bm{I}_{d_k}$ and $\E[\GV\GV^\top]=c_V\bm{I}_{d_v}$ with
\begin{equation}\label{eq:atta-cbars}
\begin{aligned}
c_Q&=\frac{1}{d_k}\Bigl(\E[c_2\beta_1]+d_v\gamma\norm{\bm{\Phi}}_\Frob^2\Bigr)
+\frac{1}{d_k^2}\Bigl(\E[c_2\beta_1]+\E[c_2\beta_2]\Bigr),\\
c_K&=\frac{1}{d_k}\Bigl(\E[c_2\beta_1]+d_v\gamma\norm{\bm{\Phi}}_\Frob^2\Bigr)
+\frac{1}{d_k^2}\,\E[(c_1^2{+}c_2)\beta_1],\\
c_V&=d_2(d_2{+}d_v{+}1)\Bigl(\mA_1+\tfrac1{d_k}\mA_3\Bigr)
+\gamma\norm{\bm{\Phi}}_\Frob^2,
\end{aligned}
\end{equation}
where $\beta_1,\beta_2$ are the polynomials \eqref{eq:atta-betas}; and
\begin{equation}\label{eq:atta-main}
\E\bigl[\GO\GO^\top\bigr]=\bar\alpha_O\,\bm{I}_{d_2}+d_v\,\gamma\;\bm{\Phi}\bm{\Phi}^\top,
\end{equation}
with
\begin{align}
\gamma&=\tfrac1{d_0}\,\E[c_2^2]
=T\Bigl[(Td_0{+}8)(T{+}d_0)^2+2(Td_0{+}10)(T{+}d_0)+5(Td_0{+}4)\Bigr],\label{eq:atta-gamma}\\
\bar\alpha_O
&=d_v(d_v{+}1)\,\mA_1+d_v\,\mA_2
+\frac{d_v(d_v{+}1)}{d_k}\,\mA_3+\frac{2d_v}{d_k}\,\mA_1
=\Theta(d_v^2d_0^6).\label{eq:atta-abar}
\end{align}
\end{theoremrst}

\begin{corollary}\label{cor:atta-scaling}
\sloppy
Under \cref{ass:atta} and for fixed $\bm{\Phi}$,
\[
c_Q\asymp c_K=\Theta\bigl(d_2^2d_v^2d_0^6/d_k\bigr),\qquad c_V=\Theta\bigl(d_2(d_2{+}d_v)\,d_0^6\bigr),\qquad \bar\alpha_O=\Theta\bigl(d_v^2d_0^6\bigr),
\]
with $c_K-c_Q=\bigl(\E[c_1^2\beta_1]-\E[c_2\beta_2]\bigr)/d_k^2$ of relative size $O(1/d_k)$; and for $a\neq b$ with $(\bm{\Phi}\bm{\Phi}^\top)_{ab}\neq0$,
\[
\frac{\E[(\GO\GO^\top)_{aa}]}{\E[(\GO\GO^\top)_{ab}]}=\Theta\bigl(d_vd_0^3\bigr).
\]
Thus the hidden dimension $d_v$ and the token dimension $d_0$ drive the diagonal dominance of $\E[\GO\GO^\top]$, while $d_k$ enters $c_V$, $\bar\alpha_O$ and the $\WO$ ratio only through bounded factors.
\end{corollary}

\subsection{Preliminary lemmas}\label{app:att:lemmas}

\Cref{lem:atta-wishart,lem:atta-N,lem:atta-S,lem:atta-E} are stated conditionally on $\bm{X}$, i.e.\ with $\bm{C}$ fixed and $c_p=\tr(\bm{C}^p)$; \cref{lem:atta-wm} supplies the moments of $\bm{C}\sim\mathcal{W}_{d_0}(T,\bm{I}_{d_0})$ needed to average over $\bm{X}$.

\begin{lemma}\label{lem:atta-quad}
Same as \cref{lem:mfa-quad}.
\end{lemma}

\begin{lemma}[Wick/Isserlis]\label{lem:atta-wick}
Same as \cref{lem:mfa-wick}.
\end{lemma}

\begin{lemma}\label{lem:atta-wishart}
Let $\bm{W}\in\R^{n\times d}$ have i.i.d.\ $\bN(0,1)$ entries, let $\bm{S}\eqdef\bm{W}^\top\bm{W}\sim\mathcal{W}_d(n,\bm{I}_d)$, and let $\bm{M}\in\R^{d\times d}$ be deterministic symmetric. Then
\begin{equation}\label{eq:atta-wishart}
\E[S_{ab}S_{cd}]=n^2\delta_{ab}\delta_{cd}+n(\delta_{ac}\delta_{bd}+\delta_{ad}\delta_{bc}),
\qquad
\E[\bm{S}\bm{M}\bm{S}]=n(n{+}1)\bm{M}+n\tr(\bm{M})\,\bm{I}_d,
\end{equation}
hence $\E[\tr(\bm{M}\bm{S}\bm{M}\bm{S})]=n(n{+}1)\tr(\bm{M}^2)+n(\tr\bm{M})^2$ and $\E[\bm{S}^2]=n(n{+}d{+}1)\bm{I}_d$.
\end{lemma}

\begin{proof}
Since $S_{ab}=\sum_{\alpha}W_{\alpha a}W_{\alpha b}$, \cref{lem:atta-wick} gives
\[
\E[S_{ab}S_{cd}]=\sum_{\alpha,\beta}\E[W_{\alpha a}W_{\alpha b}W_{\beta c}W_{\beta d}]
=\sum_{\alpha,\beta}\bigl(\delta_{ab}\delta_{cd}+\delta_{\alpha\beta}\delta_{ac}\delta_{bd}+\delta_{\alpha\beta}\delta_{ad}\delta_{bc}\bigr),
\]
which is the first identity. Contracting it against $M_{bc}$ and using $\bm{M}=\bm{M}^\top$,
\[
\E[(\bm{S}\bm{M}\bm{S})_{ad}]=\sum_{b,c}M_{bc}\,\E[S_{ab}S_{cd}]=n^2M_{ad}+nM_{ad}+n\,\delta_{ad}\tr(\bm{M}),
\]
which is the second identity; contracting once more against $M_{da}$ gives the trace identity, and $\bm{M}=\bm{I}_d$ gives $\E[\bm{S}^2]$.
\end{proof}

\begin{lemma}\label{lem:atta-N}
Under \cref{ass:atta}, conditionally on $\bm{X}$, let $\bm{N}\eqdef\WQ\bm{C}\WQ^\top\in\R^{d_k\times d_k}$. Then
\begin{equation}\label{eq:atta-N}
\begin{gathered}
\E[N_{ij}N_{kl}\,|\,\bm{X}]=c_1^2\,\delta_{ij}\delta_{kl}+c_2(\delta_{ik}\delta_{jl}+\delta_{il}\delta_{jk}),\\
\E[(\tr\bm{N})^2\,|\,\bm{X}]=d_k^2c_1^2+2d_kc_2,
\qquad
\E[\tr\bm{N}^2\,|\,\bm{X}]=d_k^2c_2+d_k(c_1^2{+}c_2).
\end{gathered}
\end{equation}
The same identities hold with $\bm{C}$ replaced by any deterministic symmetric $\bm{E}\in\R^{d_0\times d_0}$ and $(c_1,c_2)$ by $(\tr\bm{E},\tr\bm{E}^2)$, and with $\WQ$ replaced by $\WK$.
\end{lemma}

\begin{proof}
Since $N_{ij}=\sum_{a,b}(\WQ)_{ia}C_{ab}(\WQ)_{jb}$, \cref{lem:atta-wick} gives
\[
\E[(\WQ)_{ia}(\WQ)_{jb}(\WQ)_{kc}(\WQ)_{ld}]=\delta_{ij}\delta_{ab}\delta_{kl}\delta_{cd}+\delta_{ik}\delta_{ac}\delta_{jl}\delta_{bd}+\delta_{il}\delta_{ad}\delta_{jk}\delta_{bc},
\]
and contracting against $C_{ab}C_{cd}$, with $\sum_{a,b}C_{ab}C_{ab}=\tr(\bm{C}^2)=c_2$, gives the first line of \eqref{eq:atta-N}. Summing it over $i=j$ and $k=l$ gives $\E[(\tr\bm{N})^2\,|\,\bm{X}]$; summing it over $j=k$ and $l=i$ gives $\E[\tr\bm{N}^2\,|\,\bm{X}]$. Only the symmetry of $\bm{C}$ and the law of $\WQ$ were used, which gives the last assertion. (These are \cref{lem:dla-trsq,lem:dla-fourth} with $m=d_k$ and $\bm{M}=\bm{C}$.)
\end{proof}

\begin{lemma}\label{lem:atta-S}
Under \cref{ass:atta}, conditionally on $\bm{X}$, let $\bm{N}\in\R^{d_k\times d_k}$ be symmetric and independent of $\WK$, and let $\bm{S}\eqdef\WK^\top\bm{N}\WK$. Then
$\E[S_{ab}S_{cd}\,|\,\bm{X},\bm{N}]=(\tr\bm{N})^2\delta_{ab}\delta_{cd}+\tr(\bm{N}^2)(\delta_{ac}\delta_{bd}+\delta_{ad}\delta_{bc})$, and
\begin{equation}\label{eq:atta-cond}
\begin{aligned}
\E[\tr(\bm{S}\bm{C}^2\bm{S}\bm{C}^2)\,|\,\bm{X},\bm{N}]&=(\tr\bm{N})^2c_4+\tr(\bm{N}^2)(c_4{+}c_2^2),\\
\E[(\tr(\bm{S}\bm{C}^2))^2\,|\,\bm{X},\bm{N}]&=(\tr\bm{N})^2c_2^2+2\tr(\bm{N}^2)\,c_4.
\end{aligned}
\end{equation}
\end{lemma}

\begin{proof}
Since $S_{ab}=\sum_{i,j}(\WK)_{ia}N_{ij}(\WK)_{jb}$, \cref{lem:atta-wick} applied to $\E[(\WK)_{ia}(\WK)_{jb}(\WK)_{kc}(\WK)_{ld}]$ as in the proof of \cref{lem:atta-N}, followed by contraction against $N_{ij}N_{kl}$ with $\bm{N}=\bm{N}^\top$, gives the first identity. Contracting the first identity against $(\bm{C}^2)_{bc}(\bm{C}^2)_{da}$, the three terms give $(\tr\bm{N})^2\tr(\bm{C}^4)$, $\tr(\bm{N}^2)\tr(\bm{C}^4)$ and $\tr(\bm{N}^2)\tr(\bm{C}^2)^2$, which is the first line of \eqref{eq:atta-cond}; contracting against $(\bm{C}^2)_{ba}(\bm{C}^2)_{dc}$ instead, they give $(\tr\bm{N})^2\tr(\bm{C}^2)^2$, $\tr(\bm{N}^2)\tr(\bm{C}^4)$ and $\tr(\bm{N}^2)\tr(\bm{C}^4)$, which is the second line.
\end{proof}

\begin{lemma}\label{lem:atta-E}
\sloppy
Under \cref{ass:atta}, conditionally on $\bm{X}$, let $\bm{H}_O\eqdef\WO^\top\WO\sim\mathcal{W}_{d_v}(d_2,\bm{I}_{d_v})$ and $\bm{E}\eqdef\bm{C}\WV^\top\bm{H}_O\WV\bm{C}$, and define
\begin{equation}\label{eq:atta-betas}
\begin{aligned}
\beta_1&\eqdef d_2d_v\Bigl[(d_2d_v{+}d_2{+}d_v{+}3)\,c_4+(d_2{+}d_v{+}1)\,c_2^2\Bigr],\\
\beta_2&\eqdef d_2d_v\Bigl[(d_2d_v{+}2)\,c_2^2+2(d_2{+}d_v{+}1)\,c_4\Bigr].
\end{aligned}
\end{equation}
Then $\E[\tr(\bm{E}^2)\,|\,\bm{X}]=\beta_1$ and $\E[(\tr\bm{E})^2\,|\,\bm{X}]=\beta_2$.
\end{lemma}

\begin{proof}
\sloppy
Let $\bm{S}_V\eqdef\WV^\top\bm{H}_O\WV$, so that $\tr\bm{E}=\tr(\bm{S}_V\bm{C}^2)$ and $\tr\bm{E}^2=\tr(\bm{S}_V\bm{C}^2\bm{S}_V\bm{C}^2)$. Since $\bm{H}_O$ is symmetric and independent of $\WV$, \cref{lem:atta-S} applies with $(\WV,\bm{H}_O)$ in place of $(\WK,\bm{N})$:
\[
\begin{gathered}
\E[\tr\bm{E}^2\,|\,\bm{X},\bm{H}_O]=(\tr\bm{H}_O)^2c_4+\tr(\bm{H}_O^2)(c_4{+}c_2^2),\\
\E[(\tr\bm{E})^2\,|\,\bm{X},\bm{H}_O]=(\tr\bm{H}_O)^2c_2^2+2\tr(\bm{H}_O^2)\,c_4.
\end{gathered}
\]
By \cref{lem:atta-wishart} with $(n,d)=(d_2,d_v)$, $\E[(\tr\bm{H}_O)^2]=d_2d_v(d_2d_v{+}2)$ and $\E[\tr\bm{H}_O^2]=d_2d_v(d_2{+}d_v{+}1)$; substituting gives $\beta_1$ and $\beta_2$.
\end{proof}

\begin{lemma}[Wishart trace moments]\label{lem:atta-wm}
\sloppy
For $\bm{C}\sim\mathcal{W}_{d_0}(T,\bm{I}_{d_0})$: $\E[c_1]=Td_0$, $\E[c_2]=Td_0(T{+}d_0{+}1)$, and
\begin{align}
\E[c_4]&=Td_0\Bigl[(T{+}d_0)^3+6(T{+}d_0)^2+3(Td_0{+}7)(T{+}d_0)+5(Td_0{+}4)\Bigr],\label{eq:atta-wm-c4}\\
\E[c_2^2]&=Td_0\Bigl[(Td_0{+}8)(T{+}d_0)^2+2(Td_0{+}10)(T{+}d_0)+5(Td_0{+}4)\Bigr],\label{eq:atta-wm-c22}\\
\E[c_2c_4]&=Td_0\Bigl[(Td_0{+}16)(T{+}d_0)^4+(7Td_0{+}136)(T{+}d_0)^3\notag\\
&\hphantom{=Td_0\Bigl[}+\bigl(3T^2d_0^2{+}99Td_0{+}736\bigr)(T{+}d_0)^2+\bigl(8T^2d_0^2{+}305Td_0{+}1608\bigr)(T{+}d_0)\notag\\
&\hphantom{=Td_0\Bigl[}+21T^2d_0^2{+}420Td_0{+}1344\Bigr],\label{eq:atta-wm-c2c4}\\
\E[c_2^3]&=Td_0\Bigl[\bigl(T^2d_0^2{+}24Td_0{+}160\bigr)(T{+}d_0)^3+\bigl(3T^2d_0^2{+}84Td_0{+}704\bigr)(T{+}d_0)^2\notag\\
&\hphantom{=Td_0\Bigl[}+\bigl(15T^2d_0^2{+}344Td_0{+}1664\bigr)(T{+}d_0)+13T^2d_0^2{+}380Td_0{+}1312\Bigr],\label{eq:atta-wm-c23}\\
\E[c_1^2c_4]&=(Td_0{+}8)(Td_0{+}10)\,\E[c_4],
\qquad
\E[c_1^2c_2^2]=(Td_0{+}8)(Td_0{+}10)\,\E[c_2^2].\label{eq:atta-wm-fact}
\end{align}
Consequently $\mA_3=\mA_1+\mA_2$. As polynomials in $d_0$, all six moments have positive coefficients, of degree $4$ for $\E[c_4],\E[c_2^2]$ and degree $6$ for the other four; hence $\E[c_4],\E[c_2^2]=\Theta(d_0^4)$, $\E[c_2c_4],\E[c_2^3],\E[c_1^2c_4],\E[c_1^2c_2^2]=\Theta(d_0^6)$, and therefore $\gamma=\tfrac1{d_0}\E[c_2^2]=\Theta(d_0^3)$ and $\mA_1,\mA_2=\Theta(d_0^6)$.
\end{lemma}

\begin{proof}
\sloppy
Throughout the proof let $s\eqdef T{+}d_0$ and $q\eqdef Td_0$. Since $\bm{C}=\bm{X}^\top\bm{X}$ with $\bm{X}\in\R^{T\times d_0}$ standard Gaussian, every moment below is the expectation of a polynomial in the entries of $\bm{X}$, and Gaussian integration by parts, $\E[X_{ta}\Psi(\bm{X})]=\E[\partial\Psi/\partial X_{ta}]$ for polynomial $\Psi$, applies.

\emph{Step 1: a recursion in the total degree.} For $k\ge1$ and $g\eqdef\prod_{i=1}^{n}c_{m_i}$ with $m_i\ge1$ ($g=1$ if $n=0$),
\begin{equation}\label{eq:atta-wm-rec}
\E[c_kg]=\E\Bigl[\Bigl((T{+}k{-}1)\,c_{k-1}+\sum_{j=1}^{k-1}c_j\,c_{k-1-j}\Bigr)g\Bigr]+2\sum_{i=1}^{n}m_i\,\E\Bigl[c_{k+m_i-1}\prod_{l\neq i}c_{m_l}\Bigr].
\end{equation}
Indeed, $c_k=\tr(\bm{X}^\top\bm{X}\bm{C}^{k-1})=\sum_{t,a}X_{ta}(\bm{X}\bm{C}^{k-1})_{ta}$, so integration by parts in each $X_{ta}$ gives
\begin{equation}\label{eq:atta-wm-ibp}
\E[c_kg]=\sum_{t,a}\E\Bigl[\frac{\partial(\bm{X}\bm{C}^{k-1})_{ta}}{\partial X_{ta}}\,g\Bigr]+\sum_{t,a}\E\Bigl[(\bm{X}\bm{C}^{k-1})_{ta}\,\frac{\partial g}{\partial X_{ta}}\Bigr].
\end{equation}
From $\partial C_{bc}/\partial X_{ta}=\delta_{ab}X_{tc}+\delta_{ac}X_{tb}$ and the product rule, for $m\ge1$,
\begin{equation}\label{eq:atta-wm-deriv}
\begin{gathered}
\frac{\partial(\bm{C}^m)_{bc}}{\partial X_{ta}}=\sum_{r=0}^{m-1}\Bigl[(\bm{C}^r)_{ba}\,(\bm{X}\bm{C}^{m-1-r})_{tc}+(\bm{C}^r\bm{X}^\top)_{bt}\,(\bm{C}^{m-1-r})_{ac}\Bigr],\\
\frac{\partial c_m}{\partial X_{ta}}=\sum_{b}\frac{\partial(\bm{C}^m)_{bb}}{\partial X_{ta}}=2m\,(\bm{X}\bm{C}^{m-1})_{ta}.
\end{gathered}
\end{equation}
Since $(\bm{X}\bm{C}^{k-1})_{ta}=\sum_bX_{tb}(\bm{C}^{k-1})_{ba}$, \eqref{eq:atta-wm-deriv} with $m=k-1$ gives
\begin{align*}
\sum_{t,a}\frac{\partial(\bm{X}\bm{C}^{k-1})_{ta}}{\partial X_{ta}}
&=\sum_{t,a}(\bm{C}^{k-1})_{aa}
+\sum_{r=0}^{k-2}\Bigl[\sum_{t,a}(\bm{X}\bm{C}^{r})_{ta}(\bm{X}\bm{C}^{k-2-r})_{ta}\Bigr]\\
&\qquad+\sum_{r=0}^{k-2}\tr(\bm{X}\bm{C}^{r}\bm{X}^\top)\,\tr(\bm{C}^{k-2-r})\\
&=T\,c_{k-1}+(k{-}1)\,c_{k-1}+\sum_{r=0}^{k-2}c_{r+1}\,c_{k-2-r}
=(T{+}k{-}1)\,c_{k-1}+\sum_{j=1}^{k-1}c_j\,c_{k-1-j},
\end{align*}
since $\sum_{t,a}(\bm{X}\bm{C}^{r})_{ta}(\bm{X}\bm{C}^{k-2-r})_{ta}=\tr(\bm{C}^{r}\bm{X}^\top\bm{X}\bm{C}^{k-2-r})=c_{k-1}$ and $\tr(\bm{X}\bm{C}^{r}\bm{X}^\top)=c_{r+1}$; this is the first term of \eqref{eq:atta-wm-rec}. For the second, $\partial g/\partial X_{ta}=\sum_i(\partial c_{m_i}/\partial X_{ta})\prod_{l\neq i}c_{m_l}$ and, by \eqref{eq:atta-wm-deriv},
\begin{align*}
\sum_{t,a}(\bm{X}\bm{C}^{k-1})_{ta}\,\frac{\partial c_{m}}{\partial X_{ta}}
&=2m\sum_{t,a}(\bm{X}\bm{C}^{k-1})_{ta}(\bm{X}\bm{C}^{m-1})_{ta}\\
&=2m\,\tr(\bm{C}^{k-1}\bm{X}^\top\bm{X}\bm{C}^{m-1})=2m\,c_{k+m-1}.
\end{align*}

\emph{Step 2: moments containing $c_1$, and \eqref{eq:atta-wm-fact}.} For $k=1$ the first bracket of \eqref{eq:atta-wm-rec} is $Tc_0=q$ and each $c_{m_i}$ is reproduced, so
\begin{equation}\label{eq:atta-wm-c1}
\E[c_1g]=\bigl(q+2\deg g\bigr)\E[g],\qquad \deg g\eqdef\textstyle\sum_im_i,
\end{equation}
and iterating, $\E[c_1^2F]=(q{+}2d{+}2)(q{+}2d)\,\E[F]$ for $F$ homogeneous of degree $d$ in $\bm{C}$; $d=4$ gives \eqref{eq:atta-wm-fact}.

\emph{Step 3: degree at most four.} By \eqref{eq:atta-wm-rec} with $g=1$ and $k=1,2,3$, and by \eqref{eq:atta-wm-c1},
\begin{align*}
\E[c_1]&=q,\qquad
\E[c_1^2]=(q{+}2)\,q,\qquad
\E[c_2]=(T{+}1{+}d_0)\E[c_1]=q(s{+}1),\\
\E[c_1c_2]&=(q{+}4)\E[c_2]=q(q{+}4)(s{+}1),\\
\E[c_3]&=(T{+}2{+}d_0)\E[c_2]+\E[c_1^2]=q\bigl[(s{+}1)(s{+}2)+q+2\bigr]=q\bigl[s^2+3s+q+4\bigr].
\end{align*}
Then \eqref{eq:atta-wm-rec} with $(k,g)=(4,1)$ and $(2,c_2)$ gives
\begin{align*}
\E[c_4]&=(T{+}3{+}d_0)\E[c_3]+2\E[c_1c_2]\\
&=q\bigl[(s{+}3)(s^2{+}3s{+}q{+}4)+2(s{+}1)(q{+}4)\bigr]
=q\bigl[s^3+6s^2+3(q{+}7)s+5(q{+}4)\bigr],\\
\E[c_2^2]&=(T{+}1{+}d_0)\E[c_1c_2]+4\E[c_3]\\
&=q\bigl[(s{+}1)^2(q{+}4)+4(s^2{+}3s{+}q{+}4)\bigr]
=q\bigl[(q{+}8)s^2+2(q{+}10)s+5(q{+}4)\bigr],
\end{align*}
which are \eqref{eq:atta-wm-c4} and \eqref{eq:atta-wm-c22}.

\emph{Step 4: degree five.} By \eqref{eq:atta-wm-rec} with $(k,g)=(5,1)$ and $(3,c_2)$, and by \eqref{eq:atta-wm-c1} for $\E[c_1c_3]=(q{+}6)\E[c_3]$ and $\E[c_1^2c_2]=(q{+}4)(q{+}6)\E[c_2]$,
\begin{align*}
\E[c_5]&=(T{+}4{+}d_0)\E[c_4]+2\E[c_1c_3]+\E[c_2^2]\\
&=q\bigl[s^4+10s^3+(6q{+}65)s^2+5(5q{+}32)s+(q{+}4)(2q{+}37)\bigr],\\
\E[c_2c_3]&=(T{+}2{+}d_0)\E[c_2^2]+\E[c_1^2c_2]+4\E[c_4]\\
&=q\bigl[(q{+}12)s^3+4(q{+}15)s^2+(q{+}7)(q{+}24)s+(q{+}4)(q{+}36)\bigr].
\end{align*}

\emph{Step 5: degree six.} By \eqref{eq:atta-wm-rec} with $(k,g)=(4,c_2)$ and $(2,c_2^2)$, and by \eqref{eq:atta-wm-c1} for $\E[c_1c_2^2]=(q{+}8)\E[c_2^2]$,
\begin{align*}
\E[c_2c_4]&=(T{+}3{+}d_0)\E[c_2c_3]+2\E[c_1c_2^2]+4\E[c_5]
=(s{+}3)\E[c_2c_3]+2(q{+}8)\E[c_2^2]+4\E[c_5],\\
\E[c_2^3]&=(T{+}1{+}d_0)\E[c_1c_2^2]+8\E[c_2c_3]
=(s{+}1)(q{+}8)\E[c_2^2]+8\E[c_2c_3];
\end{align*}
substituting Steps 3 and 4 and collecting powers of $s$ gives \eqref{eq:atta-wm-c2c4} and \eqref{eq:atta-wm-c23}.

\emph{Step 6: orders in $d_0$.} All coefficients above are positive, so each moment is $\Theta$ of its highest power of $d_0$, read off from the displayed formulas: $d_0\cdot d_0^3$ for $\E[c_4]$, $d_0\cdot d_0\cdot d_0^2$ for $\E[c_2^2]$, $d_0\cdot d_0\cdot d_0^4$ for $\E[c_2c_4]$, $d_0\cdot d_0^2\cdot d_0^3$ for $\E[c_2^3]$, and $d_0^2$ times $\E[c_4]$ and $\E[c_2^2]$ for \eqref{eq:atta-wm-fact}. Since $\mA_1$ and $\mA_2$ are positive combinations of degree-six moments by \eqref{eq:atta-Adef}, they are $\Theta(d_0^6)$; $\mA_3=\mA_1+\mA_2$ holds by definition.
\end{proof}

\subsection{Proofs}\label{app:att:proofs}

The proof of \cref{thm:att} has three steps: the isotropy of $\E[\bm{G}_\bullet\bm{G}_\bullet^\top\,|\,\bm{X}]$ for $\bullet\in\{Q,K,V\}$ (\cref{lem:atta-iso}); the conditional closed forms of $\E[\GO\GO^\top\,|\,\bm{X}]$ (\cref{lem:atta-condwo}) and of the isotropic constants (\cref{lem:atta-condconsts}); and the average over $\bm{X}$.

\begin{lemma}[Isotropy]\label{lem:atta-iso}
Under \cref{ass:atta}, conditionally on $\bm{X}$,
\[
\begin{gathered}
\E[\GQ\GQ^\top\,|\,\bm{X}]=c_Q(\bm{C})\,\bm{I}_{d_k},\qquad
\E[\GK\GK^\top\,|\,\bm{X}]=c_K(\bm{C})\,\bm{I}_{d_k},\\
\E[\GV\GV^\top\,|\,\bm{X}]=c_V(\bm{C})\,\bm{I}_{d_v},
\end{gathered}
\]
with $c_Q(\bm{C})\eqdef\tfrac1{d_k}\E[\norm{\GQ}_\Frob^2\,|\,\bm{X}]$, $c_K(\bm{C})\eqdef\tfrac1{d_k}\E[\norm{\GK}_\Frob^2\,|\,\bm{X}]$ and $c_V(\bm{C})\eqdef\tfrac1{d_v}\E[\norm{\GV}_\Frob^2\,|\,\bm{X}]$.
\end{lemma}

\begin{proof}
Fix $\bm{X}$ and let $W\eqdef(\WQ,\WK,\WV,\WO)$. Let $\bm{D}\in\R^{d_k\times d_k}$ be diagonal with entries in $\{\pm1\}$, and let $\tau_{\bm{D}}$ denote the map $(\WQ,\WK,\WV,\WO)\mapsto(\bm{D}\WQ,\bm{D}\WK,\WV,\WO)$. Since $(\bm{D}\WQ)^\top(\bm{D}\WK)=\WQ^\top\WK$, the matrix $\bm{A}$ in \eqref{eq:atta-layer}, and hence $\bL$, is invariant: $\bL\circ\tau_{\bm{D}}=\bL$. Differentiating this identity with respect to $\WQ$ by the chain rule gives $\bm{D}^\top\GQ(\tau_{\bm{D}}W)=\GQ(W)$, i.e.
\begin{equation}\label{eq:atta-invariance}
\GQ(\tau_{\bm{D}}W)=\bm{D}\,\GQ(W).
\end{equation}
Under \cref{ass:atta}, $\tau_{\bm{D}}W\stackrel{d}{=}W$, so
\[
\E[\GQ\GQ^\top\,|\,\bm{X}]=\E\bigl[\GQ(\tau_{\bm{D}}W)\GQ(\tau_{\bm{D}}W)^\top\,\big|\,\bm{X}\bigr]=\bm{D}\,\E[\GQ\GQ^\top\,|\,\bm{X}]\,\bm{D}.
\]
For $a\neq b$, choosing $D_{aa}=1$ and $D_{bb}=-1$ gives $\E[(\GQ\GQ^\top)_{ab}\,|\,\bm{X}]=-\E[(\GQ\GQ^\top)_{ab}\,|\,\bm{X}]$, hence $0$. Likewise, for a permutation matrix $\bm{P}\in\R^{d_k\times d_k}$ the map $(\WQ,\WK)\mapsto(\bm{P}\WQ,\bm{P}\WK)$ preserves $\WQ^\top\WK$ and the law of $W$, and the same argument gives $\E[\GQ\GQ^\top\,|\,\bm{X}]=\bm{P}\,\E[\GQ\GQ^\top\,|\,\bm{X}]\,\bm{P}^\top$; hence all diagonal entries are equal, and their common value is $\tfrac1{d_k}\E[\tr(\GQ\GQ^\top)\,|\,\bm{X}]=c_Q(\bm{C})$. The same argument with $\WQ$ and $\WK$ exchanged gives the claim for $\GK$.

For $\GV$, let $\bm{D}\in\R^{d_v\times d_v}$ be diagonal with entries in $\{\pm1\}$ and let $\tau_{\bm{D}}$ map $(\WV,\WO)\mapsto(\bm{D}\WV,\WO\bm{D})$, fixing $\WQ,\WK$. Since $f(\bm{X})=\bm{A}\bm{X}(\WO\WV)^\top$ depends on $(\WV,\WO)$ only through the product $\WO\WV$, which equals $(\WO\bm{D})(\bm{D}\WV)$, $\bL\circ\tau_{\bm{D}}=\bL$; differentiating with respect to $\WV$ gives $\GV(\tau_{\bm{D}}W)=\bm{D}\,\GV(W)$, and $(\bm{D}\WV,\WO\bm{D})\stackrel{d}{=}(\WV,\WO)$ under \cref{ass:atta}. The sign and permutation arguments above apply verbatim.
\end{proof}

\begin{lemma}[Conditional closed form for $\WO$]\label{lem:atta-condwo}
Under \cref{ass:atta}, conditionally on $\bm{X}$,
\begin{equation}\label{eq:atta-wo-cond}
\E\bigl[\GO\GO^\top\,\big|\,\bm{X}\bigr]=\alpha_O(\bm{C})\,\bm{I}_{d_2}+d_v\,c_2\;\bm{\Phi}\,\bm{C}^2\,\bm{\Phi}^\top,
\end{equation}
\begin{equation}\label{eq:atta-alphaO}
\begin{aligned}
\alpha_O(\bm{C})
&=d_v(d_v{+}1)\bigl(c_1^2c_4{+}c_2c_4{+}c_2^3\bigr)+d_v\bigl(c_1^2c_2^2{+}2c_2c_4\bigr)\\
&\quad+\tfrac{d_v(d_v{+}1)}{d_k}\bigl(c_1^2c_4{+}c_1^2c_2^2{+}3c_2c_4{+}c_2^3\bigr)+\tfrac{2d_v}{d_k}\bigl(c_1^2c_4{+}c_2c_4{+}c_2^3\bigr).
\end{aligned}
\end{equation}
\end{lemma}

\begin{proof}
\sloppy
Let $\bm{\Omega}\eqdef\bm{A}\bm{Z}\bm{Z}^\top\bm{A}^\top$. Substituting $\bm{R}=\bm{A}\bm{Z}\WO^\top-\bm{Y}$ into $\GO\GO^\top=\bm{R}^\top\bm{\Omega}\bm{R}$,
\begin{equation}\label{eq:atta-fourterms}
\GO\GO^\top
=\underbrace{\WO\bm{Z}^\top\bm{A}^\top\bm{\Omega}\bm{A}\bm{Z}\WO^\top}_{\mathrm{T}_1}
-\underbrace{\WO\bm{Z}^\top\bm{A}^\top\bm{\Omega}\,\bm{Y}\vphantom{\WO^\top}}_{\mathrm{T}_2}
-\underbrace{\bm{Y}^\top\bm{\Omega}\bm{A}\bm{Z}\WO^\top}_{\mathrm{T}_3}
+\underbrace{\bm{Y}^\top\bm{\Omega}\,\bm{Y}\vphantom{\WO^\top}}_{\mathrm{T}_4}.
\end{equation}
The terms $\mathrm{T}_2,\mathrm{T}_3$ are linear in $\WO$, which is centered and independent of $(\bm{X},\WQ,\WK,\WV)$, so $\E[\mathrm{T}_2\,|\,\bm{X}]=\E[\mathrm{T}_3\,|\,\bm{X}]=\bm{0}$.

\emph{The term $\mathrm{T}_4$.} By \cref{lem:atta-quad}, $\E[\bm{Z}\bm{Z}^\top\,|\,\bm{X},\WQ,\WK]=\bm{X}\,\E[\WV^\top\WV]\,\bm{X}^\top=d_v\bm{X}\bm{X}^\top$, hence
\[
\begin{gathered}
\E[\mathrm{T}_4\,|\,\bm{X},\WQ,\WK]=d_v\,\bm{\Phi}\bigl(\bm{X}^\top\bm{A}\,\bm{X}\bm{X}^\top\bm{A}^\top\bm{X}\bigr)\bm{\Phi}^\top,\\
\bm{X}^\top\bm{A}\,\bm{X}\bm{X}^\top\bm{A}^\top\bm{X}=\tfrac1{d_k}\bm{C}\WQ^\top\WK\bm{C}^2\WK^\top\WQ\bm{C}.
\end{gathered}
\]
By \cref{lem:atta-quad}, $\E[\WK\bm{C}^2\WK^\top\,|\,\bm{X}]=c_2\bm{I}_{d_k}$ and then $\E[\WQ^\top\WQ]=d_k\bm{I}_{d_0}$, so
\begin{equation}\label{eq:atta-T4}
\E[\mathrm{T}_4\,|\,\bm{X}]=d_v\,c_2\;\bm{\Phi}\bm{C}^2\bm{\Phi}^\top .
\end{equation}

\emph{The term $\mathrm{T}_1$.} By \cref{lem:atta-quad} applied to $\WO$,
\[
\E[\mathrm{T}_1\,|\,\bm{X},\WQ,\WK,\WV]=\tr\bigl(\bm{Z}^\top\bm{A}^\top\bm{\Omega}\bm{A}\bm{Z}\bigr)\bm{I}_{d_2}=\tr\bigl((\bm{A}^\top\bm{A}\bm{Z}\bm{Z}^\top)^2\bigr)\bm{I}_{d_2}.
\]
Let $\bm{N}\eqdef\WQ\bm{C}\WQ^\top$, $\bm{S}\eqdef\WK^\top\bm{N}\WK$, $\bm{H}_V\eqdef\WV^\top\WV\sim\mathcal{W}_{d_0}(d_v,\bm{I}_{d_0})$ and $\bm{M}\eqdef\bm{C}\bm{S}\bm{C}$, which is symmetric. Then $\bm{A}^\top\bm{A}=\tfrac1{d_k}\bm{X}\bm{S}\bm{X}^\top$ and $\bm{Z}\bm{Z}^\top=\bm{X}\bm{H}_V\bm{X}^\top$, so
\[
\tr\bigl((\bm{A}^\top\bm{A}\bm{Z}\bm{Z}^\top)^2\bigr)=\tfrac1{d_k^2}\tr(\bm{M}\bm{H}_V\bm{M}\bm{H}_V).
\]
Since $\bm{H}_V$ is independent of $\bm{M}$, \cref{lem:atta-wishart} with $(n,d)=(d_v,d_0)$ gives
\[
\E\bigl[\tr(\bm{M}\bm{H}_V\bm{M}\bm{H}_V)\,\big|\,\bm{X},\WQ,\WK\bigr]=d_v(d_v{+}1)\tr(\bm{M}^2)+d_v(\tr\bm{M})^2 .
\]
Now $\tr\bm{M}=\tr(\bm{S}\bm{C}^2)$ and $\tr\bm{M}^2=\tr(\bm{S}\bm{C}^2\bm{S}\bm{C}^2)$, so \cref{lem:atta-S} (conditionally on $\bm{N}$) followed by \cref{lem:atta-N} give
\begin{equation}\label{eq:atta-trM}
\begin{aligned}
\E[\tr\bm{M}^2\,|\,\bm{X}]&=d_k^2\bigl(c_1^2c_4{+}c_2c_4{+}c_2^3\bigr)+d_k\bigl(c_1^2c_4{+}c_1^2c_2^2{+}3c_2c_4{+}c_2^3\bigr),\\
\E[(\tr\bm{M})^2\,|\,\bm{X}]&=d_k^2\bigl(c_1^2c_2^2{+}2c_2c_4\bigr)+2d_k\bigl(c_1^2c_4{+}c_2c_4{+}c_2^3\bigr).
\end{aligned}
\end{equation}
Therefore $\E[\mathrm{T}_1\,|\,\bm{X}]=\alpha_O(\bm{C})\bm{I}_{d_2}$ with $\alpha_O(\bm{C})=\tfrac1{d_k^2}\bigl(d_v(d_v{+}1)\E[\tr\bm{M}^2\,|\,\bm{X}]+d_v\E[(\tr\bm{M})^2\,|\,\bm{X}]\bigr)$, which is \eqref{eq:atta-alphaO}; adding \eqref{eq:atta-T4} gives \eqref{eq:atta-wo-cond}.
\end{proof}

\begin{lemma}[Conditional constants]\label{lem:atta-condconsts}
Under \cref{ass:atta}, conditionally on $\bm{X}$, with $t_\Phi\eqdef\tr(\bm{\Phi}\bm{C}^2\bm{\Phi}^\top)$, the constants of \cref{lem:atta-iso} are
\begin{align}
c_Q(\bm{C})&=\frac{c_2}{d_k}\bigl(\beta_1+d_v\,t_\Phi\bigr)+\frac{c_2}{d_k^2}\bigl(\beta_1+\beta_2\bigr),\label{eq:atta-cQ}\\
c_K(\bm{C})&=\frac{c_2}{d_k}\bigl(\beta_1+d_v\,t_\Phi\bigr)+\frac{c_1^2+c_2}{d_k^2}\,\beta_1,\label{eq:atta-cK}\\
c_V(\bm{C})&=d_2(d_2{+}d_v{+}1)\Bigl[\bigl(c_1^2c_4{+}c_2c_4{+}c_2^3\bigr)+\tfrac1{d_k}\bigl(c_1^2c_4{+}c_1^2c_2^2{+}3c_2c_4{+}c_2^3\bigr)\Bigr]+c_2\,t_\Phi.\label{eq:atta-cV}
\end{align}
\end{lemma}

\begin{proof}
\sloppy
Substituting $\bm{R}=\bm{A}\bm{Z}\WO^\top-\bm{Y}$ into \eqref{eq:atta-grads} splits each gradient as $\bm{G}_\bullet=\bm{a}_\bullet-\bm{b}_\bullet$, where $\bm{a}_\bullet$ collects the term containing $\bm{A}\bm{Z}\WO^\top$ and $\bm{b}_\bullet$ the term containing $\bm{Y}$. For $\bullet\in\{Q,K\}$, $\bm{a}_\bullet$ is quadratic and $\bm{b}_\bullet$ linear in $\WO$; for $\bullet=V$, $\bm{a}_V$ is quadratic and $\bm{b}_V$ linear in $\WO$ as well. Since $\WO$ is centered and independent of everything else, $\E[\tr(\bm{a}_\bullet\bm{b}_\bullet^\top)\,|\,\bm{X}]=0$ and
\[
\E[\norm{\bm{G}_\bullet}_\Frob^2\,|\,\bm{X}]=\E[\norm{\bm{a}_\bullet}_\Frob^2\,|\,\bm{X}]+\E[\norm{\bm{b}_\bullet}_\Frob^2\,|\,\bm{X}].
\]
Throughout, $\bm{H}_O\eqdef\WO^\top\WO$, $\bm{E}\eqdef\bm{C}\WV^\top\bm{H}_O\WV\bm{C}$ (as in \cref{lem:atta-E}), $\bm{N}\eqdef\WQ\bm{C}\WQ^\top$ (as in \cref{lem:atta-N}), and we use the identities $\bm{X}^\top\bm{Z}=\bm{C}\WV^\top$ and $\bm{X}^\top\bm{A}^\top\bm{X}=\tfrac1{\sqrt{d_k}}\bm{C}\WK^\top\WQ\bm{C}$.

\emph{Target parts.} From \eqref{eq:atta-grads}, $\bm{b}_Q=\tfrac1{\sqrt{d_k}}\WK\bm{C}\WV^\top\WO^\top\bm{\Phi}\bm{C}$. By \cref{lem:atta-quad} applied to $\WK$ and cyclicity of the trace,
\[
\E[\norm{\bm{b}_Q}_\Frob^2\,|\,\bm{X},\WV,\WO]=\tr\bigl(\WV\bm{C}^2\WV^\top\cdot\WO^\top\bm{\Phi}\bm{C}^2\bm{\Phi}^\top\WO\bigr),
\]
and \cref{lem:atta-quad} applied to $\WV$ ($\E[\WV\bm{C}^2\WV^\top]=c_2\bm{I}_{d_v}$) and to $\WO$ ($\E[\WO^\top\bm{\Phi}\bm{C}^2\bm{\Phi}^\top\WO]=t_\Phi\bm{I}_{d_v}$) gives $\E[\norm{\bm{b}_Q}_\Frob^2\,|\,\bm{X}]=d_vc_2t_\Phi$. The same computation applies to $\bm{b}_K=\tfrac1{\sqrt{d_k}}\WQ\bm{C}\bm{\Phi}^\top\WO\WV\bm{C}$ (\cref{lem:atta-quad} applied to $\WQ$) and to $\bm{b}_V=\WO^\top\bm{Y}^\top\bm{A}\bm{X}$ (\cref{lem:atta-quad} applied to $\WO$, then \eqref{eq:atta-T4} for the average of $\bm{X}^\top\bm{A}\bm{X}\bm{X}^\top\bm{A}^\top\bm{X}$ over $\WK,\WQ$), so
\[
\E[\norm{\bm{b}_Q}_\Frob^2\,|\,\bm{X}]=\E[\norm{\bm{b}_K}_\Frob^2\,|\,\bm{X}]=\E[\norm{\bm{b}_V}_\Frob^2\,|\,\bm{X}]=d_v\,c_2\,t_\Phi .
\]

\emph{Model part of $\GQ$.} Using $\bm{X}^\top\bm{Z}=\bm{C}\WV^\top$ and $\bm{Z}^\top\bm{A}^\top\bm{X}=\WV\bm{X}^\top\bm{A}^\top\bm{X}=\tfrac1{\sqrt{d_k}}\WV\bm{C}\WK^\top\WQ\bm{C}$,
\[
\bm{a}_Q=\tfrac1{d_k}\WK\bm{E}\WK^\top\WQ\bm{C}=\tfrac1{d_k}\bm{J}\WQ\bm{C},
\qquad
\norm{\bm{a}_Q}_\Frob^2=\tfrac1{d_k^2}\tr\bigl(\bm{J}^2\widetilde{\bm{N}}\bigr),
\]
where $\bm{J}\eqdef\WK\bm{E}\WK^\top$ and $\widetilde{\bm{N}}\eqdef\WQ\bm{C}^2\WQ^\top$. Given $\bm{X}$, $\bm{J}$ and $\widetilde{\bm{N}}$ are independent, and $\E[\widetilde{\bm{N}}\,|\,\bm{X}]=c_2\bm{I}_{d_k}$ by \cref{lem:atta-quad}, so $\E[\norm{\bm{a}_Q}_\Frob^2\,|\,\bm{X}]=\tfrac{c_2}{d_k^2}\E[\tr\bm{J}^2\,|\,\bm{X}]$. By \cref{lem:atta-N} with $\bm{E}$ in place of $\bm{C}$ and $\WK$ in place of $\WQ$, $\E[\tr\bm{J}^2\,|\,\bm{X},\bm{E}]=d_k^2\tr(\bm{E}^2)+d_k\bigl((\tr\bm{E})^2+\tr(\bm{E}^2)\bigr)$, and by \cref{lem:atta-E}, $\E[\tr\bm{J}^2\,|\,\bm{X}]=d_k^2\beta_1+d_k(\beta_1{+}\beta_2)$. Hence $\E[\norm{\bm{a}_Q}_\Frob^2\,|\,\bm{X}]=c_2\beta_1+\tfrac{c_2}{d_k}(\beta_1{+}\beta_2)$, and $c_Q(\bm{C})=\tfrac1{d_k}\bigl(\E[\norm{\bm{a}_Q}_\Frob^2\,|\,\bm{X}]+\E[\norm{\bm{b}_Q}_\Frob^2\,|\,\bm{X}]\bigr)$ is \eqref{eq:atta-cQ}.

\emph{Model part of $\GK$.} Using $\bm{X}^\top\bm{A}\bm{Z}=\tfrac1{\sqrt{d_k}}\bm{C}\WQ^\top\WK\bm{C}\WV^\top$ and $\bm{Z}^\top\bm{X}=\WV\bm{C}$,
\[
\bm{a}_K=\tfrac1{d_k}\bm{N}\WK\bm{E},
\qquad
\norm{\bm{a}_K}_\Frob^2=\tfrac1{d_k^2}\tr\bigl(\bm{N}^2\,\WK\bm{E}^2\WK^\top\bigr).
\]
Given $\bm{X}$, the matrices $\bm{N}$, $\WK$ and $\bm{E}$ are independent. By \cref{lem:atta-quad}, $\E[\WK\bm{E}^2\WK^\top\,|\,\bm{X},\bm{E}]=\tr(\bm{E}^2)\bm{I}_{d_k}$; by \cref{lem:atta-N}, $\E[\tr\bm{N}^2\,|\,\bm{X}]=d_k^2c_2+d_k(c_1^2{+}c_2)$; and by \cref{lem:atta-E}, $\E[\tr\bm{E}^2\,|\,\bm{X}]=\beta_1$. Hence $\E[\norm{\bm{a}_K}_\Frob^2\,|\,\bm{X}]=\bigl(c_2+\tfrac{c_1^2+c_2}{d_k}\bigr)\beta_1$, which together with the target part gives \eqref{eq:atta-cK}.

\emph{Model part of $\GV$.} Here $\bm{a}_V=\WO^\top\WO\bm{Z}^\top\bm{A}^\top\bm{A}\bm{X}=\bm{H}_O\bm{Z}^\top\bm{A}^\top\bm{A}\bm{X}$, so
\[
\norm{\bm{a}_V}_\Frob^2=\tr(\bm{H}_O^2\bm{M}_V),
\qquad
\bm{M}_V\eqdef\bm{Z}^\top\bm{A}^\top\bm{A}\bm{X}\bm{X}^\top\bm{A}^\top\bm{A}\bm{Z}=\tfrac1{d_k^2}\WV\bm{C}\bm{S}\bm{C}^2\bm{S}\bm{C}\WV^\top,
\]
with $\bm{S}$ as in the proof of \cref{lem:atta-condwo}. Given $\bm{X}$, $\bm{H}_O$ is independent of $\bm{M}_V$ and $\E[\bm{H}_O^2]=d_2(d_2{+}d_v{+}1)\bm{I}_{d_v}$ by \cref{lem:atta-wishart}, so $\E[\norm{\bm{a}_V}_\Frob^2\,|\,\bm{X}]=d_2(d_2{+}d_v{+}1)\E[\tr\bm{M}_V\,|\,\bm{X}]$. Taking the expectation over $\WV$ ($\E[\WV^\top\WV]=d_v\bm{I}_{d_0}$) and then over $\WK,\WQ$ by \cref{lem:atta-S,lem:atta-N},
\[
\E[\tr\bm{M}_V\,|\,\bm{X}]=\tfrac{d_v}{d_k^2}\,\E[\tr(\bm{S}\bm{C}^2\bm{S}\bm{C}^2)\,|\,\bm{X}]=d_v\Bigl[\bigl(c_1^2c_4{+}c_2c_4{+}c_2^3\bigr)+\tfrac1{d_k}\bigl(c_1^2c_4{+}c_1^2c_2^2{+}3c_2c_4{+}c_2^3\bigr)\Bigr],
\]
which is the first line of \eqref{eq:atta-trM} divided by $d_k^2$. Adding the target part and dividing by $d_v$ gives \eqref{eq:atta-cV}.
\end{proof}

\begin{proof}[Proof of \cref{thm:att}]
Take expectations over $\bm{X}$ in \cref{lem:atta-iso,lem:atta-condwo,lem:atta-condconsts}; by the tower property and \cref{ass:atta},
\[
\E[\GO\GO^\top]=\E[\alpha_O(\bm{C})]\,\bm{I}_{d_2}+d_v\,\bm{\Phi}\,\E[c_2\bm{C}^2]\,\bm{\Phi}^\top,
\qquad
\E[\bm{G}_\bullet\bm{G}_\bullet^\top]=\E[c_\bullet(\bm{C})]\,\bm{I}\quad(\bullet\in\{Q,K,V\}).
\]
For $\bm{Q}\in O(d_0)$, $\bm{Q}\bm{C}\bm{Q}^\top\stackrel{d}{=}\bm{C}$ and $c_2$ is invariant, so $\E[c_2\bm{C}^2]=\bm{Q}\,\E[c_2\bm{C}^2]\,\bm{Q}^\top$ for all $\bm{Q}$, and by Schur's lemma $\E[c_2\bm{C}^2]=\gamma\,\bm{I}_{d_0}$ with $\gamma=\tfrac1{d_0}\E[c_2\tr\bm{C}^2]=\tfrac1{d_0}\E[c_2^2]$; \eqref{eq:atta-wm-c22} gives the polynomial and the order in \eqref{eq:atta-gamma}. This proves \eqref{eq:atta-main}, and averaging \eqref{eq:atta-alphaO} termwise with \eqref{eq:atta-Adef} gives the closed form of $\bar\alpha_O=\E[\alpha_O(\bm{C})]$ in \eqref{eq:atta-abar}. The same isotropy argument gives $\E[c_2t_\Phi]=\tr(\bm{\Phi}\,\E[c_2\bm{C}^2]\,\bm{\Phi}^\top)=\gamma\norm{\bm{\Phi}}_\Frob^2$, so averaging \eqref{eq:atta-cQ}--\eqref{eq:atta-cV} gives \eqref{eq:atta-cbars}, where by \eqref{eq:atta-betas}
\begin{align}
\E[c_2\beta_1]&=d_2d_v\Bigl[(d_2d_v{+}d_2{+}d_v{+}3)\,\E[c_2c_4]+(d_2{+}d_v{+}1)\,\E[c_2^3]\Bigr],\notag\\
\E[c_2\beta_2]&=d_2d_v\Bigl[(d_2d_v{+}2)\,\E[c_2^3]+2(d_2{+}d_v{+}1)\,\E[c_2c_4]\Bigr],\label{eq:atta-Ebeta}\\
\E[(c_1^2{+}c_2)\beta_1]&=d_2d_v\Bigl[(d_2d_v{+}d_2{+}d_v{+}3)\bigl(\E[c_1^2c_4]{+}\E[c_2c_4]\bigr)+(d_2{+}d_v{+}1)\bigl(\E[c_1^2c_2^2]{+}\E[c_2^3]\bigr)\Bigr],\notag
\end{align}
all explicit by \cref{lem:atta-wm}. Finally, substituting $\mA_3=\mA_1+\mA_2$ into the closed form of $\bar\alpha_O$,
\begin{equation}\label{eq:atta-abar-orders}
\bar\alpha_O=d_v\Bigl[(d_v{+}1)\Bigl(1+\tfrac1{d_k}\Bigr)+\tfrac2{d_k}\Bigr]\mA_1+d_v\Bigl[1+\tfrac{d_v{+}1}{d_k}\Bigr]\mA_2
=\Theta(d_v^2d_0^6),
\end{equation}
all terms being nonnegative and $\mA_1,\mA_2=\Theta(d_0^6)$ by \cref{lem:atta-wm}, which is the order stated in \eqref{eq:atta-abar}.
\end{proof}

\begin{proof}[Proof of \cref{cor:atta-scaling}]
All Wishart moments of \cref{lem:atta-wm} are positive, and $d_v,d_k\ge1$ throughout.

\emph{Step 1: the $\beta$-moments.} By \eqref{eq:atta-Ebeta}, each of
\begin{align*}
\E[c_2\beta_1]&=d_2d_v\Bigl[(d_2d_v{+}d_2{+}d_v{+}3)\,\E[c_2c_4]+(d_2{+}d_v{+}1)\,\E[c_2^3]\Bigr],\\
\E[c_2\beta_2]&=d_2d_v\Bigl[(d_2d_v{+}2)\,\E[c_2^3]+2(d_2{+}d_v{+}1)\,\E[c_2c_4]\Bigr],\\
\E[(c_1^2{+}c_2)\beta_1]&=d_2d_v\Bigl[(d_2d_v{+}d_2{+}d_v{+}3)\bigl(\E[c_1^2c_4]{+}\E[c_2c_4]\bigr)+(d_2{+}d_v{+}1)\bigl(\E[c_1^2c_2^2]{+}\E[c_2^3]\bigr)\Bigr]
\end{align*}
is $d_2d_v$ times a polynomial of degree one in $d_v$ whose coefficients are positive combinations of the degree-six moments $\E[c_2c_4],\E[c_2^3],\E[c_1^2c_4],\E[c_1^2c_2^2]=\Theta(d_0^6)$ of \cref{lem:atta-wm}; since $d_2d_v\ge\tfrac12(d_2{+}d_v)$, each is
\begin{equation}\label{eq:atta-beta-orders}
\E[c_2\beta_1]=\Theta(d_2^2d_v^2d_0^6),\qquad \E[c_2\beta_2]=\Theta(d_2^2d_v^2d_0^6),\qquad \E[(c_1^2{+}c_2)\beta_1]=\Theta(d_2^2d_v^2d_0^6).
\end{equation}

\emph{Step 2: $c_Q$ and $c_K$.} Substituting \eqref{eq:atta-beta-orders} into \eqref{eq:atta-cbars},
\begin{align*}
c_Q&=\frac{1}{d_k}\Bigl(\underbrace{\E[c_2\beta_1]}_{\Theta(d_2^2d_v^2d_0^6)}+\underbrace{d_v\gamma\norm{\bm{\Phi}}_\Frob^2}_{\Theta(d_vd_0^3)}\Bigr)
+\frac{1}{d_k^2}\underbrace{\bigl(\E[c_2\beta_1]+\E[c_2\beta_2]\bigr)}_{\Theta(d_2^2d_v^2d_0^6)},\\
c_K&=\frac{1}{d_k}\Bigl(\underbrace{\E[c_2\beta_1]}_{\Theta(d_2^2d_v^2d_0^6)}+\underbrace{d_v\gamma\norm{\bm{\Phi}}_\Frob^2}_{\Theta(d_vd_0^3)}\Bigr)
+\frac{1}{d_k^2}\underbrace{\E[(c_1^2{+}c_2)\beta_1]}_{\Theta(d_2^2d_v^2d_0^6)},
\end{align*}
where $\gamma=\Theta(d_0^3)$ by \cref{lem:atta-wm}. All terms are positive, and since $d_2,d_v,d_k,d_0\ge1$ the second and third terms are at most a constant multiple of the first; hence
\[
c_Q\asymp c_K=\Theta\bigl(d_2^2d_v^2d_0^6/d_k\bigr),
\qquad
c_K-c_Q=\frac{\E[c_1^2\beta_1]-\E[c_2\beta_2]}{d_k^2}=O\Bigl(\frac{d_2^2d_v^2d_0^6}{d_k^2}\Bigr)=O\Bigl(\frac{1}{d_k}\Bigr)\,c_Q .
\]

\emph{Step 3: $c_V$.} By \eqref{eq:atta-cbars} with $\mA_3=\mA_1+\mA_2$,
\[
c_V=d_2(d_2{+}d_v{+}1)\Bigl(\mA_1+\frac{\mA_3}{d_k}\Bigr)+\gamma\norm{\bm{\Phi}}_\Frob^2,
\qquad
d_2\mA_1\le d_2\Bigl(\mA_1+\frac{\mA_3}{d_k}\Bigr)\le d_2(\mA_1+\mA_3),
\]
so $c_V$ is affine in $d_v$ with slope between $d_2\mA_1$ and $d_2(\mA_1{+}\mA_3)$ and intercept between $d_2(d_2{+}1)\mA_1$ and $d_2(d_2{+}1)(\mA_1{+}\mA_3)+\gamma\norm{\bm{\Phi}}_\Frob^2$; since $\mA_1,\mA_2=\Theta(d_0^6)$ and $\gamma=\Theta(d_0^3)$ (\cref{lem:atta-wm}), $c_V=\Theta\bigl(d_2(d_2{+}d_v)\,d_0^6\bigr)$.

\emph{Step 4: $\bar\alpha_O$.} Let $\bar\alpha_O^{\infty}\eqdef d_v(d_v{+}1)\mA_1+d_v\mA_2$ be the $d_k$-free part of \eqref{eq:atta-abar-orders}. Then
\begin{align*}
\bar\alpha_O-\bar\alpha_O^{\infty}
&=\frac{1}{d_k}\Bigl[d_v(d_v{+}1)\mA_3+2d_v\mA_1\Bigr]
\le\frac{1}{d_k}\Bigl[d_v(d_v{+}1)(\mA_1{+}\mA_2)+2d_v(d_v{+}1)\mA_1\Bigr]\\
&=\frac{3\mA_1+\mA_2}{\mA_1}\cdot\frac{d_v(d_v{+}1)\mA_1}{d_k}
\le\frac{3\mA_1+\mA_2}{\mA_1}\cdot\frac{\bar\alpha_O^{\infty}}{d_k},
\end{align*}
so that
\begin{equation}\label{eq:atta-abar-sandwich}
\bar\alpha_O^{\infty}\le\bar\alpha_O\le\Bigl(1+\frac{3\mA_1+\mA_2}{\mA_1d_k}\Bigr)\bar\alpha_O^{\infty},
\qquad
d_v^2\mA_1\le\bar\alpha_O^{\infty}\le 2d_v^2\mA_1+d_v^2\mA_2 ,
\end{equation}
and therefore, with $\mA_1,\mA_2=\Theta(d_0^6)$, $\bar\alpha_O=\Theta(d_v^2d_0^6)$.

\emph{Step 5: the $\WO$ ratio.} By \eqref{eq:atta-main}, for $a\neq b$ with $(\bm{\Phi}\bm{\Phi}^\top)_{ab}\neq0$,
\begin{gather*}
\E[(\GO\GO^\top)_{aa}]=\bar\alpha_O+d_v\gamma(\bm{\Phi}\bm{\Phi}^\top)_{aa}=\Theta(d_v^2d_0^6),\\
\E[(\GO\GO^\top)_{ab}]=d_v\gamma(\bm{\Phi}\bm{\Phi}^\top)_{ab}=\Theta(d_vd_0^3),
\end{gather*}
and dividing gives the ratio $\Theta(d_vd_0^3)$. Finally, Steps 3 and 4 show that $d_k$ enters $c_V$ and $\bar\alpha_O$ only through the factors $1+\mA_3/(\mA_1d_k)\in[1,\,1+\mA_3/\mA_1]$ and the bracket in \eqref{eq:atta-abar-sandwich}, both bounded above and below; the same is then true of the ratio.
\end{proof}

\subsection{Discussion}\label{app:att:discussion}

\begin{remark}[Scope of the isotropy]\label{rem:atta-scope}
This remark records how far the exact zeros of \cref{thm:att} for $\WQ,\WK,\WV$ extend beyond the setting of \cref{ass:atta}. The proof of \cref{lem:atta-iso} uses only that $f$ depends on $(\WQ,\WK)$ through $\WQ^\top\WK$ and on $(\WV,\WO)$ through $\WO\WV$, and that the law of the weights is invariant under the sign and permutation maps. Hence the isotropy of $\GQ,\GK,\GV$ holds conditionally on every dataset, for any differentiable loss, for any initialization whose law has these symmetries, and, applying the sign matrix to stacked heads, for multi-head attention. RoPE inserts a position-dependent matrix between queries and keys and is not covered. For these three matrices $\bm{P}\propto\bm{I}$ in expectation, so orthogonalization degenerates to the scalar normalization $\bm{G}/\norm{\bm{G}}_2$.
\end{remark}

\begin{remark}[Role of data randomness]\label{rem:atta-data}
This remark identifies which part of the argument actually uses the Gaussian tokens, since the answer differs from the ReLU setting. Unlike the ReLU setting, where Gaussian data creates the cancellation ($\kappa_1,\kappa_2$ are input integrals), here the data enters no cancellation: the zeros are weight-space symmetries holding per dataset, and \eqref{eq:atta-wo-cond} is an identity for every dataset. For instance, since $\alpha_O(\bm{C})\ge d_v^2c_2^3$ and $c_2\ge\lambda_{\max}(\bm{C})^2$, \eqref{eq:atta-wo-cond} gives $\E[(\GO\GO^\top)_{aa}\,|\,\bm{X}]/\E[(\GO\GO^\top)_{ab}\,|\,\bm{X}]\ge d_v\lambda_{\max}(\bm{C})^2/(\norm{\bm{\varphi}_a}\norm{\bm{\varphi}_b})$ deterministically, with $\bm{\varphi}_a$ the rows of $\bm{\Phi}$. Gaussian tokens only translate $c_1,c_2,c_4$ into $(T,d_0)$ via \cref{lem:atta-wm}; what is genuinely needed is independence of the data from the initialization, since a data-dependent initialization breaks the symmetry argument.
\end{remark}

\begin{remark}[Sign of $c_K-c_Q$]\label{rem:atta-sign}
This remark settles which of the query and key constants is larger, a question that \cref{cor:atta-scaling} leaves open because it only bounds their difference. By \eqref{eq:atta-cbars} and \eqref{eq:atta-Ebeta}, $c_K-c_Q=d_2d_v\bigl[(d_2d_v{+}2)\Delta_1+(d_2{+}d_v{+}1)\Delta_2\bigr]/d_k^2$ with $\Delta_1\eqdef\E[c_1^2c_4]-\E[c_2^3]$ and $\Delta_2\eqdef\E[c_1^2c_4]+\E[c_1^2c_2^2]-2\E[c_2c_4]$. When $\min(T,d_0)=1$, $\bm{C}$ has rank one, all four moments equal $\E[(\chi^2_{\max(T,d_0)})^6]$, and $c_K=c_Q$. For every integer $2\le T,d_0\le512$ the polynomials of \cref{lem:atta-wm} give $\Delta_1,\Delta_2>0$, hence $c_K>c_Q$ for all $d_2,d_v$ in that range.
\end{remark}

\begin{remark}[Fan-in scaling]\label{rem:atta-scale}
This remark checks how the conclusions change under the standard fan-in initialization in place of the unit-variance entries of \cref{ass:atta}. With entry variances $1/d_0$ ($\WQ,\WK,\WV$) and $1/d_v$ ($\WO$), each application of \cref{lem:atta-quad} picks up the corresponding variance and the $d_v$ factor of the $\WO$ ratio cancels: the ratio becomes $\Theta(1)$ in $d_k$, $d_v$ and $d_0$ alike, its size being set by the sequence length through the data moments alone, so under fan-in scaling the sequence length, not the hidden dimension, drives diagonalization, consistent with the near-diagonal $\bm{G}\bm{G}^\top$ observed along GPT-2 and LLaMA trajectories by \citet{deng2026rmnp}.
\end{remark}
\end{addblock}

\section{Asymptotic Approximation between Orthogonalization and Row Normalization}\label{sec:equiv}

What does diagonal dominance say about the orthogonalization step itself? Suppose $\bm{V}_t\bm{V}_t^\top$ is diagonal. Then $\bm{V}_t\bm{V}_t^\top=\diag(\bm{V}_t\bm{V}_t^\top)$, so $(\bm{V}_t\bm{V}_t^\top)^{-1/2}=\diag(\bm{V}_t\bm{V}_t^\top)^{-1/2}$, and since $\diag(\bm{V}_t\bm{V}_t^\top)_{ii}=\norm{(\bm{V}_t)_{i,:}}_2^2$,
\begin{equation}\label{eq:rownorm}
\bigl[(\bm{V}_t\bm{V}_t^\top)^{-1/2}\bm{V}_t\bigr]_{i,:}=\frac{(\bm{V}_t)_{i,:}}{\norm{(\bm{V}_t)_{i,:}}_2}.
\end{equation}
That is, when $\bm{V}_t\bm{V}_t^\top$ is diagonal the orthogonalization in \eqref{eq:muon} is just each row of the momentum divided by its own $\ell_2$ norm, which is exactly the row normalization used by several recent stateless optimizers~\citep{ma2025swan,scetbon2025gradient,glentis2025minimalist,wen2025sron,pethick2025training,deng2026rmnp}; equivalently, its implicit preconditioner is $\diag(\bm{P}_t)^{1/2}\kron\bm{I}_n$, which equals $H_{\mathrm{Muon}}$ when $\bm{P}_t$ is diagonal. Putting \eqref{eq:rownorm} together with \addtxt{\cref{thm:mf,thm:dl,thm:relu,thm:att}}: at a Gaussian initialization, as the hidden dimension grows, Muon's $\Theta(mn\min(m,n))$ orthogonalization collapses to a $\Theta(mn)$ row-wise $\ell_2$ normalization with the same preconditioner structure\addtxt{, and for the query, key and value weights of attention the collapse is exact in expectation}. \citet{deng2026rmnp} check that this equivalence holds throughout the training of GPT-2 and LLaMA, and that the dominance ratio grows with model size, in line with the hidden-dimension dependence in \cref{thm:relu} \addtxt{and \cref{thm:att}}. Carrying the analysis past initialization, and to deep nonlinear networks, is left for future work.

\section{Simulation of Theoretical Results}
\label{app:empirical-verification}

We numerically verify the dimension-scaling predictions of
\cref{thm:mf,thm:dl,thm:relu,thm:att} under the Gaussian
initialization of \cref{sec:assump}. For each model, we perform two
experiments. First, at a representative configuration, we estimate
the expected gradient self outer-product and examine its diagonal and
off-diagonal structure. Second, we vary the relevant model dimensions
and measure how the diagonal magnitude, off-diagonal magnitude, and
their ratio change.

To summarize the structure of a matrix
\(\bm{P}\in\R^{m\times m}\), we define the mean diagonal magnitude,
mean off-diagonal magnitude, and diagonal-dominance ratio as
\begin{equation}
D(\bm{P})
=
\frac{1}{m}\sum_{i=1}^{m}|P_{ii}|,
\qquad
\mathrm{OD}(\bm{P})
=
\frac{1}{m(m-1)}
\sum_{i\neq j}|P_{ij}|,
\qquad
R(\bm{P})
=
\frac{D(\bm{P})}{\mathrm{OD}(\bm{P})}.
\label{eq:appG-DOR}
\end{equation}

Given \(S\) independent initializations with gradients
\(\bm{G}^{(1)},\ldots,\bm{G}^{(S)}\), we estimate the expected
gradient self outer-product by
\begin{equation}
\widehat{\E}[\bm{G}\bm{G}^\top]
=
\frac{1}{S}
\sum_{s=1}^{S}
\bm{G}^{(s)}(\bm{G}^{(s)})^\top.
\label{eq:empirical-gram}
\end{equation}

We use \(\widehat D\), \(\widehat{\mathrm{OD}}\), and
\(\widehat R\) to denote the corresponding empirical estimates of \(D(\E[\bm{G}\bm{G}^\top])\),
\(\mathrm{OD}(\E[\bm{G}\bm{G}^\top])\), and
\(R(\E[\bm{G}\bm{G}^\top])\), respectively.

\subsection{Symmetric Matrix Factorization}
\label{app:sim-mf}

\paragraph{Setup.}
We next verify the predictions of \cref{thm:mf,cor:mfa-scaling} for
symmetric matrix factorization. As defined in \eqref{eq:mfa-grad},
the gradient is
\[
\bm{G}
=
(\bm{W}\bm{W}^\top-\Mstar)\bm{W},
\]
where \(\bm{W}\in\R^{d\times k}\) and \(\Mstar\in\R^{d\times d}\) is
symmetric positive semidefinite. Under \cref{ass:mfa},
\cref{thm:mf} gives
\begin{equation}
\E[\bm{G}\bm{G}^\top]
=
\alpha_3 \bm{I}_d
-
2\alpha_2 \Mstar
+
k(\Mstar)^2,
\label{eq:sim-mf-theory}
\end{equation}
where \(\alpha_2=k(k+d+1)\) and
\(\alpha_3=k^3+3(d+1)k^2+(d^2+3d+4)k\).
\cref{cor:mfa-scaling} gives the corresponding diagonal and
off-diagonal entries,
\begin{equation}
\begin{aligned}
\E[(\bm{G}\bm{G}^\top)_{ii}]
&=
\alpha_3
-
2\alpha_2 \Mstar_{ii}
+
k\bigl((\Mstar)^2\bigr)_{ii},
\\
\E[(\bm{G}\bm{G}^\top)_{ij}]
&=
-2\alpha_2 \Mstar_{ij}
+
k\bigl((\Mstar)^2\bigr)_{ij},
\qquad i\neq j.
\end{aligned}
\label{eq:sim-mf-entrywise}
\end{equation}
For fixed \(d\) and a fixed target with nonzero off-diagonal
structure, this gives
\(D(\E[\bm{G}\bm{G}^\top])=\Theta(k^3)\) and
\(\mathrm{OD}(\E[\bm{G}\bm{G}^\top])=\Theta(k^2)\), and hence
\(R(\E[\bm{G}\bm{G}^\top])=\Theta(k)\).

\paragraph{Empirical structure of the expected gradient Gram.}
We first examine the structure of the expected gradient Gram at a
representative hidden dimension. We set \(d=25\), \(k=80\), and construct a fixed
rank-5 target \(\Mstar=\bm{Q}\bm{\Lambda}\bm{Q}^\top\), where
\(\bm{Q}\in\R^{25\times 5}\) has orthonormal columns obtained from
the QR factorization of a Gaussian matrix and
\(\bm{\Lambda}=\operatorname{diag}(3,2.5,2,1.5,1)\).
We draw \(S=10{,}000\) independent initializations
\(\bm{W}^{(1)},\ldots,\bm{W}^{(S)}\), with i.i.d. entries from
\(\mathcal{N}(0,1)\). For each \(s\in\{1,\ldots,S\}\), let
\(\bm{G}^{(s)}\) denote the gradient corresponding to
\(\bm{W}^{(s)}\). Using \eqref{eq:empirical-gram}, we estimate
\(\E[\bm{G}\bm{G}^\top]\) from these \(10{,}000\) initializations.

\begin{figure}[t]
    \centering
    \includegraphics[width=0.3\linewidth]
    {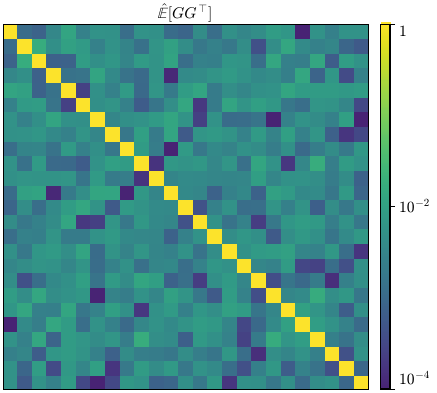}
    \caption{
    Empirical expected gradient self outer-product for symmetric matrix
    factorization with \(d=25\), \(k=80\), and a fixed rank-5 target,
    estimated over \(10{,}000\) independent initializations. Each heatmap is normalized by its empirical mean diagonal magnitude. The panel uses the logarithmic color scale as shown.
    }
    \label{fig:mf-gram-verification}
\end{figure}

Figure~\ref{fig:mf-gram-verification} shows a strongly diagonal
expected gradient Gram, with substantially smaller off-diagonal
entries. This is consistent with the diagonal-dominance behavior
predicted by \cref{thm:mf}.

\paragraph{Hidden-dimension scaling.}
We next vary \(k\) over 15 values from \(5\) to \(640\), while keeping \(d=25\) and \(\Mstar\) fixed, using \(S=10{,}000\) independent initializations at each hidden dimension.

To reduce Monte Carlo noise in the smaller off-diagonal entries, we
subtract a term whose expected off-diagonal contribution is zero. By \cref{lem:mfa-expand}, setting \(\Mstar=\bm{0}\) gives
\(\bm{G}(\bm{0})\bm{G}(\bm{0})^\top=\bm{S}^3\), where
\(\bm{S}=\bm{W}\bm{W}^\top\). Moreover,
\cref{lem:mfa-a3} gives
\(\E[\bm{S}^3]=\alpha_3\bm{I}_d\), so this term has zero expected
off-diagonal entries. We therefore define
\begin{equation}
\bm{\Delta}_{\mathrm{MF}}^{(s)}
=
\bm{G}^{(s)}(\Mstar)\bm{G}^{(s)}(\Mstar)^\top
-
\bm{G}^{(s)}(\bm{0})\bm{G}^{(s)}(\bm{0})^\top.
\label{eq:sim-mf-isolated}
\end{equation}
For \(i\neq j\),
\[
\E\!\left[
(\bm{\Delta}_{\mathrm{MF}}^{(s)})_{ij}
\right]
=
\E\!\left[
(\bm{G}\bm{G}^\top)_{ij}
\right].
\]
Thus, the paired difference removes the
\(\bm{S}^3\) contribution while preserving the expected
off-diagonal entries.

For each hidden dimension, we compute \(\widehat D\) from the full empirical
average \eqref{eq:empirical-gram}. For
\(\widehat{\mathrm{OD}}\), we first average
\(\bm{\Delta}_{\mathrm{MF}}^{(s)}\) over the same \(S\)
initializations and then apply \(\mathrm{OD}(\cdot)\). We then compute
\(\widehat R=\widehat D/\widehat{\mathrm{OD}}\). 

\begin{figure}[t]
    \centering
    \includegraphics[width=\linewidth]
    {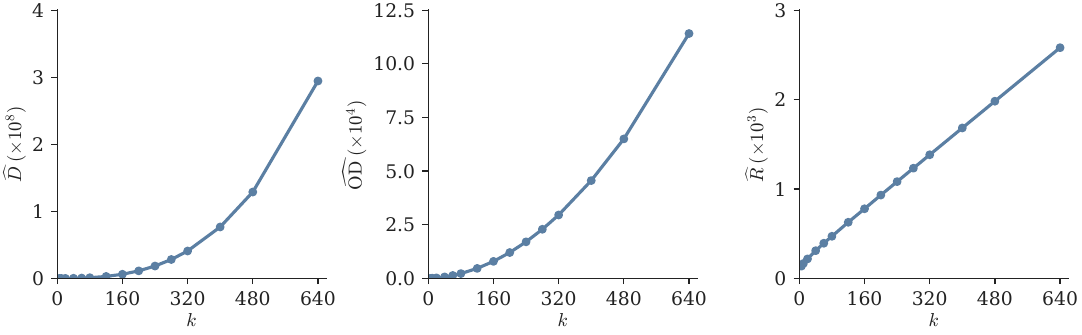}
    \caption{
    Hidden-dimension scaling for symmetric matrix factorization with \(d=25\) and
    a fixed rank-5 target. From left to right, the panels show
    \(\widehat D\), \(\widehat{\mathrm{OD}}\), and
    \(\widehat R=\widehat D/\widehat{\mathrm{OD}}\) as the
    factorization hidden dimension \(k\) is varied. Empirical estimates use
    \(10{,}000\) independent initializations at each hidden dimension;
    \(\widehat{\mathrm{OD}}\) uses the paired estimator in
    \eqref{eq:sim-mf-isolated}.
    }
    \label{fig:mf-width-scaling}
\end{figure}

Figure~\ref{fig:mf-width-scaling} shows that the diagonal magnitude
grows faster with \(k\) than the off-diagonal magnitude, while \(\widehat R\) grows approximately linearly with \(k\). This is
consistent with the theoretical predictions
\(D=\Theta(k^3)\), \(\mathrm{OD}=\Theta(k^2)\), and \(R=\Theta(k)\).

\subsection{Deep Linear Network}
\label{app:sim-dl}

\paragraph{Setup.}
We next verify the predictions of \cref{thm:dl,cor:dla-scaling} for
the deep linear network. Under the Gaussian initialization of
\cref{ass:dl}, \cref{thm:dl} gives
\begin{equation}
\E[\bm{G}_i\bm{G}_i^\top]
=
V_i s_i \bm{I}_{d_i}
+
\bm{T}_i,
\qquad
\bm{T}_i
=
\begin{cases}
\left(\displaystyle\prod_{j=1}^{L-1}d_j\right)
\bm{\Phi}\bm{\Phi}^\top,
& i=L, \\[6pt]
\left(\displaystyle\prod_{\substack{1\leq j\leq L-1\\j\neq i}}d_j\right)
\norm{\bm{\Phi}}_\Frob^2 \bm{I}_{d_i},
& i<L.
\end{cases}
\label{eq:sim-dl-theory}
\end{equation}
\cref{cor:dla-scaling} therefore gives zero expected off-diagonal
entries for every hidden layer \(i<L\), while for the output layer
\begin{equation}
\E[(\bm{G}_L\bm{G}_L^\top)_{ab}]
=
\left(\prod_{j=1}^{L-1}d_j\right)
(\bm{\Phi}\bm{\Phi}^\top)_{ab},
\qquad a\neq b.
\label{eq:sim-dl-output-offdiag}
\end{equation}

\paragraph{Empirical structure of the expected gradient Gram.}
We first examine the empirical expected gradient Grams at a
representative configuration. We use depth \(L=4\) with dimensions
\([d_0,d_1,d_2,d_3,d_4]=[512,80,80,80,512]\)
and a fixed target
\(\bm{\Phi}\in\R^{512\times512}\) with entries sampled from
\(\mathcal{N}(0,1)\). We use \(S=10{,}000\) independent Gaussian
weight initializations and estimate the expected Gram for each
\(\bm{G}_1,\ldots,\bm{G}_4\) using \eqref{eq:empirical-gram}.

\begin{figure}[t]
    \centering
    \includegraphics[width=\linewidth]
    {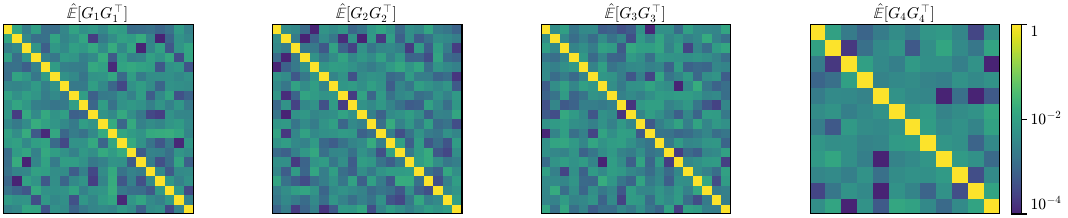}
    \caption{
    Empirical expected gradient self outer-products for the depth-$4$
    linear network with baseline dimensions
    \([512,80,80,80,512]\), estimated from \(10{,}000\) independent
    initializations. The four panels show
    \(\widehat{\E}[\bm{G}_i\bm{G}_i^\top]\) for
    \(i=1,\ldots,4\). Each heatmap is normalized by its empirical mean diagonal magnitude. A common logarithmic color scale is used across all panels.
    }
    \label{fig:dl-gram-verification}
\end{figure}

Figure~\ref{fig:dl-gram-verification} shows strongly diagonal
expected Grams for \(\bm{G}_1,\bm{G}_2,\bm{G}_3\), consistent with
\cref{cor:dla-scaling}, which gives zero expected off-diagonal entries
for the hidden layers. The output-layer Gram retains visible
off-diagonal structure, consistent with the
\(\bm{\Phi}\bm{\Phi}^\top\) term in \cref{thm:dl}.

\paragraph{Dimension scaling.}
We next vary each hidden dimension \(d_1,d_2,d_3\) individually while holding the remaining dimensions fixed at the baseline values above. Each hidden dimension is varied over a 12-point grid \(\{5,10,20,40,60,80,120,160,200,240,280,320\}\), and we evaluate
\(\bm{G}_1,\ldots,\bm{G}_4\). A fixed \(512\times512\) target is
sampled with entries from \(\mathcal{N}(0,1)\).
We use \(S=10{,}000\) independent initializations for every
configuration.

Because the expected off-diagonal entries are much smaller than the
diagonal entries, direct Monte Carlo estimation of
\(\mathrm{OD}\) has high variance. We therefore use a paired
construction to reduce this Monte Carlo variance while preserving the
expected off-diagonal entries. From \eqref{eq:dla-ABR} and \eqref{eq:dla-grad}, for fixed weights, the gradient is affine in \(\bm{\Phi}\). For each
\(s\in\{1,\ldots,S\}\), let
\(\bm{G}_i^{(s)}(\bm{\Phi})\) denote the layer-\(i\) gradient for the
\(s\)-th weight initialization and target \(\bm{\Phi}\). Using the
same weights, we evaluate the gradient at
\(\bm{\Phi}\), \(-\bm{\Phi}\), and \(\bm{0}\), and define
\begin{equation}
\bm{\Delta}_i^{(s)}
=
\frac{1}{2}
\left[
\bm{G}_i^{(s)}(\bm{\Phi})
\bm{G}_i^{(s)}(\bm{\Phi})^\top
+
\bm{G}_i^{(s)}(-\bm{\Phi})
\bm{G}_i^{(s)}(-\bm{\Phi})^\top
\right]
-
\bm{G}_i^{(s)}(\bm{0})
\bm{G}_i^{(s)}(\bm{0})^\top.
\label{eq:sim-dl-isolated}
\end{equation}
The symmetric combination cancels the target-independent and
target-linear terms. The remaining term is
\begin{equation}
\bm{\Delta}_i^{(s)}
=
(\bm{A}_i^{(s)})^\top
\bm{\Phi}
(\bm{B}_i^{(s)})^\top
\bm{B}_i^{(s)}
\bm{\Phi}^\top
\bm{A}_i^{(s)}.
\label{eq:sim-dl-isolated-expanded}
\end{equation}
By \cref{lem:dla-Phi},
\begin{equation}
\E[\bm{\Delta}_i^{(s)}]
=
\bm{T}_i.
\label{eq:sim-dl-isolated-expectation}
\end{equation}
Since the remaining term \(V_i s_i\bm{I}_{d_i}\) in
\cref{thm:dl} is diagonal, \eqref{eq:sim-dl-isolated} has exactly the
same expected off-diagonal entries as the full gradient Gram.

For each configuration, we compute \(\widehat D\) from the full
empirical expected Gram. For \(\widehat{\mathrm{OD}}\), we first
average \(\bm{\Delta}_i^{(s)}\) over the same \(S\) initializations
and then apply \(\mathrm{OD}(\cdot)\). We then compute
\(\widehat R=\widehat D/\widehat{\mathrm{OD}}\).

\begin{figure}[t]
    \centering
    \includegraphics[width=\linewidth]
    {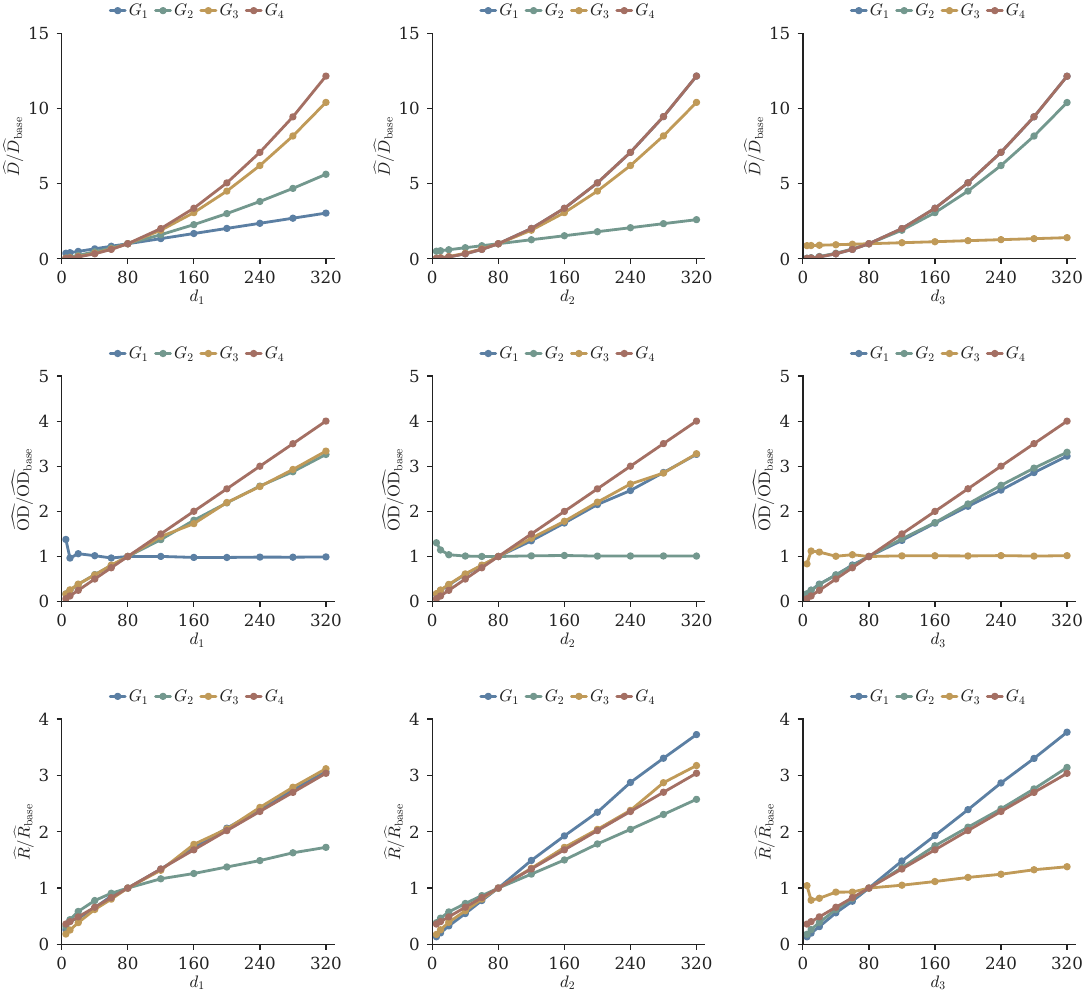}
    \caption{
    Empirical dimension scaling for the depth-$4$ linear network.
    Columns vary \(d_1,d_2,\) and \(d_3\) individually while holding the
    remaining dimensions fixed at their baseline values. Rows show
    \(\widehat D\), \(\widehat{\mathrm{OD}}\), and
    \(\widehat R=\widehat D/\widehat{\mathrm{OD}}\) for
    \(\bm{G}_1,\ldots,\bm{G}_4\), each normalized by its value at the
    baseline configuration
    \([d_0,d_1,d_2,d_3,d_4]=[512,80,80,80,512]\)  to facilitate comparison across layers with different absolute scales. Empirical estimates use \(10{,}000\) independent initializations per configuration. The off-diagonal
    statistic is estimated using \eqref{eq:sim-dl-isolated}.
    }
    \label{fig:dl-dimension-scaling}
\end{figure}

Figure~\ref{fig:dl-dimension-scaling} shows the layer-dependent
scaling predicted by \cref{thm:dl,cor:dla-scaling}. When \(d_j\) is varied, \(\widehat D\) grows more slowly for the corresponding
gradient \(\bm G_j\) than for the other hidden-layer gradients. This is consistent with \(D=\Theta(d_j)\) for \(\bm G_j\),
compared with \(D=\Theta(d_j^2)\) for
\(\bm G_i\), \(i\neq j\), when the remaining dimensions are fixed.

\cref{cor:dla-scaling} gives zero expected off-diagonal entries for \(i<L\), and the remaining nonzero values are finite-sample Monte
Carlo variances. For the output layer, the normalized \(\widehat D\) curves
show approximately quadratic growth, while the normalized \(\widehat{\mathrm{OD}}\) curves grow approximately linearly when any
one of \(d_1,d_2,d_3\) is varied. This is consistent with \(D=\Theta(d_j^2)\), \(\mathrm{OD}=\Theta(d_j)\), and hence
\(R=\Theta(d_j)\).

\subsection{Two-Layer ReLU Network}
\label{app:sim-relu}
\paragraph{Setup.}
We next verify the predictions of \cref{thm:relu} for the two-layer
ReLU network
\(f(\bm{x})=\bm{W}_2\sigma(\bm{W}_1\bm{x})\), with hidden dimension
\(d_1\). Unlike the hidden layers of the deep linear network, the expected gradient self outer-products are not exactly diagonal.

Under the Gaussian initialization of \cref{ass:relu},
\cref{thm:relu} gives, for \(\bm{G}_1\),
\begin{equation}
\E\!\left[
(\bm{G}_1\bm{G}_1^\top)_{aa}
\right]
=
\frac{d_0d_2^2}{4}
+
d_0d_2
\left(
(d_1-1)\kappa_1+\frac{1}{2}
\right)
+
\frac{\norm{\bm{\Phi}}_\Frob^2}{4},
\label{eq:sim-relu-g1-diag}
\end{equation}
while for \(a\neq c\),
\begin{equation}
\E\!\left[
(\bm{G}_1\bm{G}_1^\top)_{ac}
\right]
=
d_0d_2\kappa_2
+
O(d_2\sqrt{d_0}),
\label{eq:sim-relu-g1-offdiag}
\end{equation}
where \(\kappa_1\) and \(\kappa_2\) are defined in
\cref{thm:relu} and depend only on \(d_0\). Holding \(d_0\) and
\(d_2\) fixed while varying \(d_1\) therefore gives
\begin{equation}
D\!\left(\E[\bm{G}_1\bm{G}_1^\top]\right)
=
\Theta(d_1),
\qquad
\mathrm{OD}\!\left(\E[\bm{G}_1\bm{G}_1^\top]\right)
=
\Theta(1),
\qquad
R\!\left(\E[\bm{G}_1\bm{G}_1^\top]\right)
=
\Theta(d_1).
\label{eq:sim-relu-g1-scaling}
\end{equation}

For \(\bm{G}_2\), \cref{thm:relu} gives
\begin{equation}
\E[\bm{G}_2\bm{G}_2^\top]
=
\alpha\bm{I}_{d_2}
+
\frac{d_1}{4}\bm{\Phi}\bm{\Phi}^\top,
\qquad
\alpha=\Theta(d_1^2).
\label{eq:sim-relu-g2}
\end{equation}
Thus, for fixed \(d_0\), \(d_2\), and a fixed target with nonzero
off-diagonal structure,
\begin{equation}
D\!\left(\E[\bm{G}_2\bm{G}_2^\top]\right)
=
\Theta(d_1^2),
\qquad
\mathrm{OD}\!\left(\E[\bm{G}_2\bm{G}_2^\top]\right)
=
\Theta(d_1),
\qquad
R\!\left(\E[\bm{G}_2\bm{G}_2^\top]\right)
=
\Theta(d_1).
\label{eq:sim-relu-g2-scaling}
\end{equation}

\paragraph{Empirical structure of the expected gradient Gram.}
We first examine the empirical expected gradient Grams at a
representative configuration. We use
\([d_0,d_1,d_2]=[512,80,12]\) and a fixed target
\(\bm{\Phi}\in\R^{12\times512}\) whose entries are sampled from
\(\bN(0,9)\). For each of \(S=10{,}000\) independent Gaussian
weight initializations, we estimate the population gradients using
\(128{,}000\) independent inputs sampled from
\(\bN(\bm{0},\bm{I}_{d_0})\), and estimate the expected gradient
Grams using \eqref{eq:empirical-gram}.

\begin{figure}[t]
    \centering
    \includegraphics[width=0.5\linewidth]
    {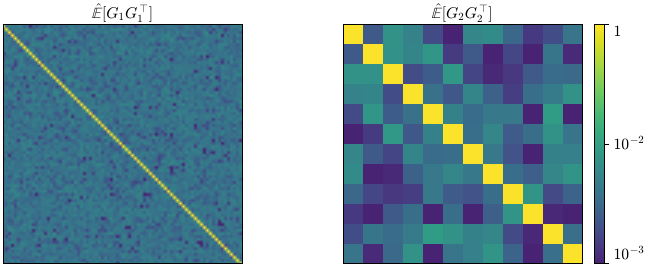}
    \caption{
    Empirical expected gradient self outer-products for the two-layer
    ReLU network with dimensions
    \([d_0,d_1,d_2]=[512,80,12]\), estimated from \(10{,}000\) independent weight initializations. The panels show
    \(\widehat{\E}[\bm{G}_1\bm{G}_1^\top]\) and
    \(\widehat{\E}[\bm{G}_2\bm{G}_2^\top]\). Each heatmap is normalized by its empirical mean diagonal magnitude. A common logarithmic color scale is used across all panels.
    }
    \label{fig:relu-gram-verification}
\end{figure}

Figure~\ref{fig:relu-gram-verification} shows diagonally dominant
empirical expected Grams for both \(\bm{G}_1\) and \(\bm{G}_2\),
with smaller but nonzero off-diagonal entries. This is consistent
with \cref{thm:relu}, which predicts asymptotic diagonal dominance.

\paragraph{Dimension scaling.}
We set \(d_0=512\) and \(d_2=12\), and vary
\(d_1\in\{5,\allowbreak 10,\allowbreak 20,\allowbreak 40,\allowbreak 60,\allowbreak 80,\allowbreak 120,\allowbreak 160,\allowbreak 200,\allowbreak 240,\allowbreak 280,\allowbreak 320,\allowbreak 400,\allowbreak 480,\allowbreak 640\}\). The same fixed target is used
throughout the sweep, and each population gradient is estimated from
\(128{,}000\) Gaussian inputs for each of \(S=10{,}000\) independent
weight initializations.

Because the expected off-diagonal entries are much smaller than the
diagonal entries, direct Monte Carlo estimation of
\(\mathrm{OD}\) is noisy. We therefore use the following expressions to
reduce this Monte Carlo variance while preserving the expected
off-diagonal entries.

For \(\bm{G}_1\), \eqref{eq:rela-G1decomp} gives
\begin{equation}
\bm{G}_1
=
\bm{G}_1^{(1)}
-
\bm{G}_1^{(2)}.
\label{eq:sim-relu-g1-decomp}
\end{equation}
By \eqref{eq:rela-parity}, the cross terms vanish in expectation, and
\cref{lem:rela-G12} gives
\begin{equation}
\bm{G}_1^{(2)}
=
\frac{1}{2}\bm{W}_2^\top\bm{\Phi},
\qquad
\E\!\left[
\bm{G}_1^{(2)}(\bm{G}_1^{(2)})^\top
\right]
=
\frac{\norm{\bm{\Phi}}_\Frob^2}{4}\bm{I}_{d_1}.
\label{eq:sim-relu-g1-target}
\end{equation}
Hence,
\begin{equation}
\E[\bm{G}_1\bm{G}_1^\top]
=
\E\!\left[
\bm{G}_1^{(1)}(\bm{G}_1^{(1)})^\top
\right]
+
\frac{\norm{\bm{\Phi}}_\Frob^2}{4}\bm{I}_{d_1}.
\label{eq:sim-relu-g1-split}
\end{equation}
The second term is diagonal, so all expected off-diagonal entries
come from \(\bm{G}_1^{(1)}\). Moreover, setting
\(\bm{\Phi}=\bm{0}\) in \eqref{eq:rela-G1decomp} gives
\(\bm{G}_1(\bm{0})=\bm{G}_1^{(1)}\). For each
\(s\in\{1,\ldots,S\}\), we therefore define
\begin{equation}
\bm{\Delta}_1^{(s)}
=
\bm{G}_1^{(s)}(\bm{0})
\bm{G}_1^{(s)}(\bm{0})^\top.
\label{eq:sim-relu-g1-isolated}
\end{equation}
For \(a\neq c\),
\begin{equation}
\E\!\left[
(\bm{\Delta}_1^{(s)})_{ac}
\right]
=
\E\!\left[
(\bm{G}_1\bm{G}_1^\top)_{ac}
\right].
\label{eq:sim-relu-g1-offdiag-preserved}
\end{equation}
Thus, \eqref{eq:sim-relu-g1-isolated} preserves the expected
off-diagonal entries while removing the diagonal target contribution.

For \(\bm{G}_2\), \cref{lem:rela-G2form} gives
\(
\bm{G}_2
=
\bm{W}_2\bm{H}
-
\bm{\Phi}\bm{C}^\top.
\)
For each \(s\in\{1,\ldots,S\}\), let
\(\bm{G}_2^{(s)}(\bm{\Phi})\) denote the output-layer gradient for
the \(s\)-th weight initialization and target \(\bm{\Phi}\). Using
the same weights and input samples, we evaluate the gradient at
\(\bm{\Phi}\), \(-\bm{\Phi}\), and \(\bm{0}\), and define
\begin{equation}
\bm{\Delta}_{2}^{(s)}
=
\frac{1}{2}
\left[
\bm{G}_2^{(s)}(\bm{\Phi})
\bm{G}_2^{(s)}(\bm{\Phi})^\top
+
\bm{G}_2^{(s)}(-\bm{\Phi})
\bm{G}_2^{(s)}(-\bm{\Phi})^\top
\right]
-
\bm{G}_2^{(s)}(\bm{0})
\bm{G}_2^{(s)}(\bm{0})^\top.
\label{eq:sim-relu-g2-isolated}
\end{equation}
Substituting \cref{lem:rela-G2form} into
\eqref{eq:sim-relu-g2-isolated} cancels the terms independent of
\(\bm{\Phi}\) and linear in \(\bm{\Phi}\), leaving
\begin{equation}
\bm{\Delta}_{2}^{(s)}
=
\bm{\Phi}
(\bm{C}^{(s)})^\top
\bm{C}^{(s)}
\bm{\Phi}^\top.
\label{eq:sim-relu-g2-isolated-expanded}
\end{equation}
By \cref{lem:rela-Cform},
\(\bm{C}=\tfrac12\bm{W}_1\), and
\cref{lem:rela-quad} gives
\(\E[\bm{W}_1^\top\bm{W}_1]=d_1\bm{I}_{d_0}\).
Hence
\begin{equation}
\E[\bm{\Delta}_{2}^{(s)}]
=
\frac{d_1}{4}\bm{\Phi}\bm{\Phi}^\top.
\label{eq:sim-relu-g2-isolated-expectation}
\end{equation}
Since the remaining term
\(\alpha\bm{I}_{d_2}\) in \eqref{eq:sim-relu-g2} is diagonal,
\eqref{eq:sim-relu-g2-isolated} has exactly the same expected
off-diagonal entries as the full \(\bm{G}_2\bm{G}_2^\top\).

For both layers, we compute \(\widehat D\) from the full empirical
expected Gram. For \(\bm{G}_1\), we compute
\(\widehat{\mathrm{OD}}\) from the zero-target empirical Gram. For
\(\bm{G}_2\), we first average
\(\bm{\Delta}_{2}^{(s)}\) over the same \(S\) initializations and
then apply \(\mathrm{OD}(\cdot)\). We then compute
\(\widehat R=\widehat D/\widehat{\mathrm{OD}}\).

\begin{figure}[t]
    \centering
    \includegraphics[width=\linewidth]
    {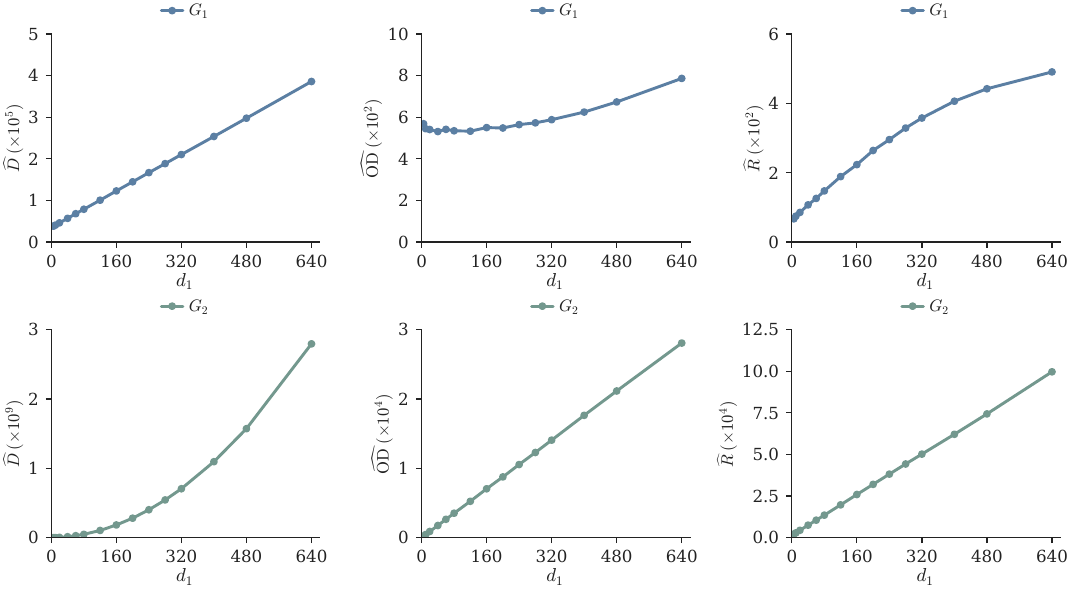}
    \caption{
    Hidden-dimension scaling for the two-layer ReLU network with
    \(d_0=512\) and \(d_2=12\). The top and bottom rows correspond to
    \(\bm{G}_1\) and \(\bm{G}_2\), respectively, and show \(\widehat D\), \(\widehat{\mathrm{OD}}\), and
    \(\widehat R=\widehat D/\widehat{\mathrm{OD}}\) as \(d_1\) is
    varied. For \(\bm{G}_1\),
    \(\widehat{\mathrm{OD}}\) is estimated from the zero-target Gram;
    for \(\bm{G}_2\), it is estimated using
    \eqref{eq:sim-relu-g2-isolated}.
    }
    \label{fig:relu-dimension-scaling}
\end{figure}

Figure~\ref{fig:relu-dimension-scaling} shows the different
hidden-dimension scalings predicted by \cref{thm:relu}. For \(\bm{G}_1\), \(\widehat D\) grows approximately linearly with
\(d_1\), while \(\widehat{\mathrm{OD}}\) remains approximately constant, consistent with \(D=\Theta(d_1)\) and
\(\mathrm{OD}=\Theta(1)\) in
\eqref{eq:sim-relu-g1-scaling}. For \(\bm{G}_2\),
\(\widehat D\) grows approximately quadratically with \(d_1\), while \(\widehat{\mathrm{OD}}\) grows approximately linearly,
consistent with \eqref{eq:sim-relu-g2-scaling}. Accordingly, \(\widehat R\) increases with \(d_1\) for both layers, consistent
with the predicted \(R=\Theta(d_1)\).

\subsection{Single-Layer Attention}
\label{app:sim-attention}
\paragraph{Setup.}
Finally, we verify the predictions of
\cref{thm:att,cor:atta-scaling} for single-layer attention
defined in \eqref{eq:att-layer}. We consider the gradients
\(\GQ,\GK,\GV,\GO\) of the query, key, value, and output
projections, respectively. Under the Gaussian initialization of
\cref{ass:att}, \cref{thm:att} gives
\begin{equation}
\begin{aligned}
\E[\GQ\GQ^\top] &= c_Q \bm{I}_{d_k}, \\
\E[\GK\GK^\top] &= c_K \bm{I}_{d_k}, \\
\E[\GV\GV^\top] &= c_V \bm{I}_{d_v}, \\
\E[\GO\GO^\top]
&=
\bar{\alpha}_O \bm{I}_{d_2}
+
d_v\gamma\,\bm{\Phi}\bm{\Phi}^\top,
\end{aligned}
\label{eq:sim-attention-theory}
\end{equation}
where \(c_Q,c_K,c_V,\bar{\alpha}_O\), and \(\gamma\) are given
explicitly in \cref{thm:att} and \cref{app:att:results}.
Thus, the expected gradient self outer-products for the query, key,
and value projections are exactly diagonal,
\begin{equation}
\mathrm{OD}\!\left(
\E[\bm{G}_p\bm{G}_p^\top]
\right)
=
0,
\qquad
p\in\{Q,K,V\}.
\label{eq:sim-attention-qkv}
\end{equation}
Combining \cref{thm:att,cor:atta-scaling}, for a fixed target with
\(\mathrm{OD}(\bm{\Phi}\bm{\Phi}^\top)>0\), we obtain
\begin{equation}
R\!\left(
\E[\GO\GO^\top]
\right)
=
\Theta(d_vd_0^3).
\label{eq:sim-attention-output-ratio}
\end{equation}
Hence increasing the hidden dimension \(d_v\) or the token dimension
\(d_0\) makes the expected output-projection Gram increasingly
diagonally dominant, while the hidden dimension \(d_k\) enters the
ratio only through a bounded factor.

\paragraph{Empirical structure of the expected gradient Gram.}
We first examine the empirical expected gradient Grams at a
representative configuration. We use linear attention with
\(T=16\), \(d_0=32\), \(d_k=d_v=20\), and \(d_2=12\).
The entries of the fixed target
\(\bm{\Phi}\in\R^{d_2\times d_0}\) are sampled from
\(\bN(0,1)\). We use \(S=10{,}000\) independent initializations.
For each initialization, the rows of
\(\bm{X}\in\R^{T\times d_0}\) are sampled independently from
\(\bN(\bm{0},\bm{I}_{d_0})\), independently of the Gaussian weights,
as specified in \cref{ass:att}. We estimate the expected gradient
Gram for each projection using \eqref{eq:empirical-gram}.

\begin{figure}[t]
    \centering
    \includegraphics[width=\linewidth]
    {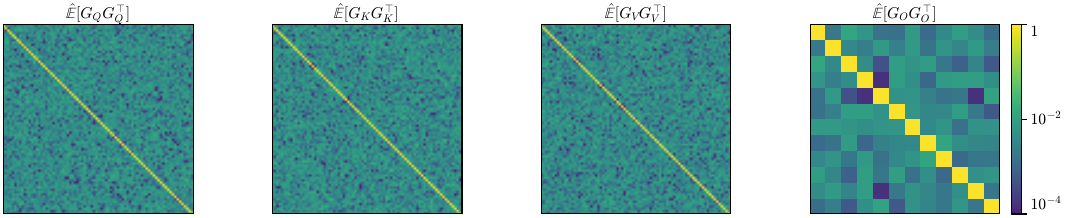}
    \caption{
    Empirical expected gradient self outer-products for single-layer
    attention with
    \(T=16\), \(d_0=32\), \(d_k=d_v=20\), and \(d_2=12\),
    estimated from \(10{,}000\) independent initializations. The four
    panels show
    \(\widehat{\E}[\bm{G}_p\bm{G}_p^\top]\) for
    \(p\in\{Q,K,V,O\}\). Each heatmap is normalized by its empirical mean diagonal magnitude. A common logarithmic color scale is used across all panels.
    }
    \label{fig:attention-gram-verification}
\end{figure}

Figure~\ref{fig:attention-gram-verification} shows strongly diagonal
empirical expected Grams for the query, key, and value projections,
consistent with the exact isotropy in \cref{thm:att}. The
output-projection Gram is also diagonally dominant but retains
nonzero off-diagonal structure, consistent with the
\(d_v\gamma\,\bm{\Phi}\bm{\Phi}^\top\) term in
\eqref{eq:sim-attention-theory}.

\paragraph{Dimension scaling.}
We examine four dimension sweeps: $d_k=d_v$ jointly, $d_k$ with $d_v=80$, $d_v$ with $d_k=80$, and $d_0=d_2$ jointly. For the first three sweeps, we set \(T=32\), \(d_0=512\), and \(d_2=12\), and use the 12-point grid
\(\{5,10,20,40,60,80,120,160,200,240,280,320\}\).
For the joint \(d_0=d_2\) sweep, we set
\(T=16\) and \(d_k=d_v=20\), and use the 13-point grid
\(\{32,64,128,192,256,320,384,448,512,640,768,896,1024\}\).
Further, for the joint \(d_0=d_2\) sweep, we rescale the target $\bm{\Phi}$ at each
dimension to have the same Frobenius norm as the \(32\times32\) reference target. This controls for the explicit target dependence in
\cref{thm:att}: the diagonal constants \(c_Q,c_K,c_V\) contain terms
proportional to \(\norm{\bm{\Phi}}_\Frob^2\) in
\eqref{eq:atta-cbars}, while the output Gram contains the term
\(d_v\gamma\,\bm{\Phi}\bm{\Phi}^\top\) in
\eqref{eq:atta-main}. Holding \(\norm{\bm{\Phi}}_\Frob\) fixed
therefore isolates the effect of increasing \(d_0=d_2\) from the increase in target magnitude from the larger target matrix. For every configuration, we use \(S=10{,}000\) independent
initializations.For each initialization, the rows of \(\bm{X}\) are
sampled independently from
\(\bN(\bm{0},\bm{I}_{d_0})\), independently of the weights.

Since the expected off-diagonal entries are much smaller than the
diagonal entries, direct Monte Carlo estimation of
\(\mathrm{OD}\) is noisy. We therefore use the following expressions to reduce this Monte Carlo variance while preserving the
expected off-diagonal entries. From \eqref{eq:atta-notation} and \eqref{eq:atta-grads}, each gradient
is affine in \(\bm{\Phi}\) when the weights and \(\bm{X}\) are fixed. For each
\(s\in\{1,\ldots,S\}\), let
\(\bm{G}_p^{(s)}(\bm{\Phi})\) denote the gradient for projection
\(p\in\{Q,K,V,O\}\) under the \(s\)-th initialization and target
\(\bm{\Phi}\). Using the same weights and \(\bm{X}\), we evaluate the
gradient at \(\bm{\Phi}\), \(-\bm{\Phi}\), and \(\bm{0}\), and define
\begin{equation}
\bm{\Delta}_p^{(s)}
=
\frac{1}{2}
\left[
\bm{G}_p^{(s)}(\bm{\Phi})
\bm{G}_p^{(s)}(\bm{\Phi})^\top
+
\bm{G}_p^{(s)}(-\bm{\Phi})
\bm{G}_p^{(s)}(-\bm{\Phi})^\top
\right]
-
\bm{G}_p^{(s)}(\bm{0})
\bm{G}_p^{(s)}(\bm{0})^\top.
\label{eq:sim-attention-isolated}
\end{equation}
The symmetric combination cancels the terms independent of
\(\bm{\Phi}\) and linear in \(\bm{\Phi}\), leaving the contribution
quadratic in the target. Combining this cancellation with the target-dependent terms in
\eqref{eq:atta-cbars} and \eqref{eq:atta-main} gives
\begin{equation}
\begin{aligned}
\E[\bm{\Delta}_Q^{(s)}]
&=
\frac{d_v\gamma}{d_k}
\norm{\bm{\Phi}}_\Frob^2\bm{I}_{d_k},\\
\E[\bm{\Delta}_K^{(s)}]
&=
\frac{d_v\gamma}{d_k}
\norm{\bm{\Phi}}_\Frob^2\bm{I}_{d_k},\\
\E[\bm{\Delta}_V^{(s)}]
&=
\gamma
\norm{\bm{\Phi}}_\Frob^2\bm{I}_{d_v},\\
\E[\bm{\Delta}_O^{(s)}]
&=
d_v\gamma\,\bm{\Phi}\bm{\Phi}^\top.
\end{aligned}
\label{eq:sim-attention-isolated-expectations}
\end{equation}
Thus, \(\bm{\Delta}_p^{(s)}\) in
\eqref{eq:sim-attention-isolated} has zero expected off-diagonal
entries for \(p\in\{Q,K,V\}\), exactly as in
\eqref{eq:sim-attention-qkv}, while for \(p=O\) it preserves the
off-diagonal term in \eqref{eq:sim-attention-theory}.

For each configuration, we compute \(\widehat D\) from the full
empirical expected Gram. For \(\widehat{\mathrm{OD}}\), we first
average \(\bm{\Delta}_p^{(s)}\) over the same \(S\) initializations
and then apply \(\mathrm{OD}(\cdot)\). We then compute
\(\widehat R=\widehat D/\widehat{\mathrm{OD}}\).

\begin{figure}[t]
    \centering
    \includegraphics[width=\linewidth]
    {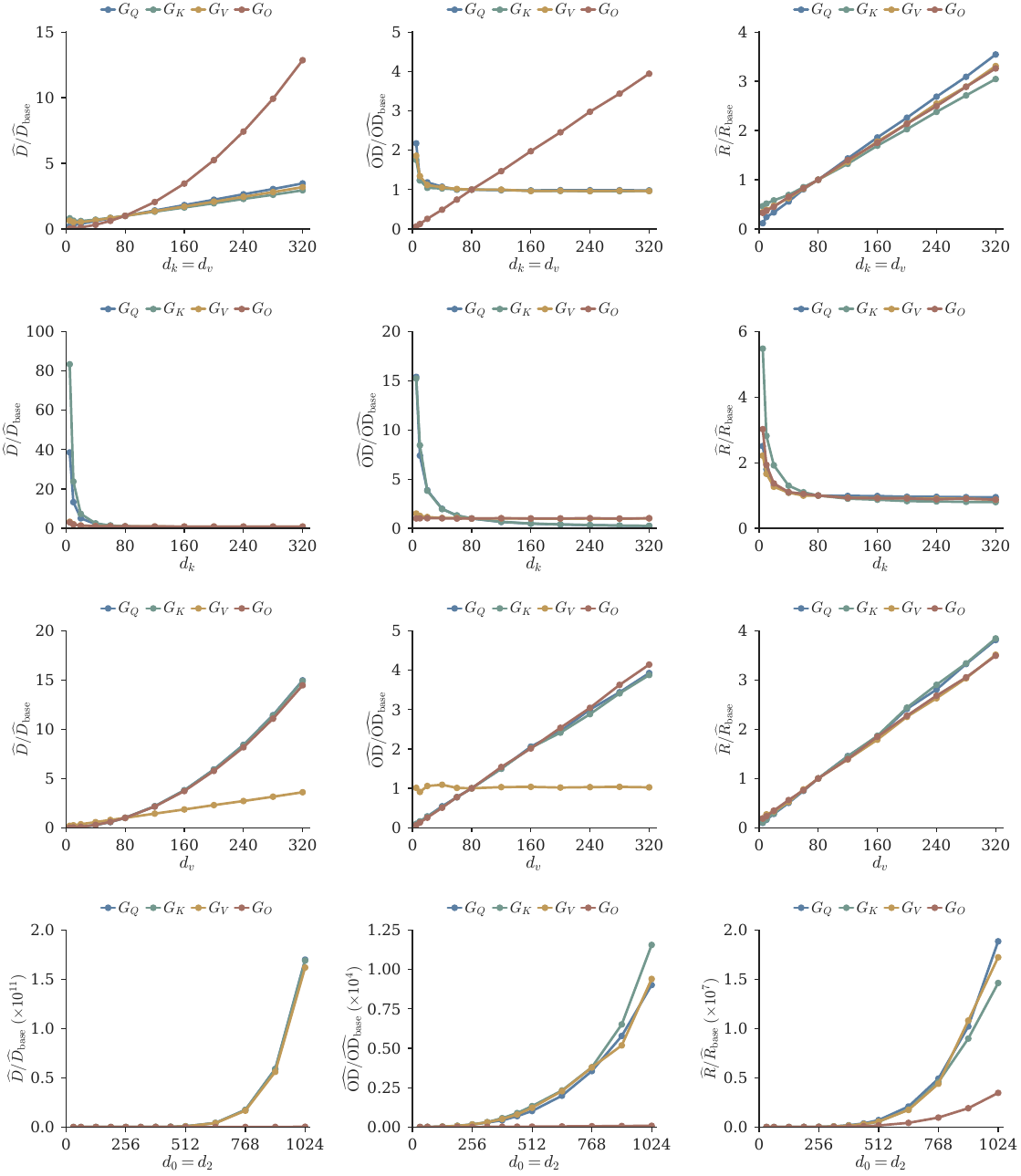}
    \caption{
    Empirical dimension scaling for single-layer attention.
    Rows vary \(d_k=d_v\) jointly, \(d_k\), \(d_v\), and
    \(d_0=d_2\) (top to bottom). Columns show
    \(\widehat D\), \(\widehat{\mathrm{OD}}\), and
    \(\widehat R=\widehat D/\widehat{\mathrm{OD}}\) for
    \(\GQ,\GK,\GV,\) and \(\GO\), each normalized by its value at the
    baseline configuration
    (first three rows: $T=32$, $d_0=512$, $d_k=d_v=80$, $d_2=12$; fourth row: $T=16, d_0=d_2=32, d_k=d_v=20$), to
    facilitate comparison across projections with different absolute
    scales. Empirical estimates use \(10{,}000\) independent
    initializations per configuration; the off-diagonal statistics are
    estimated using \eqref{eq:sim-attention-isolated}.
    }
    \label{fig:attention-dimension-scaling}
\end{figure}

Figure~\ref{fig:attention-dimension-scaling} shows the dimension-dependent diagonal scaling predicted by \cref{thm:att,cor:atta-scaling}. When \(d_k\) is varied with
\(d_v\) fixed, \(\widehat D\) for \(\GQ\) and \(\GK\) decreases, consistent with
\(D=\Theta(d_2^2d_v^2d_0^6/d_k)\), while \(\GV\) and \(\GO\) depend only weakly on \(d_k\). Increasing \(d_v\) increases \(\widehat D\) for all four projections; the stronger growth of
\(\GQ,\GK,\GO\) relative to \(\GV\) is consistent with their quadratic versus linear dependence on \(d_v\). When \(d_k=d_v\)
are increased jointly, \(\widehat D\) grows approximately linearly for \(\GQ,\GK,\GV\) and more rapidly for \(\GO\), again matching
the predicted orders.

For \(Q\), \(K\), and \(V\), the population off-diagonal entries are
exactly zero by \cref{thm:att}; the nonzero
\(\widehat{\mathrm{OD}}\) values are finite-sample Monte Carlo
variances. For the output projection,
\(\widehat{\mathrm{OD}}\) grows approximately linearly with \(d_v\) and is nearly unchanged when only \(d_k\) is varied, while \(\widehat R\) grows approximately linearly with \(d_v\), consistent with \cref{cor:atta-scaling}, which gives \(D=\Theta(d_v^2d_0^6)\),
\(\mathrm{OD}=\Theta(d_vd_0^3)\), and
\(R=\Theta(d_vd_0^3)\). For the joint \(d_0=d_2=m\) sweep, \cref{cor:atta-scaling} gives
\(D=\Theta(m^8)\) for \(\GQ,\GK,\GV\), while for \(\GO\), \(D=\Theta(m^6)\), \(\mathrm{OD}=\Theta(m^3)\), and hence
\(R=\Theta(m^3)\).

\section{Empirical Evaluation on Large Language Models}
\label{sec:llm-extension}

We next test whether the diagonal dominance behavior of \cref{thm:att} persists in transformer language models. We consider Qwen3.5-0.8B-Base~\citep{qwen35}, Granite-4.0-1B-Base~\citep{granite4}, and Llama-3.2-1B~\citep{llama32}. We perform
two experiments. First, we examine the expected gradient Grams of individual attention projections at random initialization. Second, we
measure how their diagonal dominance changes with model width across the three models. 

\subsection{Gradient Gram Structure at Initialization}
\label{app:llm-gram-structure}

We first examine the gradient Gram of attention projections in Qwen3.5-0.8B-Base,
Granite-4.0-1B-Base, and Llama-3.2-1B at random initialization. The models are instantiated from their released Hugging Face configurations~\citep{wolf2020transformers} without pretrained
weights. We set the sequence length to \(T=256\), retain each model's
released model width and head dimension, and set the number of
key/value heads equal to the number of query heads. Native RoPE
behavior is used throughout.

We analyze head \(0\) in block \(3\) for Qwen and in block \(0\)
for Granite and Llama. For Qwen, DeltaNet blocks are excluded, so
block \(3\) is the first full-attention block considered. For each
projection \(p\in\{Q,K,V,O\}\), we estimate
\begin{equation}
\widehat{\bm P}_p
=
\frac{1}{N_WN_X}
\sum_{w=1}^{N_W}
\sum_{x=1}^{N_X}
\bm G_{p,w,x}\bm G_{p,w,x}^{\top},
\label{eq:llm-empirical-gram}
\end{equation}
using \(N_W=25\) independent random initializations and
\(N_X=64\) fixed FineWeb-Edu~\citep{penedo2024the} sequences. We use \(N_W=25\) independent random initializations obtained from each model's native Hugging Face initialization scheme, which uses
zero-mean Gaussian initialization for the attention projection
weights with model-specific scales.
The same input sequences are reused across all initializations.
For Qwen, the gating rows of the query projection are excluded when
forming the query Gram.

Figure~\ref{fig:llm-gram-structure} shows the resulting empirical
expected Grams. All four attention projections exhibit a clear
diagonal component across all three models. This indicates that the
diagonal structure predicted for the single-layer attention model in
\cref{thm:att} also appears in deeper transformer architectures at
random initialization.

\begin{figure}[t]
    \centering
    \includegraphics[width=\linewidth]
    {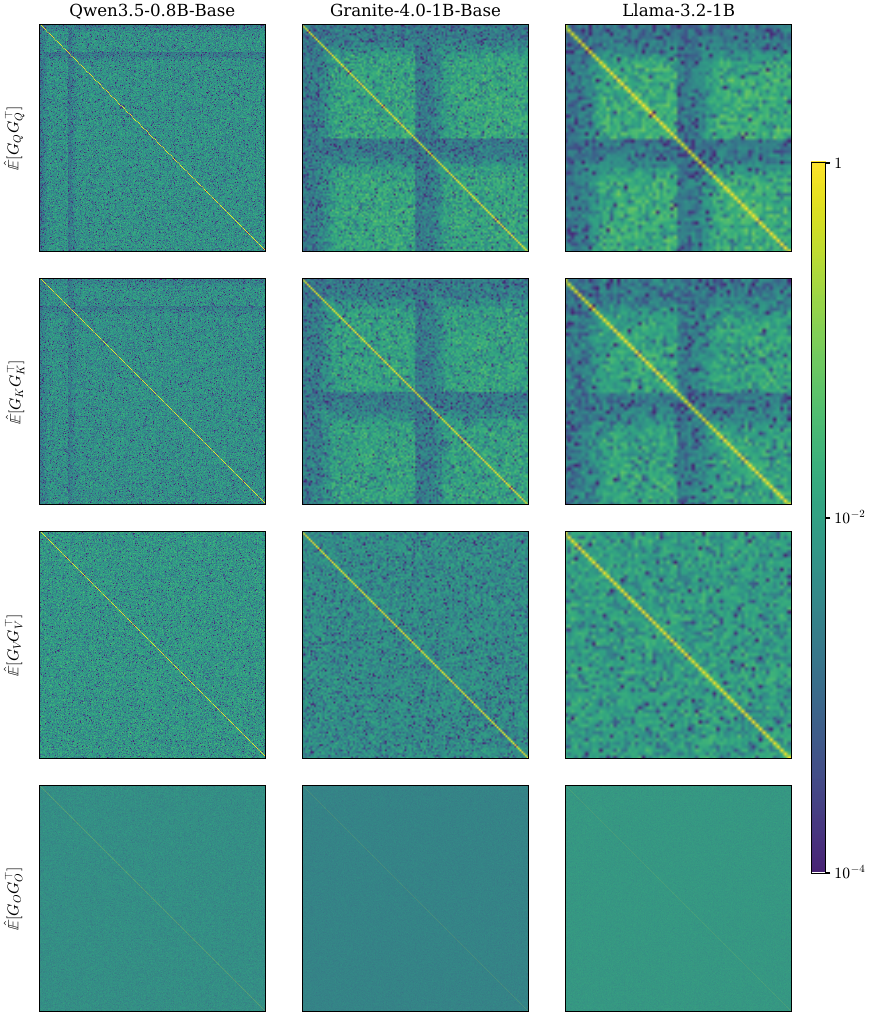}
    \caption{
    Empirical expected gradient self outer-products for
    Qwen3.5-0.8B-Base, Granite-4.0-1B-Base, and Llama-3.2-1B
    (columns), estimated from \(25\) independent random initializations
    and \(64\) fixed FineWeb-Edu sequences of length \(T=256\).
    The rows show
    \(\widehat{\mathbb{E}}[\GQ\GQ^\top]\),
    \(\widehat{\mathbb{E}}[\GK\GK^\top]\),
    \(\widehat{\mathbb{E}}[\GV\GV^\top]\), and
    \(\widehat{\mathbb{E}}[\GO\GO^\top]\).
    We use head \(0\) of block \(3\) for Qwen and head \(0\) of
    block \(0\) for Granite and Llama.
    Each heatmap is normalized by its empirical mean diagonal magnitude.
    A common logarithmic color scale is used across all panels.
}
    \label{fig:llm-gram-structure}
\end{figure}
\subsection{Model Width Scaling}
\label{app:llm-width-scaling}

We next examine how diagonal dominance changes with model width.
Using the same random initialization protocol as above, we vary the
transformer hidden width (\texttt{hidden\_size}) while holding the head
dimension (\texttt{head\_dim}) fixed at its released value. In the
notation of our single-layer attention model, this corresponds to
varying the residual input and output dimensions \(d_0=d_2\) while
keeping \(d_k=d_v\) fixed. The released head dimensions are \(256\)
for Qwen3.5-0.8B-Base, \(128\) for Granite-4.0-1B-Base, and \(64\)
for Llama-3.2-1B. The query-head count
(\texttt{num\_attention\_heads}) is held fixed, and the number of
key/value heads (\texttt{num\_key\_value\_heads}) is set equal to the number of query heads to enable analogous dimension sweep in the models. The sequence length is fixed at \(T=256\).

We consider the model-width grid
\[
\{128,256,512,768,1024,1536,2048,3072,4096,6144,
8192,12288,16384\}.
\]
At each width, we use \(25\) independent random initializations (obtained from
each model's native Hugging Face initialization scheme, which uses zero-mean Gaussian initialization for the attention projection
weights with model-specific scales) and
the same \(64\) fixed FineWeb-Edu sequences. We analyze head \(0\) in
block \(3\) for Qwen and in block \(0\) for Granite and Llama. For
each projection \(p\in\{Q,K,V,O\}\), we report the empirical
diagonal-dominance ratio
\(\widehat R_p=R(\widehat{\bm P}_p)\), with \(R(\cdot)\) defined in
\eqref{eq:appG-DOR}.

\begin{figure}[t]
    \centering
    \includegraphics[width=\linewidth]
    {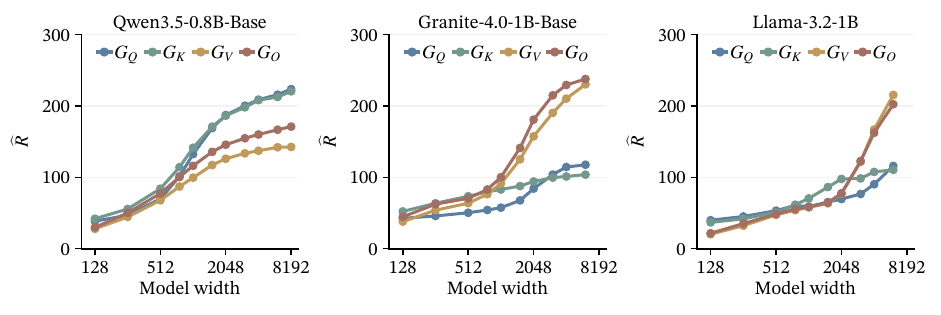}
    \caption{
    Model width scaling of the empirical diagonal-dominance ratio
    \(\widehat R\) for
    \(\GQ,\GK,\GV,\) and \(\GO\) in
    Qwen3.5-0.8B-Base, Granite-4.0-1B-Base, and Llama-3.2-1B.
    Model width (\texttt{hidden\_size}) is varied while the head
    dimension (\texttt{head\_dim}) is held fixed at its released value.
    We analyze head \(0\) in block \(3\) for Qwen and in block \(0\)
    for Granite and Llama. Each point is estimated from \(25\)
    independent random initializations and \(64\) fixed FineWeb-Edu
    sequences of length \(T=256\).
    }
    \label{fig:llm-dimension-scaling}
\end{figure}

Figure~\ref{fig:llm-dimension-scaling} shows that
\(\widehat R\) increases with model width for all four attention
projections across all three models. For \(\GO\), this is consistent
with the \(R=\Theta(d_vd_0^3)\) scaling in
\cref{cor:atta-scaling}. The increase is strongest for
\(\GQ\) and \(\GK\) in Qwen, and for \(\GV\) and \(\GO\) in
Granite and Llama.

\end{document}